\documentclass{article}

\usepackage{microtype}
\usepackage{graphicx}
\usepackage{subcaption}
\usepackage{booktabs} % for professional tables

\usepackage{amsmath,amsfonts,bm}

\def\eqref#1{Eq.~(\ref{#1})} %equation~\ref{#1}} Suzuki edit
\def\1{\bm{1}}

\def\va{{\bm{a}}}
\def\vb{{\bm{b}}}

\def\vd{{\bm{d}}}

\def\vh{{\bm{h}}}

\def\vn{{\bm{n}}}

\def\vu{{\bm{u}}}
\def\vv{{\bm{v}}}
\def\vw{{\bm{w}}}
\def\vx{{\bm{x}}}
\def\vy{{\bm{y}}}
\def\vz{{\bm{z}}}

\def\mA{{\bm{A}}}

\def\mD{{\bm{D}}}
\def\mE{{\bm{E}}}

\def\mI{{\bm{I}}}

\def\mN{{\bm{N}}}

\def\mP{{\bm{P}}}

\def\mU{{\bm{U}}}

\def\mW{{\bm{W}}}
\def\mX{{\bm{X}}}

\DeclareMathAlphabet{\mathsfit}{\encodingdefault}{\sfdefault}{m}{sl}
\SetMathAlphabet{\mathsfit}{bold}{\encodingdefault}{\sfdefault}{bx}{n}

\newcommand{\E}{\mathbb{E}}

\newcommand{\R}{\mathbb{R}}

\newcommand{\poly}{\mathrm{poly}}

\usepackage{hyperref}

\newcommand{\inner}[2] {\left\langle #1, #2 \right\rangle}

\usepackage{etoc} % 追加
\let\oldaddcontentsline\addcontentsline

\usepackage[preprint]{neurips_2026}
\let\addcontentsline\oldaddcontentsline

\usepackage{amsmath}
\usepackage{amssymb}
\usepackage{mathtools}
\usepackage{amsthm}
\usepackage{algorithm2e}
\usepackage{algorithmic}
\usepackage{enumitem}

\usepackage[capitalize,noabbrev]{cleveref}

\theoremstyle{plain}
\newtheorem{theorem}{Theorem}[section]

\newtheorem{lemma}[theorem]{Lemma}

\theoremstyle{definition}
\newtheorem{definition}[theorem]{Definition}

\theoremstyle{remark}
\newtheorem{remark}[theorem]{Remark}

\newenvironment{reptheorem}[2]{%
  \theoremstyle{plain}%
  \newtheorem*{rep@theorem}{\bfseries Theorem \ref{#1} {\normalfont (#2)}}%
  \begin{rep@theorem}%
}{%
  \end{rep@theorem}%
}

\newenvironment{replemma}[2]{%
  \theoremstyle{plain}%

  \newtheorem*{rep@lemma}{\bfseries Lemma \ref{#1} {\normalfont (#2)}}%
  
  \begin{rep@lemma}%
}{%
  \end{rep@lemma}%
}

\newenvironment{replemma2}[2]{%
  \theoremstyle{plain}%

  \newtheorem*{rep@lemma2}{\bfseries Lemma \ref{#1} {\normalfont (#2)}}%
  
  \begin{rep@lemma2}%
}{%
  \end{rep@lemma2}%
}
\usepackage[textsize=tiny]{todonotes}

\title{Test time training enhances in-context learning of nonlinear functions}

\author{
  Kento Kuwataka \\%\thanks{Use footnote for providing further information about author (webpage, alternative address)---\emph{not} for acknowledgingfunding agencies.} \\
  Department of Mathematical Engineering and Information Physics\\
  The University of Tokyo\\
  Hongo 7-3-1, Bunkyo-ku, Tokyo \\
  \texttt{kuwataka-kento954@g.ecc.u-tokyo.ac.jp} \\
  \And
   Taiji Suzuki \\
   Department of Mathematical Engineering and Information Physics \\
   Hongo 7-3-1, Bunkyo-ku, Tokyo \\
   \texttt{taiji@mist.i.u-tokyo.ac.jp} \\
}

\begin{document}
\maketitle

\begin{abstract}
  Test-time training (TTT) enhances model performance by explicitly updating designated parameters prior to each prediction to adapt to the test data.  While TTT has demonstrated considerable empirical success, its theoretical underpinnings remain limited, particularly for nonlinear models.  In this paper, we investigate the combination of TTT with in-context learning (ICL), where the model is given a few examples from the target distribution at inference time.  We analyze this framework in the setting of single-index models, where the feature vector is drawn from a hidden low-dimensional subspace.  For single-layer transformers trained with gradient-based algorithms and adopting TTT, we establish an upper bound on the prediction risk.  Our theory reveals that TTT enables the single-layer transformers to adapt to both the feature vector and the link function, which vary across tasks.  This creates a sharp contrast with ICL alone, which is theoretically difficult to adapt to shifts in the link function.  Moreover, we provide the convergence rate with respect to the data length, showing the predictive error can be driven arbitrarily close to the noise level as the context size and the network width grow.
\end{abstract}

\setlist[itemize]{leftmargin=*}

\section{Introduction}
In-context learning (ICL,\citet{brown2020languagemodelsfewshotlearners}) is a powerful capability of pretrained transformers to solve tasks using a few labeled examples provided as input, without updating their weights.  This ability has gained increasing attention with the advent of models with massive context windows, as more examples lead to significantly improved performance \citep{agarwal2024manyshotincontextlearning}. This approach has also proven effective for multimodal tasks \citep{jiang2024manyshotincontextlearningmultimodal}. From a theoretical perspective, transformers are known to implement algorithms such as linear regression\citep{garg2023transformerslearnincontextcase,vonoswald2023transformerslearnincontextgradient,zhang2024incontextlearninglineartransformer,gatmiry2024loopedtransformerslearnimplement} . Recent studies have extended this understanding to nonlinear settings, showing that transformers can learn nonlinear single-index models \citep{oko2024pretrainedtransformerefficientlylearns} and that softmax attention facilitates data-efficient feature learning \citep{nishikawa2025nonlinear}. Nevertheless, ICL faces its inherent limitations: the performance of ICL is fundamentally constrained by factors like the pretraining data \citep{bigoulaeva2025inherentlimitspretrainedllms} and model architecture \citep{naim2025analyzinglimitsincontextlearning}.

Test-time training (TTT) has emerged as a promising strategy to overcome these barriers.  TTT adapts the model by updating its parameters on the test data before each prediction. This adaptive mechanism has led to strong empirical success across various fields, including large language models \citep{hu2025testtimelearninglargelanguage} and video object segmentation \citep{NEURIPS2023_4267d84c}. In the context of ICL, TTT can be seamlessly integrated by using the in-context examples as data for task-specific adaptation. For instance, \citet{akyürek2025surprisingeffectivenesstesttimetraining} demonstrated that this combination achieves notable improvements on few-shot reasoning benchmarks.

Despite its empirical success, the theoretical foundations of TTT remain underdeveloped.  A notable work by \citet{gozeten2025testtimetrainingprovablyimproves} established the statistical efficiency of TTT over standard ICL, but their analysis was restricted to linear regression with a linear transformer.  This simplified model fails to capture the true potential of TTT's power on nonlinear, complex tasks. This gap motivates our primary research question:
\begin{center}
\textit{Does test-time training improve in-context learning in nonlinear settings?} 
\end{center} 
Furthermore, existing theoretical works commonly analyze performance in high-dimensional regimes, showing that the prediction loss is $o_d(1)$ and thus vanishes as the data dimension $d$ grows. In practice, however, this dimension is fixed, making it crucial to understand how the loss behaves as the number of data points $n$ increases. This raises our second key question:
\begin{center}
\textit{How does the test loss behave when we fix the dimension and increase the data size?}
\end{center}

\subsection{Our contribution}
To address the questions above, we analyze the performance of transformers on learning single-index models, a simple type of nonlinear function.  In our setting, each task is defined as
\begin{equation*}
    \vx_1^{t},\dots,\vx_N^{t},\vx^{t} \overset{\text{i.i.d.}}{\sim} \mathcal {N}(0,\mI_d),~ y_i^t \approx \sigma_*^t(\inner{\beta^t}{\vx_i^t})~
    (i=1,\cdots,n)
\end{equation*}
where $\sigma_*^t$ is an unknown polynomial that varies across tasks, and $\beta^t$ is drawn from a fixed $r$-dimensional subspace.  Using the in-context data $\vx_1,\cdots,\vx_{N}$, our model constructs the predictor for the new query $\vx$.  We establish a rigorous upper-bound on the predictive risk for a transformer that utilizes TTT.  Our main result is as follows:

\begin{theorem}[Informal]\label{theo:main}
    Consider learning single-index polynomial $\sigma_*^t(\inner{\beta^t}{\vx})$ with the transformer trained via Algorithm \ref{alg:pretraining}.  Then, with probability at least 0.99, we can construct the model $f_\mathrm{TF}$ for each prompt that satisfies
   $\mathbb{E}[|f_\mathrm{TF}(\vx^t)-\sigma_*^t(\inner{\beta^t}{\vx^t})|]=\tilde{O}(m^{-1/2})+\tilde{O}(\sqrt{\frac{r\sqrt{r}}{N_{test}}})$,
    where $N_{pt}$ and $T_{pt}$ are the context length and the number of tasks in pretraining, respectively, $N_{test}$ is the context length in test-time, and $m$ is the network width, if $N_{pt}=T_{pt}= \tilde{\Omega}(d^{\Omega(\mathrm{ie}(\sigma_*^t))}r^2)$ and $N_{test}= \tilde{\Omega}(r^{\Omega(\mathrm{ge}( \sigma_*^{\mathrm{test}}))})$, where $\mathrm{ie}(\sigma_*)$ and $\mathrm{ge}(\sigma_*)$ are the information exponent and the general exponent of the polynomial $\sigma_*$, respectively.
\end{theorem}

This theorem ensures the effectiveness of TTT in learning nonlinear single-index models, extending its known applicability to linear models. 
Our analysis reveals several strengths of our approach:
    \begin{itemize}
        \item Efficient sample complexity: Theorem~\ref{theo:main} implies that $N_{test} = \tilde{\Omega}(r^{\Theta(\mathrm{ge}( \sigma_*^{\mathrm{test}}))})$ to ensure low predictive loss $\tilde{O}(\sqrt{\frac{r\sqrt{r}}{N_{test}}})$, which does not depend on the entire dimension $d$.  This shows that transformers can adapt to the low-dimensionality of $\beta$.  In addition, this statistical complexity does not depend on either the degree of the polynomial $\mathrm{deg}(\sigma_*^{\mathrm{test}})$ or $\mathrm{ie}(\sigma_*^{\mathrm{test}})$, which outperforms CSQ learners and shows a comparable performance with that of SQ learners.
        \item Flexibility for varying nonlinearity: Our framework allows the link function $\sigma_*$ to vary across tasks.  This adaptability is enabled by TTT, which fine-tunes the MLP layer each time using task-specific test data.
        \item Statistical guarantee for practical settings: We provide an explicit convergence rate with respect to the context length $N_{test}$.  This result offers a statistical guarantee in practical scenarios where the dimensions $d,r$ are large but fixed.  This establishes a fundamental advantage over standard ICL: while existing frameworks \citep{nishikawa2025nonlinear} fail to prove that the predictive risk vanishes as $N_{test} \to \infty$, our TTT-based approach ensures that the error vanishes with increasing context size.
    \end{itemize}

\subsection{Related works}
\paragraph{In-context learning and its theoretical analysis}
In-context learning \citep{brown2020languagemodelsfewshotlearners} is the ability of transformers to adapt to the specific task using few labeled examples, without updating any parameters.  \citet{agarwal2024manyshotincontextlearning} demonstrated that many in-context examples lead to considerably improved performance, and \citet{jiang2024manyshotincontextlearningmultimodal} confirmed that many-shot ICL is also beneficial for multimodal tasks.   The theoretical background of ICL is extensively studied.  For example,  a wide array of works \citep{garg2023transformerslearnincontextcase,vonoswald2023transformerslearnincontextgradient,zhang2024incontextlearninglineartransformer,gatmiry2024loopedtransformerslearnimplement} have shown that linear transformers can be trained to perform linear regression in-context.  For nonlinear transformers, \citet{cheng2024transformersimplementfunctionalgradient} demonstrated that nonlinear transformers learn to perform gradient descent and thus learn nonlinear functions.  Beyond standard regression, transformers can also learn structured systems and dynamical processes in-context, such as causal structures \citep{nichani2024transformerslearncausalstructure} and Markov processes \citep{edelman2024the,rajaraman2024transformers}.  In addition, \citet{dherin2025learningtrainingimplicitdynamics} showed that ICL in a single-transformer block (a self-attention layer and subsequent MLP layer) corresponds to low-rank update in MLP layer.  Furthermore, \citet{deora2025incontext} investigated ICL across tasks with hierarchical complexity, demonstrating that transformers inherently identify and favor the simplest sufficient explanation via an in-context Bayesian Occam's razor.  Regarding limitations of ICL, \citet{bigoulaeva2025inherentlimitspretrainedllms} argued that pretraining datasets impose a fundamental limit on the model's capability with ICL.  Furthermore, \citep{naim2025analyzinglimitsincontextlearning} demonstrated that the transformer's ICL ability to generalize functions is limited to certain input values, and found that this limitation comes from layer normalization and softmax attention.

\paragraph{Test-time training}
Test-time training \citep{sun2020testtimetrainingselfsupervisiongeneralization,Liu2021TTTWD} updates the model using test data before making predictions, thereby addressing distribution shifts.  TTT achieved success in many fields.  For example, \citet{NEURIPS2023_4267d84c} applied TTT for the video object segmentation task and achieved a significant improvement in the performance.  Moreover, \citet{hu2025testtimelearninglargelanguage} analyzed test-time learning of large language models and achieved at least 20\% higher performance on domain knowledge adaptation.  Furthermore, \citet{zhang2025testtimetrainingright} proposed adopting a large chunk update, and validated the effectiveness of their approach to long-context data through tasks like image sets and language model.  As for TTT combined with few-shot prediction, \citet{akyürek2025surprisingeffectivenesstesttimetraining} reported that introducing TTT with in-context examples resulted in 6 times higher accuracy in the Abstraction and Reasoning Corpus and a 7.3 percent higher score on BIG-Bench Hard.  Finally, regarding the theory behind TTT for in-context learning, \citet{gozeten2025testtimetrainingprovablyimproves} analyzed the linear transformer with a single gradient step and characterized the prediction risk of the model with TTT, showing that TTT can mitigate distribution shift.

\section{Preliminaries and problem settings}

\paragraph{Notations}
Let $\mathrm{He}_i(z)=(-1)^i\mathrm{e}^{\frac{z^2}{2}}\frac{\mathrm{d}^i}{\mathrm{d}z^i}\mathrm{e}^{\frac{-z^2}{2}}$ be the degree-$i$ (probabilist's) Hermite polynomial.  $\mathbb{S}^{d-1}$ denotes the unit sphere in $\mathbb{R}^{d}$.  We denote the identity matrix with $d$ rows and $d$ columns as $\mI_d$.  For matrix $\mA$, we denote its $\ell_2$ operator norm and Frobenius norm as $\|\mA\|_2$ and $\|\mA\|_F$, respectively.  For a set $\mathrm{S}$, $\mathrm{Unif(S)}$ denotes the uniform distribution over $\mathrm{S}$.  $\tilde{O},\tilde{\Omega},\tilde{\Theta}$ means $O,\Omega,\Theta$ where polylogarithmic terms of $d$ and $1/\varepsilon$ are hidden.  $O_d,\Omega_d,\Theta_d$ means the order with respect to the dimension $d,r$, while $\Theta_\varepsilon$ denotes the order in terms of $\varepsilon$.

\subsection{In-context learning and test-time training}
We consider the basic setting in ICL, which is introduced by \citet{garg2023transformerslearnincontextcase} (see  \citet{oko2024pretrainedtransformerefficientlylearns} and \citet{lee2024neuralnetworklearnslowdimensional}).  In this setting, the model is given a sequence $(\vx_1,y_1,\cdots,\vx_N,y_N,\vx)$ called prompt.  The labeled pairs $(\vx_i,y_i) \in \mathbb{R}^{d}\times \mathbb{R}$ are called contexts, and $\vx \in \mathbb{R}^d$ is referred to as query. The model is asked to predict the output that corresponds to $\vx$ based on the context.  The context is sometimes abbreviated as $(\mX_n,\vy_n)$ where $\mX_n = (\vx_1,\cdots,\vx_n), \vy_n = (y_1,\cdots,y_n)$.  In this work, we assume that the $\vx$ and $y$ are generated as follows:
\begin{equation*}
    \vx_1^{t},\dots,\vx_N^{t},\vx^{t} \overset{\text{i.i.d.}}{\sim} \mathcal {N}(0,\mI_d),~ y_i^{t} = f_{*}^{t}(\vx_i^{t}) + \zeta_i,
     ~ \zeta_i\sim\mathrm{Unif}({\{-\tau, \tau\}}).
\end{equation*}

While ICL aims to predict $y = f_{*}^{t}(\vx)+\zeta$  without updating parameters for each prompt, we introduce test-time training to further enhance the model's accuracy.  In pretraining, we train the model with $T_{pt}$ distinct datasets $(\mX_{N_{pt}}^t,\vy_{N_{pt}}^t,\vx^t,y^t)_{t=1}^{T_{pt}}$, with each prompt consisting of $N_{pt}$ contexts.  At test-time, we divide the context into four groups with $i$-th group's length $N_i$, forming the prompt $(\mX_{N_1},\vy_{N_1},\mX_{N_2},\vy_{N_2}, \mX_{N_3},\vy_{N_3},\mX_{N_4},\vy_{N_4},\vx)$.  Each group of data plays a different role in TTT: See subsection~\ref{sec:algorithm} for details.  Let $N_{test}=\sum_{i=1}^{4}N_i$ be the total number of contexts at test-time.

For the evaluation of the model $f(\vx,\theta)$ with parameter $\theta$, we define prediction risk as
\begin{equation*}
    \mathcal{R}_f(\theta)=\mathbb{E}[|f(\vx,\theta)-y|],
\end{equation*}
where $y=f_*(\vx)+\zeta$ and the expectation is taken over the prompt $\vx_i,\vx \sim \mathcal{D}_{\vx}, f_* \sim \mathcal{D}_{f_*}, \zeta_i,\zeta \sim \mathcal{D}_{\zeta}$.  Note that $\mathcal{R}_f(\theta)$ not only depends on the dimensions $d,r$ but also on the context length $N$.

\subsection{Single index model}

We consider single-index models for the input-output relationship, where the output depends solely on the direction of the feature vector $\beta$.  To predict this relationship accurately, the models are expected to learn the target direction $\beta$ from the high-dimensional data in $\mathbb{R}^{d}$.  Consequently, single-index models have been extensively studied in machine learning theory \citep{bai2020linearizationquadratichigherorderapproximation,ba2022highdimensionalasymptoticsfeaturelearning,bietti2022learningsingleindexmodelsshallow,mousavihosseini2023neuralnetworksefficientlylearn,Berthier_2024}, particularly to examine adaptability to the low-dimensional subspace.  The target function $f_{*}^{t}(\vx)$ is determined as 
$f_{*}^{t}(\vx) = \sigma_*^{t}(\inner{\beta^t}{\vx}).$
    
\begin{enumerate}[leftmargin=10pt, labelindent=0pt]
    \item \textbf{Feature vector:} The feature vector $\beta$ is chosen uniformly from the unit sphere in an $r$-dimensional subspace.  Let $S_r$ be an $r$-dimensional linear subspace in $\mathbb{R}^{d}$.  For each prompt, $\beta^t$ is uniformly drawn from its support $\mathrm{Supp}(\beta) = \{\beta \mid \beta \in S_r, \|\beta\|=1\}$. 
    \item \textbf{Link function:} We take $\sigma_*^t(z) = \sum_{i=Q}^{P}\frac{c_i^t}{i!}\mathrm{He}_i(z)$ where $1\leq Q \leq P$.  We assume $P$ and $Q$ are constants of $O(1)$.  The coefficients $c_i^t$ are drawn from any distribution $\mathcal{D}_{\sigma^*}$ that satisfies
    \begin{equation*}
        \mathbb{E}[c_Q^t]=\Theta(1) \neq 0 ,~\sum_{i=Q}^{P}{(c_i^t)}^2 \leq R_c=\Theta(1) ~(\text{a.s.})~  \text{and} ~ (c_Q^t,\cdots,c_P^t)\neq (0,\cdots,0)~(\text{a.s.}).
    \end{equation*}
\end{enumerate}

\begin{remark}
    We draw a new feature vector $\beta^t$ and a new link function $\sigma_*^t(z)$ for each prompt.  When $r\ll d$, $\beta^t$ has low-dimensional support.  We will later demonstrate that our model can leverage this low-dimensionality.  Note that our analysis is also valid when $r=d$.  As for the link function, we do not assume a specific distribution of the coefficients $c_i$: our framework allows different distributions as long as they satisfy the assumptions above.  
\end{remark}

The complexity of learning single-index model is governed by three key quantities: the degree of the polynomial $\mathrm{deg}(\sigma_*)$, the information exponent $\mathrm{ie}(\sigma_*)$, and the general exponent $\mathrm{ge}(\sigma_*)$.

\begin{itemize}
    \item The information exponent \citep{pmlr-v75-dudeja18a,JMLR:v22:20-1288} $\mathrm{ie}(\sigma_*)$ is the smallest non-zero degree of the hermite expansion of $\sigma_*$.
    \item The general exponent \citep{damian2024computationalstatisticalgapsgaussiansingleindex} $\mathrm{ge}(\sigma_*)$ is the minimum of $\mathrm{ie}(f \circ\sigma_*)$ with respect to all the $L_2$- measurable transformation $f$. 
\end{itemize}

By definition, $\mathrm{deg}(\sigma_*)\geq \mathrm{ie}(\sigma_*) \geq \mathrm{ge}(\sigma_*)$ holds.  This relationship characterizes the statistical complexity required by different algorithms.  
\begin{itemize}
    \item The kernel method, which relies on pre-determined feature maps, requires the sample complexity $n \gtrsim d^{\Theta(\mathrm{deg}(\sigma_*))}$ to ensure low predictive error \citep{ghorbani2020linearizedtwolayersneuralnetworks,pmlr-v139-donhauser21a}.
    \item The models with access to correlational statistical query (CSQ), of the form $\mathbb{E}[\phi(\vx)y]$, require a sample size of $n \gtrsim d^{\Theta(\mathrm{ie}(\sigma_*))}$, known as CSQ lower bound \citep{damian2022neural,pmlr-v195-abbe23a}.  This improved sample complexity, independent of $\mathrm{deg}(\sigma_*)$, has been achieved by models like two-layer neural network with online SGD \citep{JMLR:v22:20-1288,damian2023smoothinglandscapeboostssignal} or one-step gradient descent\citep{damian2022neural,dandi2025twolayerneuralnetworkslearn}.
    \item For the broader class of statistical queries (SQ), which takes the form $\mathbb{E}[\phi(\vx,y)]$, the sample complexity is further improved to $n \gtrsim d^{\Theta(\mathrm{ge}(\sigma_*))}$.  This advantage comes from applying nonlinear transformations to the label $y$, thereby reducing $\mathrm{ie}(\sigma_*)$.  Recent works have achieved this complexity by reusing the data \citep{lee2024neuralnetworklearnslowdimensional, arnaboldi2025repetitaiuvantdatarepetition} or adjusting the loss function \citep{joshi2024complexitylearningsparsefunctions}.  Furthermore, when $\sigma_*$ is a polynomial, the following result holds:

\begin{lemma}[\citet{lee2024neuralnetworklearnslowdimensional}, Proposition 6]
        It holds that
      $ \mathrm{ge}(\sigma_*)=\begin{cases}
            1~(\text{if}~\sigma_*~\text{is not even})\\
            2~(\text{if}~\sigma_*~\text{is even})
        \end{cases}.$

    Moreover, $\mathrm{ge}(\sigma_*)=\min_{j\geq 1}\mathrm{ie}((\sigma_*)^j)$ holds.
\end{lemma}
This implies that, for any polynomial $\sigma_*$, $\mathrm{ge}(\sigma_*)\leq2$ is a small constant, regardless of $\mathrm{deg}(\sigma_*)$.
\end{itemize}
 Our goal is to achieve the test-time sample complexity of $N_{test} = r^{\Theta(\mathrm{ge}(\sigma_*))}$, which is independent of the entire dimension $d$ and surpasses the CSQ limit.

\subsection{Student model}
To facilitate our theoretical analysis, we need to establish a concrete architecture for the student model.  
We employ a single-layer transformer model, formulated as follows.  In pretraining, we use the same model as \citet{nishikawa2025nonlinear}.  We first construct the embedding $\mE$ as
    $$\mE = \begin{bmatrix}
     \vx_1 & \cdots & \vx_N & \vx\\
     y_1 & \cdots & y_N & 1
    \end{bmatrix}    
    \in \mathbb{R}^{(d+1)\times(N+1)}.$$
    
Then, we apply the softmax attention layer as 
\begin{equation*}
    \mathrm{Attn}(\mE)
    =\mW^V\mE\cdot \mathrm{softmax}(\mathrm{Mask}(\rho^{-1}\cdot (\mW^K\mE)^\top\mW^Q\mE)),
\end{equation*}

where $\rho$ is temperature and $\mW^V,\mW^K,\mW^Q \in \mathbb{R}^{(d+1)\times(d+1)}$ are the parameters for attention layer.  The softmax is applied to each column, while the Mask function converts all the elements in the final row into $-\infty$ to prevent the model from focusing on the uninformative final row.  We further apply a multi-layer perceptron (MLP) layer.  For the activation function, we use ReLU function $\sigma(z)=\max{\{0,z\}}$ throughout the paper. With the parameters $\mW^F \in \mathbb{R}^{m \times(d+1)}, \vb \in \mathbb{R}^{m}$, and $\va \in \mathbb{R}^{m}$ where $m$ is the network width, the model's output is 
\begin{align*}
     f_{\mathrm{IC}}(\mathbf{\Gamma},\mX_N,\vy_N,\vx) &=\mathrm{MLP} \circ \mathrm{Attn}(\mE)_{:,N+1}\\
  &  = \va^{\top} \sigma(\mW^F\mathrm{Attn}(\mE)_{:,N+1}+\vb),
\end{align*}
where $\sigma$ is applied entry-wise.
Finally, we adopt some simplification.  Let $\mW^{KQ}=(\mW^K)^{\top}\mW^Q \in \mathbb{R}^{(d+1)\times(d+1)}$ and $\mW^{FV}=\mW^F\mW^V \in \mathbb{R}^{m\times(d+1)}$.  We use the following parametrization:
\begin{equation*}
    \mW^{KQ} = \begin{bmatrix}
        \mathbf{\Gamma} & \mathbf{0}_{d\times1}\\
        \mathbf{0}_{1\times d} & 1
    \end{bmatrix}
    ,\, \mW^{FV} = \begin{bmatrix}
        \mathbf{O}_{m\times d} & \vv
    \end{bmatrix}, 
\end{equation*}
for $\mathbf{\Gamma} \in \mathbb{R}^{d\times d}$ and $\vv \in \mathbb{R}^{m \times 1}$.  This kind of simplification, specifying some of the parameters as zero, is often adopted in many theoretical works on transformers \citep{zhang2023trainedtransformerslearnlinear,huang2023incontextconvergencetransformers,kim2025transformersprovablysolveparity}.  Overall, the model's output is written as
\begin{equation*}
    f_{\mathrm{IC}}(\mathbf{\Gamma},\mX_N,\vy_N,\vx)
    = \sum_{j=1}^{m}a_j\sigma\left(v_j \frac{\sum_{i=1}^{N}y_i\mathrm{e}^{y_i/\rho}\mathrm{e}^{\vx_i^{\top}\mathbf{\Gamma}\vx/\rho}}{\sum_{i=1}^{N}\mathrm{e}^{y_i/\rho}\mathrm{e}^{\vx_i^{\top}\mathbf{\Gamma}\vx/\rho}}+b_j \right).
\end{equation*}

See Appendix A in \citet{nishikawa2025nonlinear} for how to derive this equation.  At test-time, we adopt low-rank adaptation (LoRA).  Specifically, we change the attention matrix $\mathbf{\Gamma}^*$ as $\mathbf{\Gamma}_u = \mathbf{\Gamma}^* + \vu\vu^{\top}$, where $\mathbf{\Gamma}^*$ is fixed during test-time and $\vu \in \mathbb{R}^{d}$ is a trainable parameter vector.  Our goal is to find $\hat{\vu} \approx \beta$ using test-time context data.  See section~\ref{sec:algorithm} for how test data is used to find $\hat{\vu}$.  The final prediction of the model is as follows:
\begin{equation*}
    f_\mathrm{TF}(\vx,\hat{\vu},\vv,\va,\vb) = \sum_{j=1}^{m}a_j\sigma(v_j\inner{\hat{\vu}}{\vx}+b_j).
\end{equation*}

\begin{remark}\label{rem:inevitableloss}

Incorporating in-context data into the final prediction, a standard practice in existing TTT frameworks \citep{akyürek2025surprisingeffectivenesstesttimetraining}, improves efficiency in the small-data regime but creates an asymptotic bottleneck in the large-data regime.  As \citet{nishikawa2025nonlinear} explains, when the context length $N$ is sufficiently large, the denominator of the attention output $N^{-1}\sum_{i=1}^{N}\mathrm{e}^{y_i/\rho}\mathrm{e}^{\vx_i^{\top}\Gamma\vx/\rho}$ approximates $\mathbb{E}[\exp\sigma_*({\inner{\beta}{\vx_1}/\rho})\exp(\inner{\Gamma_*\vx}{\vx_1}/\rho)]$. The key of their work is, since $\exp(\inner{\Gamma_*\vx}{\vx_1}/\rho)$ contains all the Hermite coefficients and $\exp\sigma_*({\inner{\beta}{\vx_1}/\rho})$ has nonzero coefficient for $\mathrm{He}_{\mathrm{ge}(\sigma_*)}$, this attention layer computes $\inner{\beta}{\vx}^{\mathrm{ge}(\sigma_*)}$.  However, this mechanism inevitably retains $\inner{\beta}{\vx}^{k}$ for degrees $k > \mathrm{ge}(\sigma_*)$, causing an unavoidable bias that persists even as $N\rightarrow \infty$.  By contrast, our final predictor omits in-context data and relies solely on $\hat{\vu}$, thereby avoiding asymptotic bias.  We adopt this design to demonstrate that TTT can surpass the standard ICL limit and attain exact recovery.
\end{remark}

\subsection{Training algorithm}\label{sec:algorithm}
We employ a gradient-based training algorithm, as specified in Algorithm~\ref{alg:pretraining}.  The training procedure consists of pretraining and test-time training, with the latter divided into three stages.
\begin{itemize}
    \item Pretraining:  We optimize $\Gamma$ via one-step gradient descent over $T_{pt}$ prompts.  This scheme is originally taken from \citet{nishikawa2025nonlinear}.  The effectiveness of one-step gradient descent is confirmed by works such as \citet{ba2022highdimensionalasymptoticsfeaturelearning} and \citet{damian2022neural}, which demonstrate that one-step gradient already captures the key feature.   
    \item TTT stage I:  We apply a single gradient descent step to $\vu^{(0)}$ with $L_2$-regularization.  This stage prevents catastrophic forgetting, a phenomenon where the parameter change caused by TTT deprives the model of the fundamental ability to adapt to the original task.  The goal of this stage is to find a good initial value of $\vu^{(1)}$ that satisfies $\inner{\beta}{\vu^{(1)}} \geq 1/\mathrm{polylog}(d)$. For this stage, we use the output from the attention layer of the original model as a teacher signal, which is defined as $g(\Gamma_*,\mX_{N1},\vy_{N_1},\vx) = \frac{N_1^{-1}\sum_{i=1}^{N_1}y_i\mathrm{e}^{y_i/\rho}\mathrm{e}^{\vx_i^\top\Gamma_*\vx/\rho}}{N_1^{-1}\sum_{i=1}^{N_1}\mathrm{e}^{y_i/\rho}\mathrm{e}^{\vx_i^\top\Gamma_*\vx/\rho}}$.  In this stage, the query $\vx$ do not require ground-truth label $y$.  Therefore, we use $N_{new}$ newly generated vectors $w_1,\dots,w_{N_{new}}\overset{\text{i.i.d.}}{\sim} \mathcal {N}(0,\mI_d)$ for the query.  The necessity of this stage is further discussed in subsection \ref{section:weakrecovery}.
    \item TTT stage II:  We continue to optimize $\vu$ by applying multi-step online SGD scheme that originates from \citet{lee2024neuralnetworklearnslowdimensional}, using $(\mX_{N_2},\vy_{N_2})$ as contexts.  In this stage, we use the ground truth $(\mX_{N_3},\vy_{N_3})$ as the teacher signal.  This stage aims to align $\vu$ more closely to the target $\beta$.  As we will see, increasing the number of the SGD steps leads to the convergence of $\vu$ to $\beta$.
    \item TTT stage III:   Finally, we fit the MLP layer to the nonlinear link function $ \sigma_*^{\mathrm{test}}$.   Specifically, we randomize $\vv$ and $\vb$ and optimize $\va$ with ridge regression, following the recipe in \citet{nishikawa2025nonlinear}.  The problem's convexity with respect to $\va$ ensures the global optimum is efficiently found.
\end{itemize}

\begin{algorithm}[t!]
%\SetKwInOut{Input}{Input}
%\SetKwInOut{Output}{Output}
%\SetKwBlock{PreTraining}{Pretraining: One-step Gradient descent on Attention Matrix}{}
%\SetKwBlock{TestTraining}{Test-Time Training}{}
%\SetKwBlock{StageOne}{Stage I: Initialization of u with signals from Original Model}{}
%\SetKwBlock{StageTwo}{Stage II: Strong Recovery}{}
%\SetKwBlock{StageThree}{Stage III: Training of MLP Layer}{}
%\SetKw{Initialize}{Initialize}

\caption{Pretraining and test-time training of transformer} 
 \label{alg:pretraining}

 \begin{algorithmic}
     \STATE {\bfseries Input:} Learning rate $\eta_{pt},\eta_1,\eta_2$, regularization rate $\lambda_{pt},\lambda_1,\lambda_2$, initialization scale $\alpha_{pt},\alpha_1,\alpha_2$, dimensions $d,r$, temperature $\rho$.
     \STATE $\Gamma(0)\sim \mI_d/\sqrt
{d},~\vv(0)\sim\mathrm{Unif}(\{\pm 1\}^m),~\vb(0)=\bm{0}_m,~\va(0)=\alpha_{pt}\bm{1}_m$.
     \STATE {\bfseries  \COMMENT {Pretraining}}
     \STATE $\Gamma^*\leftarrow \Gamma{(0)}-\eta_{pt}\frac{1}{2T_{pt}}\sum_{t=1}^{T_{pt}}\nabla_{\Gamma}(( f_{\mathrm{IC}}(\Gamma,\mX_{N_{pt}},\vy_{N_{pt}},\vx^t)-y^t)^2+\lambda_{pt}\|\Gamma\|_F^2)\big|_{\Gamma=\Gamma{(0)}}$.
     \STATE {\bfseries \COMMENT {Test-time Training}}
     \STATE $\va(0)=\alpha_{1}\bm{1}_m$.
     \FOR{$i=1$ {\bfseries to} $N_1+N_2$}
       \STATE $\vx_i \leftarrow \sqrt{r}\Gamma^*\vx_i$.
     \ENDFOR
     \STATE $\vu^{(0)}\sim \mathcal{N}(0,\mI_d)$, $\vu^{(0)} \leftarrow \sqrt{r}\Gamma_*\vu^{(0)}$,  $\vu^{(0)} \leftarrow \frac{\vu^{(0)}}{\sqrt{r}\|\vu^{(0)}\|}$.
     \STATE {\bfseries \COMMENT{Stage I: Weak Recovery}}
     \STATE Draw $\vw_1,\cdots,\vw_{N_{new}} \overset{\text{i.i.d.}}{\sim} \mathcal {N}(0,\mI_d)$.
     \STATE $b = \frac{1}{N_{new}}\sum_{i=1}^{N_{new}}g(\Gamma^*,\mX_{N_1},\vy_{N_1},\vw_i)$, 
     \STATE $\vu^{(1)} = \vu^{(0)} - \eta_1\{ \sqrt{r} \Gamma^* \nabla_{\vu}\frac{1}{2N_{new}}\sum_{i=1}^{N_{new}}$
     \STATE $\quad (f_{\mathrm{IC}}(\Gamma_\vu,\mX_{N_1},\vy_{N_1},\vw_i)- (g(\Gamma_*,\mX_{N_1},\vy_{N_1},\vw_i)-b))^2+\frac{\lambda_1}{2}\nabla_{\vu}\|\vu\|^2\}\big|_{\vu=\vu^{(0)}}$,
     \STATE $\vu^{(1)} \leftarrow \frac{\vu^{(1)}}{\|\vu^{(1)}\|}$. 
     \STATE $\va(0)=\alpha_{2}\bm{1}_m$.
%     \FOR{$i=N_1+1$ {\bfseries to} $N_1+N_2$}
%       \STATE $\vx_i \leftarrow \sqrt{r}\Gamma^*\vx_i$.
%     \ENDFOR
     \STATE {\bfseries \COMMENT{Stage II: Strong Recovery}}
     \FOR{$t=1$ {\bfseries to} $N_3$}
       \STATE $\vu^{(t+1)} = \vu^{(t)} - \eta_2\sqrt{r}\Gamma^*\nabla_u(\frac{1}{2}(f_{\mathrm{IC}}(\Gamma_u,\mX_{N_2},\vy_{N_2},\vx_{N_1+N_2+t})-y_{N_1+N_2+t})^2)\big|_{\vu=\vu^{(t)}}$.
       \STATE $\vu^{(t+1)} \leftarrow \frac{\vu^{(t+1)}}{\|\vu^{(t+1)}\|}$.
     \ENDFOR
     \STATE $b_j^*\sim \mathrm{Unif}([-\log^2{d},\log^2{d}])$ , $\vv^*=\vv(0)$.
     \STATE {\bfseries \COMMENT{Stage III: Training of MLP Layer}}
     \STATE $\va^*\leftarrow \arg\min_{\va} \frac{1}{2N_4}\sum_{t=N_1+N_2+N_3+1}^{N_1+N_2+N_3+N_4}(f_{\mathrm{TF}}(\vx_{t},\vu^{(N_3+1)},\vv^*,\va,\vb^*)-y_{t})^2+\frac{\lambda_2}{2}\|\va\|^2$. 
     \STATE {\bfseries Output:} Prediction $f_{\mathrm{TF}}(\vx, \vu^{(N_3+1)}, \vv^*, \va^*, \vb^*)$.
 \end{algorithmic}

%\TestTraining{
%   $\va(0)=\alpha_{1}\bm{1}_m$. \\
%\For{$i=1$ \KwTo $N_1$}{$\vx_i \leftarrow \sqrt{r}\Gamma^*\vx_i$.}
 % $\vu^{(0)}\sim \mathcal{N}(0,\mI_d)$, $\vu^{(0)} \leftarrow \sqrt{r}\Gamma_*\vu^{(0)}$,  $\vu^{(0)} \leftarrow \frac{\vu^{(0)}}{\sqrt{r}\|\vu^{(0)}\|}$.\\
  
%  \StageOne{
 % Draw $\vw_1,\cdots,\vw_{N_{new}} \overset{\text{i.i.d.}}{\sim} \mathcal {N}(0,\mI_d)$. \\
%  $b = \frac{1}{N_{new}}\sum_{i=1}^{N_{new}}g(\Gamma^*,\mX_{N_1},\vy_{N_1},\vw_i)$, \\
%  $\vu^{(1)} = \vu^{(0)} - \eta_1\{ \sqrt{r} \Gamma^* \nabla_{\vu}\frac{1}{2N_{new}}\sum_{i=1}^{N_{new}}(f_{\mathrm{IC}}(\Gamma_\vu,\mX_{N_1},\vy_{N_1},\vw_i)-(g(\Gamma_*,\mX_{N_1},\vy_{N_1},\vw_i)-b))^2+\frac{\lambda_1}{2}\nabla_{\vu}\|\vu\|^2\}$, \\
%  $\vu^{(1)} \leftarrow \frac{\vu^{(1)}}{\|\vu^{(1)}\|}$. 
%  }
%  $\va(0)=\alpha_{2}\bm{1}_m$. \\
%  \For{$i=N_1+1$ \KwTo $N_1+N_2$}{$\vx_i \leftarrow \sqrt{r}\Gamma^*\vx_i$.}
%  \StageTwo{
%  \For{$t=1$ \KwTo $t=N_3$}{
%   $\vu^{(t+1)} = \vu^{(t)} - \eta_2\sqrt{r}\Gamma^*\nabla_u(\frac{1}{2}(f_{\mathrm{IC}}(\Gamma_u,\mX_{N_2},\vy_{N_2},\vx_{N_1+N_2+t})-y_{N_1+N_2+t})^2)$. %\tcp{one SGD step} 
%   $\vu^{(t+1)} \leftarrow \frac{\vu^{(t+1)}}{\|\vu^{(t+1)}\|}$.
 % }
 % }
%  \Initialize{$b_j^*\sim \mathrm{Unif}([-\log^2{d},\log^2{d}])$ , $\vv^*=\vv(0)$}.\\
%  \StageThree{
%  $\va^*\leftarrow \arg\min_{\va} \frac{1}{2N_4}\sum_{t=N_1+N_2+N_3+1}^{N_1+N_2+N_3+N_4}(f_{\mathrm{TF}}(\vx_{t},\vu^{(N_3+1)},\vv^*,\va,\vb^*)-y_{t})^2+\frac{\lambda_2}{2}\|\va\|^2$. 
%  } 
%}
%\Output{Prediction $f_{\mathrm{TF}}(\vx, \vu^{(N_3+1)}, \vv^*, \va^*, \vb^*)$.}

\end{algorithm}

\begin{remark}
%While Algorithm 1 involves multiple stages, this structure reflects distinct practical objectives rather than arbitrary complexity. In real-world implementations, pretraining and test-time training naturally operate as separate phases. Furthermore, Stage I utilizes a self-supervised objective for initialization, which fundamentally differs from the supervised objectives in the later stages; thus, it must be treated as a distinct step even in practice. The only structural deviation from a typical implementation is the separation of Stage II (Attention learning) and Stage III (MLP learning). This layer-wise decoupling is primarily introduced for theoretical analysis: it is a standard technique for obtaining rigorous guarantees and has been widely adopted in prior works \citep{damian2022neural,pmlr-v195-abbe23a,lee2024neuralnetworklearnslowdimensional,oko2024pretrainedtransformerefficientlylearns}.

Algorithm 1 reflects practical requirements rather than arbitrary complexity. In real-world implementations, pretraining and test-time training naturally operate as separate phases, and Stage I’s self-supervised objective necessitates a separate step from later supervised stages. 
The only significant deviation from typical end-to-end training is the layer-wise decoupling of Stages II and III, which is a standard theoretical technique for obtaining rigorous guarantees and has been widely adopted in prior works \citep{damian2022neural,pmlr-v195-abbe23a,lee2024neuralnetworklearnslowdimensional,oko2024pretrainedtransformerefficientlylearns}.
\end{remark}

\section{Background}
For learning Gaussian single-index models, \citet{nishikawa2025nonlinear} investigated in-context learning (ICL) where the target function depends on a feature vector with low-dimensional support.  They established that pretrained transformers can adapt to low-dimensional feature spaces in-context, achieving a predictive risk of $o_d(1)$ with a test-time sample complexity of $N_{\text{test}} = \tilde{\Omega}(r^{3\mathrm{ge}(\sigma_*)/2})$.  Despite these guarantees, the ICL framework in \citet{nishikawa2025nonlinear} suffers from two fundamental limitations:\textbf{(i)Static Link Functions:} \citet{nishikawa2025nonlinear} assume that the link function $\sigma_*$ is fixed across all pretraining and test tasks, as task adaptation relies exclusively on the static pre-trained representations. Standard ICL cannot adapt to tasks where the link function varies at test time.  This is because the training of the relevant layer is done only in pretraining.  \textbf{(ii)Asymptotic Bias as $N_{\text{test}} \to \infty$:}  In practical settings where dimensions $d$ and $r$ are fixed, we require the predictive risk to vanish as the context size $N_{\text{test}}$ increases.  However,  \citet{nishikawa2025nonlinear} guarantees that the ICL risk is $o_{d}(1)$, meaning it provides no guarantee that the risk diminishes as $n \rightarrow \infty$.  This is an inherent consequence of adopting softmax attention.  
As \citet{nishikawa2025nonlinear} explains, when the context length $N$ is sufficiently large, the denominator of the attention output $N^{-1}\sum_{i=1}^{N}\mathrm{e}^{y_i/\rho}\mathrm{e}^{\vx_i^{\top}\Gamma\vx/\rho}$ approximates $\mathbb{E}[\exp\sigma_*({\inner{\beta}{\vx_1}/\rho})\exp(\inner{\Gamma_*\vx}{\vx_1}/\rho)]$. The key of their work is, since $\exp(\inner{\Gamma_*\vx}{\vx_1}/\rho)$ contains all the Hermite coefficients and $\exp\sigma_*({\inner{\beta}{\vx_1}/\rho})$ has nonzero coefficient for $\mathrm{He}_{\mathrm{ge}(\sigma_*)}$, this attention layer computes $\inner{\beta}{\vx}^{\mathrm{ge}(\sigma_*)}$.  However, this mechanism inevitably retains $\inner{\beta}{\vx}^{k}$ for degrees $k > \mathrm{ge}(\sigma_*)$, causing an unavoidable bias that persists even as $N\rightarrow \infty$.
We resolve both limitations by integrating TTT (Algorithm 1). Fine-tuning the MLP layer during inference dynamically fits task-specific link functions. Furthermore, by omitting in-context tokens in our final predictor $f_{\mathrm{TF}}$, we eliminate the asymptotic bias of softmax attention, enabling exact recovery ($R_{\mathrm{TF}} \to \tau$) as $N_{\text{test}} \to \infty$.

\section{Main result}
Based on the problem settings above, we are now ready to present our main result.

\begin{reptheorem}{theo:main}{Formal}
  We denote the link function drawn in inference-time as $
  \sigma_*^{\mathrm{test}}$.  Suppose that $T_{pt},N_{pt}=\tilde{\Omega}(r^2d^{Q+2})$, $N_1 = \tilde{\Omega}(r^{\mathrm{ge}( \sigma_*^{\mathrm{test}})+2})$, $N_{new} = \tilde{\Omega}(r^{\mathrm{ge}( \sigma_*^{\mathrm{test}})+2})$,  $N_2=\tilde{\Theta}(r^2)$.  Moreover, we assume that $m, N_{pt}, T_{pt}, N_1, N_2, N_3, N_4, N_{new}=O(\mathrm{poly}(d))$ and there exists $\alpha_r > 0$ that satisfies $r=\Omega(d^{\alpha_r})$.  When we fix $d$ large enough,  then there exists $\lambda_{pt},\lambda_1,\lambda_2,\eta_{pt}, \eta_1, \eta_2$ such that the prediction risk is low with probability at least 0.99 over the training data and random initialization.  Concretely, for the model trained via Algorithm~\ref{alg:pretraining}, we have that %\vspace{-0.1cm}
  \begin{equation*}
       |\mathcal{R}_{f_{\mathrm{TF}}}(\vu,\vv^*,\va^*,\vb^*)-\tau|
       =\tilde{O}(N_4^{-1/2})+\tilde{O}(m^{-1/2})+\tilde{O}\Bigg(\sqrt{\frac{r\sqrt{r}}{N_3}}\Bigg), 
  \end{equation*}%\vspace{-0.1cm}
with probability at least 0.99. 
\end{reptheorem}
The proof is given in Appendix \ref{sec:ProofOfMainTheorem}. 
Our result has the following advantages compared to prior works.

\paragraph{(i) Nonlinearity:}
While \citet{gozeten2025testtimetrainingprovablyimproves} also investigates TTT combined with ICL, their analysis is restricted to linear datasets and a transformer with linear attention.  By contrast, we consider a more complex and general setting of nonlinear single-index models with softmax attention transformers, thereby extending the analysis beyond linear transformers.

\paragraph{(ii) The adaptability to link function:}
Our framework has the flexibility of using a task-specific link function $\sigma_*^t$.  In the previous work by \citet{nishikawa2025nonlinear}, the link function is fixed throughout all the tasks (only $\beta$ varies) because the training of the relevant layer is done only in pretraining.  In contrast, the use of test-time training allows the model to adapt to the characteristics of each task, making our algorithm effective even when the link function varies from one task to another.

\paragraph{(iii) The convergence rate of predictive loss with respect to n:}
The analysis by \citet{nishikawa2025nonlinear} guarantees that the ICL risk is $o_{d}(1)$, but it provides no guarantee that the risk diminishes as $n \rightarrow \infty$.  In fact, this is an inherent consequence of adopting softmax attention, as we discussed in Remark~\ref{rem:inevitableloss}.  In contrast, our result overcomes that limitation. When $d$ is a sufficiently large constant, our theory ensures that the prediction risk can be made arbitrarily close to the inevitable noise $\tau$ by increasing the context size $N_{test}$ and the network width $m$.  This is because the increase in the test-time context length $N_3$ allows the vector $\vu$ to converge to the ground truth $\beta$.  

In addition, it is worth noting that we achieve efficient sample complexity.  Theorem~\ref{theo:main} implies that $N_{test} = \tilde{\Omega}(r^{2+ \mathrm{ge}( \sigma_*^{\mathrm{test}})})$ is required to ensure low predictive loss $\tilde{O}(\sqrt{\frac{r\sqrt{r}}{N_{test}}})$, yielding two key benefits.  First, our sample complexity does not depend on the entire dimension $d$, which means this model adapts to the low-dimensionality of $\beta$.  Second, our test-time sample complexity is independent of either $\mathrm{deg}(\sigma_*^{\mathrm{test}})$ or $\mathrm{ie}(\sigma_*^{\mathrm{test}})$, which breaks the CSQ lower bound.  Since $\mathrm{ge}(\sigma_*) \leq 2$ holds for polynomials, the required number of samples is remarkably small.
Crucially, while \citet{nishikawa2025nonlinear} also achieve sample complexity with these two features ($N_{test}=\tilde{\Omega}(r^{3\mathrm{ge}(\sigma_*)/2})$), their result guarantees convergence with respect to $d$, not $N_{test}$. This implies that their result is valid only in the asymptotic limit where $d \to \infty$. Moreover, they do not provide the convergence rate with respect to $d$. In contrast, our result holds for any sufficiently large fixed $d$.

%These advantages are further supported by numerical experiments in section ~\ref{sec:experiment}.

\subsection{Proof sketch}
The outline of the proof involves three stages. First, we achieve weak recovery, a nontrivial overlap between $\beta$ and $\vu$.  We leverage the output of the pretrained attention layer as a teacher signal, which reduces the information exponent from $\mathrm{ie}(\sigma_*)$  to $\mathrm{ge}(\sigma_*)$. We exploit this property to significantly accelerate the initial alignment of $\vu$ compared to vanilla SGD. We then show that subsequent training on the ground-truth signal $y$ ensures strong recovery ($\inner{\beta}{\vu} \geq 1-\varepsilon$). Finally, the optimization of the MLP layer allows the model to adapt to the task-specific nonlinearity $\sigma_*$.

\subsubsection{Exploiting the pretrained attention matrix}
As shown in \citet{nishikawa2025nonlinear}, the attention matrix $\Gamma^*$ after pretraining captures the $r$-dimensional Span of $\beta$:
\begin{lemma}[Informal, Correspond to Proposition 22 in \citet{nishikawa2025nonlinear}]\label{lemm:pretraining2}
     After running the pretraining in Algorithm \ref{alg:pretraining} with $T_{pt},N_{pt} = \tilde{\Omega}(r^2d^{Q+2})$, it holds that
    \begin{equation*}
        \Gamma^*\approx c_r\mathbb{E}_{\beta}[\beta\beta^\top] 
    \end{equation*}
    with high probability, where $c_r=\tilde{\Theta}(\sqrt{r})$.

\end{lemma}

    This means that $\Gamma^*$ can project vectors into the $r$-dimensional subspace $\mathrm{Supp}(\beta)$.  Therefore, by multiplying $\sqrt{r}\Gamma^*$ to the context $\mX_{N_1},\mX_{N_2}$ and the gradient (the coefficient $\sqrt{r}$ is just for adjusting the scale), we can make the problem virtually $r$-dimensional, even though the entire dimension is $d$.  This ensures the sample complexity depends only on $r$, not the whole dimension $d$.    
    
\subsubsection{Weak recovery}\label{section:weakrecovery}
First, we initialize $\vu$ by one step gradient descent using the output from the original model $g(\Gamma_*, \mX_{N_1},\vy_{N_1},w_i)$.  We do not use the signal $y$ in this process.  Intuitively, this self-distillation prevents catastrophic forgetting.  Taking the information from the original model makes the LoRA model similar to the original model, thereby preserving the desired features of it.  \\ 
To clarify why we need to use the original model as teacher, consider training the LoRA model with the true signal $y$.  Then, by calculating the gradient, we can get the following:
\begin{lemma}[Informal]\label{lemm:gradresult}
The following holds with high probability:
\begin{equation*}
    \frac{1}{2}\sqrt{r}\Gamma^* \nabla_\vu{( f_{\mathrm{IC}}(\Gamma_u,\mX_N,\vy_N,\vx)-y)^2} 
    \approx \tilde{\Theta}(\alpha m) \inner{\beta}{\vu}  \{(\sqrt{r})^{-(\mathrm{ie}(\sigma_*)-1)}+\inner{\beta}{\vu}^{2\mathrm{ie}(\sigma_*)-2}\} \beta. 
\end{equation*}
\end{lemma}
See Appendix~\ref{appx:calc} for details.  At initialization $\inner{\beta}{\vu}=\tilde{O}(1/r)$ holds, so the signal strength is $r^{-\Theta(\mathrm{ie}(\sigma_*^{\mathrm{test}}))}$, resulting in the sample complexity of $r^{\Theta(\mathrm{ie}(\sigma_*^{\mathrm{test}}))}$.
However, using the original model as teacher signal significantly reduces this data length.  As \citet{nishikawa2025nonlinear} showed, the attention layer after pretraining can compute $\inner{\beta}{\vx}^{\mathrm{ge}( \sigma_*^{\mathrm{test}})}$ in-context.  Since $\mathrm{ie}(\mathrm{He}_{\mathrm{ge}(\sigma_*)})=\mathrm{ge}(\sigma_*)$, this mechanism boosts the signal strength to $r^{-\Theta(\mathrm{ge}(\sigma_*^{\mathrm{test}}))}$.  This improved signal strength leads to the required test-time data length $N_{test}=r^{\Theta(\mathrm{ge}(\sigma_*^{\mathrm{test}}))}$, which surpasses CSQ limit.

\subsubsection{Strong recovery and link function adaptation}
Weak recovery is insufficient for reliably predicting $\inner{\beta}{\vx}$ using only $\vu$.  Therefore, we further optimize $\vu$ until we achieve strong recovery, defined as $\inner{\beta}{\vu} \geq 1-\varepsilon$ for a small error tolerance $\varepsilon>0$.  By establishing the initial alignment of $\inner{\beta}{\vu^{(1)}} > 1/\mathrm{polylog}(d)$, the model attains a $O(1/\mathrm{polylog}(d))$ learning signal. This signal strength ensures that the sample size needed for strong recovery remains unaffected by $\mathrm{ge}( \sigma_*^{\mathrm{test}})$, significantly reducing the data requirements.
Moreover, our approach achieves superior sample complexity to the prior work by \citet{lee2024neuralnetworklearnslowdimensional}.  While their work suggested a linear increase in $\inner{\beta}{\vu^{(n)}}$, yielding the required data length $N=\Theta_{\varepsilon}(\varepsilon^{-2})$, we find that the error $1-\inner{\beta}{\vu^{(n)}}$ converges to $0$ geometrically beyond a certain point.  This acceleration leads to an improved sample complexity of  $N=\Theta_{\varepsilon}(\frac{1}{\varepsilon}\log{\frac{1}{\varepsilon}})$.  See Appendix~\ref{appx:strong} for further details.  Finally, we train MLP layer to predict the link function $ \sigma_*^{\mathrm{test}}$.  Following prior work \citet{nishikawa2025nonlinear}, we analyze the empirical loss and we derive the upper bound of the predictive risk using Rademacher complexity.  See Appendix~\ref{appx:mlp} for details.

\section{Synthetic experiment}\label{sec:experiment}
We conducted a numerical experiment to examine the effectiveness of TTT compared to ICL.  To evaluate whether our theory holds with real-world settings, we employed a 2-layer GPT-2 model \citep{Radford2019LanguageMA} to learn Gaussian single-index functions.  See Appendix~\ref{appx:experiment} for further experimental details.

The ambient and intrinsic dimensions were set to $d=r=8$.  For each task $t$, the data was generated as follows: $\vx_1^t,\cdots,\vx_N^t,\vx^t \overset{\text{i.i.d.}}{\sim} \mathcal {N}(0,\mI_d) $, $\beta^t \sim \mathrm{Unif}(\mathbb{S}^{r-1})$ and $y_i^t = \sigma_*^t(\inner{\beta}{\vx_i}) = \frac{1}{\sqrt{3!}}\mathrm{He}_3(\inner{\beta}{\vx_i}) + \frac{c^t}{\sqrt{4!}}\mathrm{He}_4(\inner{\beta}{\vx_i})$, where $c^t \sim \mathrm{Unif}(-0.5,0.5)$.  We compared two settings:
\begin{enumerate}[leftmargin=10pt, labelindent=0pt]
    \item In-context learning, configured following \citet{garg2023transformerslearnincontextcase}, \citet{oko2024pretrainedtransformerefficientlylearns} and \citet{nishikawa2025nonlinear}.  The model was pretrained before calculating the predictive loss. 
    \item Test-time training:  We adapted the same pre-trained model as we used in the evaluation of ICL by fine-tuning the LoRA parameters applied to the attention and MLP layers.
\end{enumerate}

As shown in Figure~\ref{fig:result1}, ICL exhibits unexpected proficiency in adapting to varying link functions, leaving little room for TTT-based improvement.  This discrepancy is likely due to the architecture of 2-layer GPT-2 model.  As noted by \citet{oko2024pretrainedtransformerefficientlylearns}, MLP layers followed by linear attention (a structure present in GPT-2) are capable of adapting to nonlinear link functions.  It is plausible that the GPT-2 model learned to fit the task-specific link function during pretraining.  Nevertheless, we observe that TTT at least does not degrade performance even when ICL functions effectively.

This result spurred us to analyze the limitation of ICL and the true potential of TTT in more challenging tasks: we changed the distribution of link function as $y_i^t = \frac{c_3^t}{\sqrt{3!}}\mathrm{He}_3(\inner{\beta}{\vx_i}) + \frac{c_4^t}{\sqrt{4!}}\mathrm{He}_4(\inner{\beta}{\vx_i})$, where $c_3^t \sim \mathrm{Unif}(0.5,1.5), c_4^t \sim \mathrm{Unif}(-0.5,0.5)$ at test-time.  The result, summarized in Figure~\ref{fig:result2}, shows that TTT significantly outperforms ICL in this setting, with accuracy improving steadily with the growth in data length.  Taken together, these results highlight TTT as a robust strategy that matches ICL in standard settings while significantly outperforming it under distribution shifts.

\begin{figure}[htb]
    \centering % 全体を中央揃え
    
    % --- 1つ目の図 ---
    \begin{subfigure}[t]{0.35\linewidth}
        \centering
        \includegraphics[width=\linewidth]{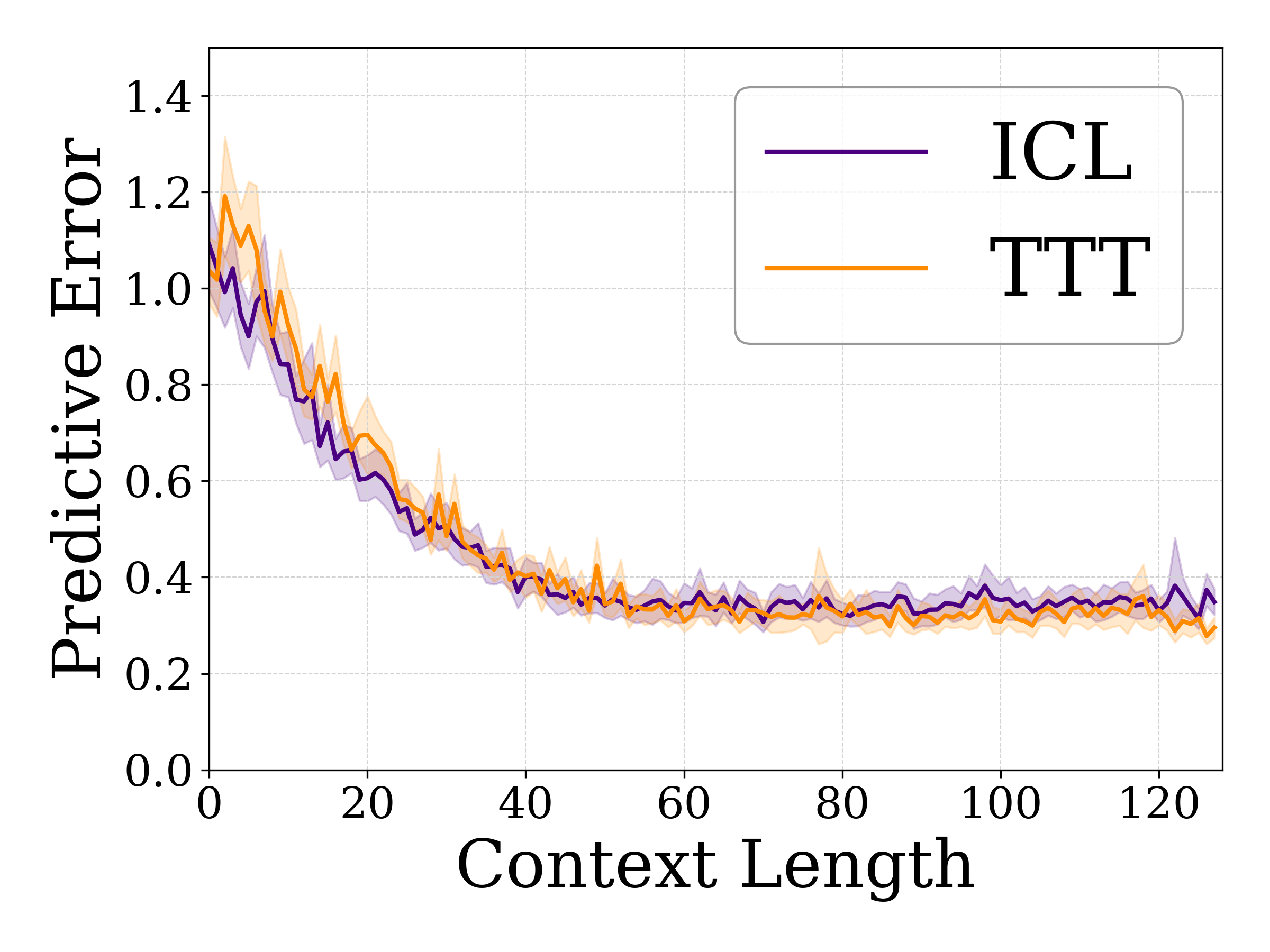}
        \caption{In-distribution setting: The distribution of the link function is consistent between pretraining and test-time.}
        \label{fig:result1}
    \end{subfigure}
    \hfill % 図の間にスペースを自動で空ける
    % --- 2つ目の図 ---
    \begin{subfigure}[t]{0.35\linewidth}
        \centering
        \includegraphics[width=\linewidth]{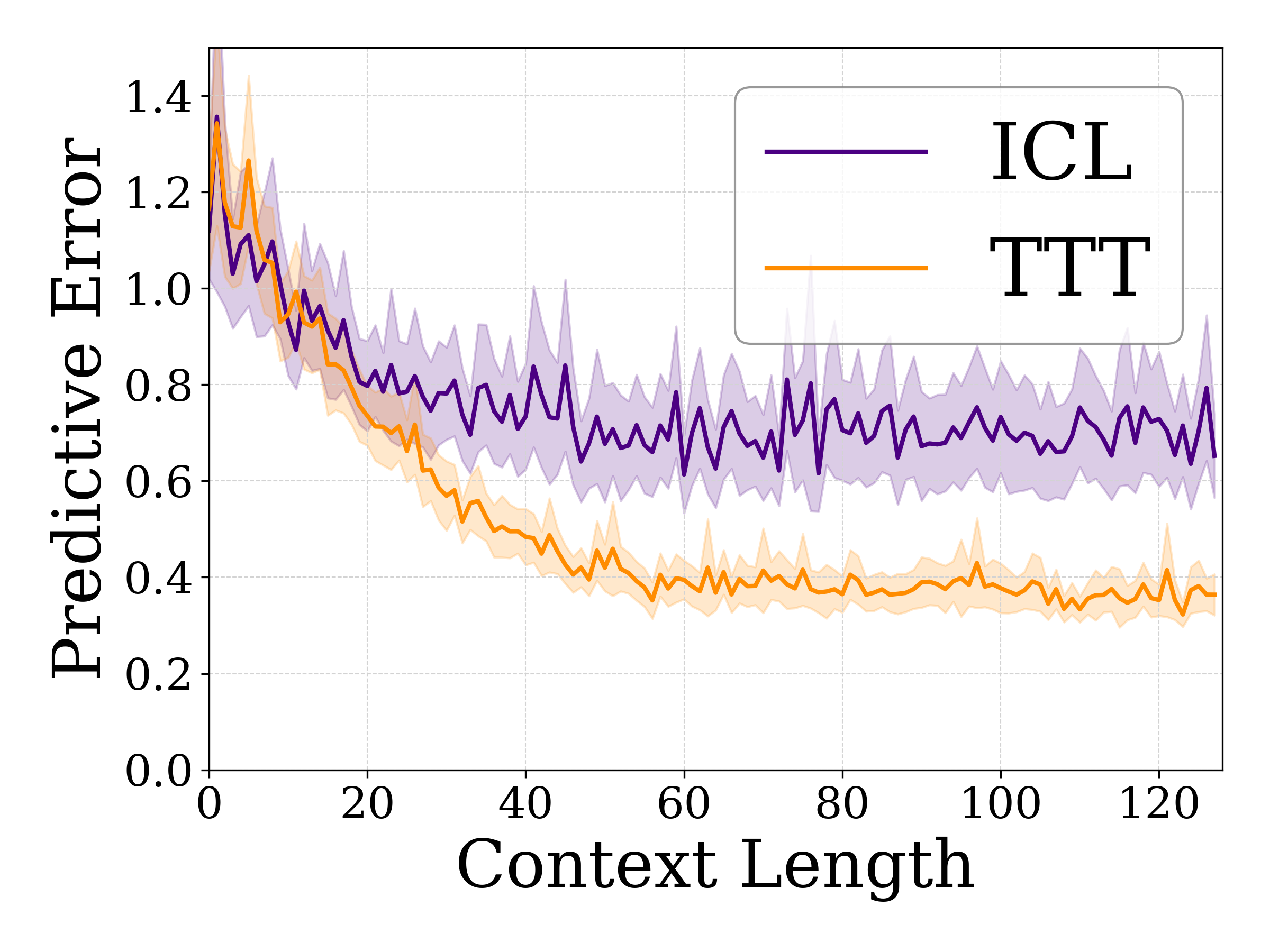}
        \caption{Out-of-distribution setting: The distribution of the link function at test-time differs from that of pre-training.}
        \label{fig:result2}
    \end{subfigure}
    
    \caption{The predictive error of In-context learning (ICL) and Test-time training (TTT) for a pretrained GPT-2 model on single-index polynomials (see Section~\ref{sec:experiment} for details).  Shaded areas: 95\% CIs.  While (a) shows that ICL effectively adapts to varying link functions for in-distribution tasks, (b) demonstrates that TTT outperforms ICL under distribution shift, showing its superior adaptability.}
    \label{fig:main}
\end{figure}

\section{Conclusion}\label{section:conclusion}
We have investigated test-time training combined with in-context learning.  We provided an upper bound of the predictive loss in terms of $m$ and $N$.  Our result shows that for a large test-time sample size $N_{test}$, the predictive loss approaches the inevitable noise $\tau$, even when $d$ remains finite.

\paragraph{Future Work and limitation} 
We identify several directions. \textbf{(i) Task Generalization}: This work assumes single-index models with polynomial link functions and consistent support of $\beta$.\footnote{Even if the support of $\beta$ differs between pretraining and test time, the predictive error still converges.  However, the sample complexity degrades because the model cannot fully exploit learned subspace.  The degree of degradation depends on $a$, the norm of the projection of $\beta_{test}$ onto the pretraining subspace, which scales as $a^{-(\mathrm{ie(\sigma_*)-1})}+r^{\Theta(\mathrm{ge}(\sigma_*))}$.}  Extending these guarantees to broader function classes is a natural next step. \textbf{(ii) Algorithmic Gap}: Our 3-stage analysis reflects distinct optimization objectives but differs from practical end-to-end training. Extending our predictive risk bounds to such settings remains an important open challenge.  \textbf{(iii) Computational overhead of TTT}: TTT requires additional computational time and memory.  The trade-off between computational cost and prediction accuracy is an important practical consideration.
    
%\section*{LLM Usage Statement}
%Large language models are used for three purposes: to proofread and polish English writing, to help us find related works, and to write the code for the synthetic experiment. 
%We did not use any LLM assistant for designing the problem settings and constructing the proofs.

%\section*{Impact Statement}
%This paper presents work whose goal is to advance the field of machine learning. There are many potential societal consequences of our work, none of which we feel must be specifically highlighted here.

\section*{Acknowledgement}
We thank Yujin Song for his insightful comment and offering us the baseline code for the numerical experiment. KK was partially supported by JST CREST (JPMJCR2015). TS was partially supported by JSPS KAKENHI (24K02905) and JST CREST (JPMJCR2115). This research is supported by the National Research Foundation, Singapore, Infocomm Media Development Authority under its Trust Tech Funding Initiative, and the Ministry of Digital Development and Information under the AI Visiting Professorship Programme (award number AIVP-2024-004). Any opinions, findings and conclusions or recommendations expressed in this material are those of the author(s) and do not reflect the views of National Research Foundation, Singapore, Infocomm Media Development Authority, and the Ministry of Digital Development and Information.

\bibliography{references_icml2026}
\bibliographystyle{icml2026}

%%%%%%%%%%%%%%%%%%%%%%%%%%%%%%%%%%%%%%%%%%%%%%%%%%%%%%%%%%%%%%%%%%%%%%%%%%%%%%%
%%%%%%%%%%%%%%%%%%%%%%%%%%%%%%%%%%%%%%%%%%%%%%%%%%%%%%%%%%%%%%%%%%%%%%%%%%%%%%%
% APPENDIX
%%%%%%%%%%%%%%%%%%%%%%%%%%%%%%%%%%%%%%%%%%%%%%%%%%%%%%%%%%%%%%%%%%%%%%%%%%%%%%%
%%%%%%%%%%%%%%%%%%%%%%%%%%%%%%%%%%%%%%%%%%%%%%%%%%%%%%%%%%%%%%%%%%%%%%%%%%%%%%%

\newpage

\tableofcontents

\appendix
\section{Definition of high probability events}
\begin{definition}
    We say that an event $A$ occurs with high probability when the following holds:
    \begin{equation*}
        1-P(A) \leq O(d^{-C_*}), 
    \end{equation*}
    where $C_*$ is a sufficiently large constant that is independent of $d$ and $r$.
\end{definition}

Note that we can redefine $C_*$ to be sufficiently large if needed.
A basic example is the Gaussian tail bound.  When $x \sim \mathcal{N}(0,1)$ and $t>0$, we have
\begin{equation*}
    P(|x|>t)\leq 2\exp(-t^2/2). 
\end{equation*}
Thus, by letting $t=\sqrt{2C_*\log{d}}$, we see that 
\begin{equation*}
    P(|x|>\sqrt{2C_*\log{d}})\leq 2d \exp(-C_*) = O(d^{-C_*}). 
\end{equation*}
In such a situation, we say that $|x|=O(\sqrt{\log{d}})$ with high probability.

When $A_1,\dots,A_M$ occurs with high probability where $M = \mathrm{poly}(d)$, then the intersection $A_1 \cap A_2\cap \dots \cap A_M$ occurs with high probability.  In this paper, we assume that $N_{pt}, T_{pt}, N_1, N_2, N_3, N_4, m = O(\mathrm{poly(d)})$.  Because of this assumption, when we take the intersection of high probability events $A_1,\dots,A_M$, $M = \mathrm{poly}(d)$ is satisfied. 
 
\section{Pretraining}
We follow the pretraining algorithm that was considered in  \citet{nishikawa2025nonlinear}.  
In the original paper, the link function $\sigma_*$ is fixed across all tasks, whereas in this paper, the link function varies depending on the tasks.  
We should take this difference into account and ensure that the pretraining algorithm works properly even in our setting.  

%\addtocounter{theorem}{-3}
\begin{replemma2}{lemm:pretraining2}{Formal}
    After running the pretraining in Algorithm \ref{alg:pretraining} with $T_{pt},N_{pt} = \tilde{\Omega}(r^2d^{Q+2})$, it holds that
    \begin{equation*}
        \Gamma^*=\frac{1}{r^{1/2}\log^{C_{\kappa}}d}\left(r\mathbb{E}_{\beta}[\beta\beta^\top]+\mN \right)
    \end{equation*}
    with high probability over the data distribution, where $\|\mN\|_F=O_d(1/\sqrt{d})$ holds, where $C_\kappa$ can be taken to be sufficiently large.
\end{replemma2}
%\addtocounter{theorem}{2}

\begin{proof}
    We only consider the difference between our settings and \citet{nishikawa2025nonlinear}.  
    See Section C of \citet{nishikawa2025nonlinear} for the original argument.

    If we fix $t$, we may apply Lemma 19 and Lemma 20 in \citet{nishikawa2025nonlinear}.  
    In addition to that, since the norm upper bound remains unchanged, we can bound the difference between the expectation and the actual value just like Eq.(C.3) and Eq.(C.4) in Lemma 21 of \citet{nishikawa2025nonlinear}.  What remains is to calculate $\mathbb{E}_{\vx,\beta,t}[yz^k\beta\vx^{\mathrm{T}}]=(\rho \sqrt{d})^{-k}\mathbb{E}_{\beta}[\mathbb{E}_{\vx,t}[\sigma_*'^t(\inner{\beta}{\vx})(\inner{\beta}{\vx})^k]\beta\beta^\top]+(\rho \sqrt{d})^{-k}k\mathbb{E}_{\beta}[\mathbb{E}_{\vx,t}[\sigma_*^t(\inner{\beta}{\vx})(\inner{\beta}{\vx})^{k-1}]\beta\beta^\top]$.  Because of the definition of $\mathrm{ie}(\sigma_*)$, $\mathbb{E}_{\vx,\beta,t}[yz^k\beta\vx^{\mathrm{T}}]=0$ when $k<Q-1$.  As we assume that $\mathbb{E}_{t}[c_Q]=\Theta(1)$, when $k = Q-1$, we see that 
\begin{align*}
    \mathbb{E}_{\beta}[\mathbb{E}_{\vx,t}[\sigma_*'^t(\inner{\beta}{\vx})(\inner{\beta}{\vx})^k]\beta\beta^\top]\asymp&\mathbb{E}_{\beta}[\mathbb{E}_{\vx}[\mathbb{E}_t[c_Q]\mathrm{He}_{Q-1}(\inner{\beta}{\vx})\inner{\beta}{\vx}^{Q-1}]\beta\beta^{\top}]\\
    \asymp& (\rho \sqrt{d})^{-(Q-1)}\mathbb{E}_{\beta}[\beta\beta^\top],  
\end{align*}
    and this is the main term that is proportional to $\mathbb{E}_{\beta}[\beta\beta^\top]$.
    This means that we can use the same argument as \citet{nishikawa2025nonlinear} with $\mathrm{ie}(\sigma_*)=Q$.
 \end{proof}

 Let $\kappa = \log^{-C_\kappa}{d}$.  In other words, $\kappa$ satisfies $\kappa = \Theta(\log^{-C_\kappa}{d})$ where $C_\kappa$ can be taken sufficiently large.

 \begin{lemma}\label{lemm:tightresultforpt}
 Suppose that the context length satisfies $N_1 = \tilde{\Omega}(r^{\mathrm{ge}(\sigma_*)+1})$.  Then, it holds with high probability that 
 \begin{equation*}
        g(\Gamma^*,\mX_{N_1},\vy_{N_1},\vx)=P_1'+P_2'\left(\frac{\inner{\vx}{\beta}}{\sqrt{r}}\right)^{\mathrm{ge}(\sigma_*)}+n_3, 
    \end{equation*}
    where $P_1' = o_d(1)$, $P_2'=\Theta_d((\log{d})^{-C_{P_2}})$ and $n_3 = o_d(P_2'r^{-\mathrm{ge}(\sigma_*)/2-1/2}\log^{-2\deg(\sigma_*)+2}d)$.
     
 \end{lemma}

 \begin{proof}
    When $N_1 = \tilde{\Omega}(r^{\mathrm{ge}(\sigma_*)+1})$, the (h.o.t.) in the proof of Proposition 11 in \citet{nishikawa2025nonlinear} can be evaluated as $o(\rho^{-1}P_2r^{-\mathrm{ge}(\sigma_*)/2-1/2}\log^{-2\deg(\sigma_*)+2}d)$ by carefully examining the term. 
    Using this fact, the same argument as the proof of Proposition 11 in \citet{nishikawa2025nonlinear} yields the result.
 \end{proof}

\section{Gradient calculation}\label{appx:calc}
For the sake of simplicity, we write $\sigma_*^{\mathrm{test}}$ as $\sigma_*$, and $f_{\mathrm{IC}}(\Gamma_u,\mX_N,\vy_N,\vx)=\sum_{j=1}^{m}a_j\sigma\left(v_j\frac{\sum_{i=1}^{N_1}y_i \mathrm{e}^{y_i/\rho} \mathrm{e}^{\vx_i^{*\top} \Gamma_u \vx /\rho}}{\sum_{i=1}^{N_1} \mathrm{e}^{y_i/\rho} \mathrm{e}^{\vx_i^{*\top} \Gamma_u \vx /\rho}}+b_j\right)$ as $f_{\mathrm{IC}}(\vx)$.
Let $\Gamma_u = \Gamma^* + uu^{\mathrm{T}}$, $\vx_i^* = \sqrt{r}\Gamma^*\vx_i$, $\beta^* = \sqrt{r}\Gamma^*\beta$. 
We define $L_m$ as $|\{j\in\{1,2,\dots,m\}\mid v_j(0)=1\}|$.
Moreover, let us define $A_i, B_i$ as $\bar\sigma_*(z)\exp((\bar\sigma_*(z)) = \sum_{i\geq0} \frac{A_i}{i!}z^i$, and $\exp((\bar\sigma_*(z)) = \sum_{i\geq0} \frac{B_i}{i!}z^i$, where 

$\bar{\sigma_*}(z):=\begin{cases}
    \frac{\sigma_*(z)}{\rho}~\left(\mathrm{if}~\left|\frac{\sigma_*(z)}{\rho}\right|\leq \frac{1}{\log d}\right)\\
    0~(\mathrm{otherwise})
\end{cases}$.

\begin{lemma}\label{lemm:gradientcalc}
Take $\rho = \Theta((\log{d})^{C_\rho})$ where $C_\rho$ is a constant sufficiently large.  Let $\alpha$ be the initial scale of the MLP parameters $a_j$.  Suppose that $\vu$ satisfies $\|u\|=C_u\leq1$ and $\inner{\vu}{\vx_i^*}=C_u\tilde{O}_{d}(1)$.  Then 
\begin{equation*}
    \sqrt{r}\Gamma^*\frac{\partial f_{\mathrm{IC}}}{\partial \vu} = \alpha L_m \left\{\sqrt{r}\Gamma^*\vx \cdot (\gamma(\vx,y) \inner{\beta^*}{\vu}+\sqrt{r}\Gamma^*\gamma(\vx,y) \inner{\vx}{\vu} \beta^*+\sqrt{r}\Gamma^*\vx \cdot n_1+ \inner{\vx}{\vu} \vn_2\right\}, 
\end{equation*}
where 
\begin{equation*}
   \gamma(\vx,y) = P_0+P_1z+\dots+P_{\mathrm{ie}(\sigma)-1}z^{\mathrm{ie}(\sigma)-1}+\dots
\end{equation*}
is satisfied with $z=\inner{\beta}{\sqrt{r}\Gamma^{*\top}\Gamma_u\vx}/\rho$, 
and $n_1=\tilde{O}\left(C_u\sqrt{\tfrac{1}{N_1}}\right)$ and $\vn_2=\tilde{O}\left(\sqrt{\tfrac{r}{N_1}}\right)$ hold.
\end{lemma}

Note that from Lemma~\ref{lemm:Gammaxandu}, the condition $\inner{\vu}{\vx_i^*}=C_u\tilde{O}_{d}(1)$ is satisfied at the initialization.

\begin{proof}

Let
\begin{align*}
    \pi_1(\vx,y) & =\frac{1}{N_1}\sum_{i=1}^{N_1}\frac{y_i}{\rho}\mathrm{e}^{y_i/\rho}\mathrm{e}^{\vx_i^{*\top}\Gamma_u\vx/\rho},  \\
    \pi_2(\vx,y)&  =\frac{1}{N_1}\sum_{i=1}^{N_1}\mathrm{e}^{y_i/\rho}\mathrm{e}^{\vx_i^{*\top}\Gamma_u\vx/\rho}, \\
    \xi_1(\vx,y)& =\frac{1}{N_1}\sum_{i=1}^{N_1}\frac{y_i}{\rho}\mathrm{e}^{y_i/\rho}\mathrm{e}^{\vx_i^{*\top}\Gamma_u\vx/\rho}\{\inner{\vx_i^*}{\vu}\vx+\inner{\vu}{\vx}\vx_i^*\}, \\
    \xi_2(\vx,y)&  =\frac{1}{N_1}\sum_{i=1}^{N_1}\mathrm{e}^{y_i/\rho}\mathrm{e}^{\vx_i^{*\top}\Gamma_u\vx/\rho}\{\inner{\vx_i^*}{\vu}\vx+\inner{\vu}{\vx}\vx_i^*\}. 
\end{align*}
Then, we can write the gradient of $f_{\mathrm{IC}}$ with respect to $\vu$ as 
\begin{equation*}
    \frac{\partial f_{\mathrm{IC}}}{\partial \vu} = \sum_{j=1}^{m}a_j\sigma'\left(v_j\frac{\sum_{i=1}^{N_1}y_i \mathrm{e}^{y_i/\rho} \mathrm{e}^{\vx_i^* \Gamma_u \vx /\rho}}{\sum_{i=1}^{N_1} \mathrm{e}^{y_i/\rho} \mathrm{e}^{\vx_i^* \Gamma_u \vx /\rho}}+b_j\right)\cdot\frac{\xi_1\pi_2-\xi_2\pi_1}{\pi_2^2}. 
\end{equation*}
First, from the assumption $\inner{\vu}{\vx_i^*}=C_u\tilde{O}_{d}(1)$ with $C_u\leq 1$ and Lemma~\ref{lemm:gaussianinner}, we can see that 
\begin{equation*}
    \vx_i^{*\top} uu^\top\vx=\inner{\vx_i^*}{\vu}\inner{\vu}{\vx}=\tilde{O}_{d}(1). 
\end{equation*}
Therefore, by taking $C_\rho$ sufficiently large, we can say
    $\vx_i^{*\top}\Gamma_u\vx/\rho=o_d(1)$
with high probability.  Hence, following the same argument as Proposition 11 in \citet{nishikawa2025nonlinear}, we can say that
the content of $\sigma'(\cdot)$ is $P_1(1+o(1))$ using the positive constant $P_1$.  Therefore, $\rho'(\cdot)=1$ if and only $v_j=1$, which means that
\begin{equation*}
    \frac{\partial f_{\mathrm{IC}}}{\partial \vu} = \alpha L_m\frac{\xi_1\pi_2-\xi_2\pi_1}{\pi_2^2}.
\end{equation*}
Next, we evaluate the term $\frac{\xi_1\pi_2-\xi_2\pi_1}{\pi_2^2}$.  
For that purpose, we first derive the expectation of each component of this term, and then bound the deviation of its actual value from the expectation.
Let $k_\rho = \frac{1}{2}(\exp(\tau/\rho)+\exp(-\tau/\rho))$, then the following holds:
\begin{align}
\mathbb{E}_{\vx_1,\zeta_1}[\mathrm{e}^{\bar\sigma_*(\inner{\beta}{\vx_1})/\rho+\zeta_1/\rho}\mathrm{e}^{\vx_1^{*\top}\Gamma_u\vx/\rho}\vx_1^*]=&k_{\rho}\mathbb{E}_{\vx_1}[\mathrm{e}^{\sigma_*(\inner{\beta}{\vx_1})}\mathrm{e}^{\vx_1^{*\top}\Gamma_u\vx/\rho}\vx_1^*]  \\
=&k_{\rho}\mathbb{E}_{\vx_1}\left[\sum_{i\geq0}\frac{B_{i}}{i!}\mathrm{He}_i(\inner{\beta}{\vx_1})\mathrm{e}^{\vx_1^{*\top}\Gamma_u\vx/\rho}\vx_1^*\right]  \\
=&\sqrt{r}\Gamma^*k_{\rho}\beta\mathbb{E}_{\vx_1}\left[\sum_{i\geq0}\frac{B_{i+1}}{i!}\mathrm{He}_i(\inner{\beta}{\vx_1})\mathrm{e}^{\vx_1^{*\top}\Gamma_u\vx/\rho}\right]  \\
&\label{eq:Exeyattention}
+\sqrt{r}\Gamma^*k_{\rho}\frac{\Gamma_u\vx}{\rho}\mathbb{E}_{\vx_1}\left[\sum_{i\geq0}\frac{B_{i}}{i!}\mathrm{He}_i(\inner{\beta}{\vx_1})\mathrm{e}^{\vx_1^{*\top}\Gamma_u\vx/\rho}\right],
\end{align}
(recall that $z=\inner{\beta}{\sqrt{r}\Gamma^{*\top}\Gamma_u\vx}/\rho$).  Here, we used Stein's lemma to derive the final equality.
The expectation in the second term of the RHS can be calculated as 
\begin{align*}
    &\mathbb{E}_{\vx_1}\left[\sum_{i\geq0}\frac{B_{i}}{i!}\mathrm{He}_i(\inner{\beta}{\vx_1})\mathrm{e}^{\vx_1^{*\top}\Gamma_u\vx/\rho}\right]\\
    =& \mathbb{E}_{\vx_1}\left[\sum_{i\geq0}\frac{B_i}{i!}\mathrm{He}_i(\inner{\beta}{\vx_1})\sum_{j\geq0}\frac{1}{j!}\left(\frac{\|\sqrt{r}\Gamma^{*\top}\Gamma_u\vx\|}{\rho}\right)^j \exp\left(\frac{\|\sqrt{r}\Gamma^{*\top}\Gamma_u\vx\|^2}{2\rho^2}\right)\mathrm{He}_j\left(\inner{\vx_1}{\frac{\sqrt{r}\Gamma^{*\top}\Gamma_u\vx}{\|\sqrt{r}\Gamma^{*\top}\Gamma_u\vx\|}}\right)\right]\\
    =&\sum_{i\geq0}\frac{B_{i}}{i!}\left(\frac{\|\sqrt{r}\Gamma^{*\top}\Gamma_u\vx\|}{\rho}\right)^i \exp\left(\frac{\|\sqrt{r}\Gamma^{*\top}\Gamma_u\vx\|^2}{2\rho^2}\right)\inner{\beta}{\frac{\sqrt{r}\Gamma^{*\top}\Gamma_u\vx}{\|\sqrt{r}\Gamma^{*\top}\Gamma_u\vx\|}}^i\\
    =& \exp\left(\frac{\|\sqrt{r}\Gamma^{*\top}\Gamma_u\vx\|^2}{2\rho^2}\right)\left(\sum_{i\geq0}\frac{B_{i}}{i!}\inner{\beta}{\sqrt{r}\Gamma^{*\top}\Gamma_u\vx/\rho}^i\right).
\end{align*}
The expectation in the first term of Eq.~\eqref{eq:Exeyattention} can be evaluated in a similar way, and we obtain: 
\begin{align}
    &\mathbb{E}_{\vx_1,\zeta_1}[\mathrm{e}^{\bar\sigma_*(\inner{\vx_1}{\beta})/\rho+\zeta_1/\rho}\mathrm{e}^{\vx_1^{*\top}\Gamma_u\vx/\rho}\vx_1^*]\notag   \\
    &=\sqrt{r}\Gamma^*k_{\rho}\beta \exp\left(\frac{\|\sqrt{r}\Gamma^{*\top}\Gamma_u\vx\|^2}{2\rho^2}\right)\left(\sum_{i\geq0}\frac{B_{i+1}}{i!}z^i\right) \notag \\
    &~~~~~+ 
    \sqrt{r}\Gamma^*\frac{k_\rho}{\rho}\Gamma_u\vx \exp\left(\frac{\|\sqrt{r}\Gamma^{*\top}\Gamma_u\vx\|^2}{2\rho^2}\right)\left(\sum_{i\geq0}\frac{B_{i}}{i!}z^i\right).\label{eq:gradtechnique}
\end{align}
Similarly, using $k_\rho' := \frac{\tau}{2\rho}(\exp(\tau/\rho)-\exp(-\tau/\rho))$, we have that 
\begin{align*}
    &\mathbb{E}_{\vx_1,\zeta_1}\left[\frac{\bar\sigma_*(\inner{\vx_1}{\beta})+\zeta_1}{\rho}\mathrm{e}^{\bar\sigma_*(\inner{\vx_1}{\beta})/\rho+\zeta_1/\rho}\mathrm{e}^{\vx_1^{*\top}\Gamma_u\vx/\rho}\vx_1^*\right]\\
    =&\sqrt{r}\Gamma^*k_{\rho}\exp\left(\frac{\|\sqrt{r}\Gamma^{*\top}\Gamma_u\vx\|^2}{2\rho^2}\right)\left(\sum_{i\geq0}\frac{A_{i+1}}{i!}z^i\right) \beta+\sqrt{r}\Gamma^*\frac{k_\rho}{\rho} \exp\left(\frac{\|\sqrt{r}\Gamma^{*\top}\Gamma_u\vx\|^2}{2\rho^2}\right)\left(\sum_{i\geq0}\frac{A_{i}}{i!}z^i\right)\Gamma_u\vx\\
    &
    +\sqrt{r}\Gamma^*k_\rho'\exp\left(\frac{\|\sqrt{r}\Gamma^{*\top}\Gamma_u\vx\|^2}{2\rho^2}\right)\left(\sum_{i\geq0}\frac{B_{i+1}}{i!}z^i\right)\beta
    +
    \sqrt{r}\Gamma^*\frac{k_\rho'}{\rho} \exp\left(\frac{\|\sqrt{r}\Gamma^{*\top}\Gamma_u\vx\|^2}{2\rho^2}\right)\left(\sum_{i\geq0}\frac{B_{i}}{i!}z^i\right)\Gamma_u\vx. 
\end{align*}

Let $F_1=\frac{1}{N_1}\sum_{i=1}^{N_1}\left[\frac{y_i}{\rho}\mathrm{e}^{y_i/\rho}\mathrm{e}^{\vx_i^{*\top}\Gamma_u\vx/\rho}\vx_i^*\right]$ and $F_2=\frac{1}{N_1}\sum_{i=1}^{N_1}[\mathrm{e}^{y_i/\rho}\mathrm{e}^{\vx_i^{*\top}\Gamma_u\vx/\rho}\vx_i^*]$.  Then, the expectation of $\xi_1$ and $\xi_2$ can be written as  $\mathbb{E}_{\vx_1,\zeta_1}[\xi_1]=\inner{[\mathbb{E}[F_1]]}{\vu}\vx+\inner{\vu}{\vx}\mathbb{E}[F_1]$ and $\mathbb{E}_{\vx_1,\zeta_1}[\xi_2]=\inner{\mathbb{E}[F_2]}{\vu}\vx+\inner{\vu}{\vx}\mathbb{E}[F_2]$ respectively.

Based on the argument above, the term $\frac{\xi_1\pi_2-\xi_2\pi_1}{\pi_2^2}$ is concentrated on
\begin{equation*}
    \left(\mathbb{E}[\pi_2]\right)^{-2} \left\{\inner{\mathbb{E}[F_1]\mathbb{E}[\pi_2]}{\vu}\vx+\inner{\vu}{\vx}\mathbb{E}[F_1]\mathbb{E}[\pi_2]-\inner{\mathbb{E}[F_2]\mathbb{E}[\pi_1]}{\vu}\vx-\inner{\vu}{\vx}\mathbb{E}[F_2]\mathbb{E}[\pi_1]\right\}.
\end{equation*}
Also, following the same argument to obtain Eq.~\eqref{eq:gradtechnique}, we have
\begin{equation*}
\mathbb{E}_{\vx_1,\zeta_1}\left[\mathrm{e}^{\bar\sigma_*(\inner{\vx_1}{\beta})/\rho+\zeta_1/\rho}\mathrm{e}^{\vx_1^{*\top}\Gamma_u\vx/\rho}\right]=k_\rho \exp\left(\frac{\|\sqrt{r}\Gamma^{*\top}\Gamma_u\vx\|^2}{2\rho^2}\right)\left(\sum_{i\geq0}\frac{B_{i}}{i!}z^i\right), 
\end{equation*}
and
\begin{align*}
     &\mathbb{E}_{\vx_1,\zeta_1}\left[\frac{\bar\sigma_*(\inner{\vx_1}{\beta})+\zeta_1}{\rho}\mathrm{e}^{\bar\sigma_*(\inner{\vx_1}{\beta})/\rho+\zeta_1/\rho}\mathrm{e}^{\vx_1^{*\top}\Gamma_u\vx/\rho}\right]\\
     =&k_\rho \exp\left(\frac{\|\sqrt{r}\Gamma^{*\top}\Gamma_u\vx\|^2}{2\rho^2}\right)\left(\sum_{i\geq0}\frac{A_{i}}{i!}z^i\right)
     +
     k_\rho'\exp\left(\frac{\|\sqrt{r}\Gamma^{*\top}\Gamma_u\vx\|^2}{2\rho^2}\right)\left(\sum_{i\geq0}\frac{B_{i}}{i!}z^i\right). 
\end{align*}

Using these result, we obtain that  
\begin{align*}
    &\left(\mathbb{E}_{\vx_1,\zeta_1}[\pi_2]\right)^{-2}(\mathbb{E}_{x_1,\zeta_1}[F_1]\mathbb{E}_{x_1,\zeta_1}[\pi_2]-\mathbb{E}_{x_1,\zeta_1}[F_2]\mathbb{E}_{x_1,\zeta_1}[\pi_1])\\
    =&\sqrt{r}\Gamma^*\left(k_\rho\sum_{i\geq0}\frac{B_{i}}{i!}z^i\right)^{-2} \times \\
    &  \textstyle 
    \left\{\left[k_{\rho}\!\left(\sum\limits_{i\geq0}\tfrac{A_{i+1}}{i!}z^i\!\right)\beta
    \!+\!\frac{k_\rho}{\rho}\left(\sum\limits_{i\geq0}\tfrac{A_{i}}{i!}z^i\right)\Gamma_u\vx
    +k_\rho'\!\left(\!\sum\limits_{i\geq0}\tfrac{B_{i+1}}{i!}z^i\!\right)\beta
    \!+\!\frac{k_\rho'}{\rho}\!\left(\!\sum\limits_{i\geq0}\frac{B_{i}}{i!}z^i\!\right)\Gamma_u\vx \right]\cdot
     k_\rho\!\left(\!\sum\limits_{i\geq0}\tfrac{B_{i}}{i!}z^i\!\right) \right. \\
    & \textstyle  \left. 
    -\left[k_{\rho}\beta \left(\sum\limits_{i\geq0}\frac{B_{i+1}}{i!}z^i \right)+\frac{k_\rho}{\rho}\Gamma_u\vx \left(\sum\limits_{i\geq0}\frac{B_{i}}{i!}z^i\right)\right]\cdot \left[k_\rho \left(\sum\limits_{i\geq0}\frac{A_{i}}{i!}z^i\right)+k_\rho'\left(\sum\limits_{i\geq0}\frac{B_{i}}{i!}z^i\right)\right]\right\}\\
    =&\textstyle \sqrt{r}\Gamma^*\left(k_\rho\sum\limits_{i\geq0}\frac{B_{i}}{i!}z^i\right)^{-2}\left\{\left[k_{\rho}\left(\sum\limits_{i\geq0}\frac{A_{i+1}}{i!}z^i\right) \beta+k_\rho'\left(\sum\limits_{i\geq0}\frac{B_{i+1}}{i!}z^i\right)\beta\right]\cdot k_\rho\left(\sum\limits_{i\geq0}\frac{B_{i}}{i!}z^i\right)\right.\\
    &\textstyle \left. -k_{\rho}\beta \left(\sum\limits_{i\geq0}\frac{B_{i+1}}{i!}z^i\right)\cdot \left[k_\rho \left(\sum\limits_{i\geq0}\frac{A_{i}}{i!}z^i\right)+k_\rho'\left(\sum\limits_{i\geq0}\frac{B_{i}}{i!}z^i \right) \right] \right\}.
\end{align*}
    Therefore, we can expand this value using $z$. Specifically,
\begin{equation*}
    \left(\mathbb{E}_{x_1,\zeta_1}[\pi_2]\right)^{-2}(\mathbb{E}_{x_1,\zeta_1}[F_1]\mathbb{E}_{x_1,\zeta_1}[\pi_2]-\mathbb{E}_{x_1,\zeta_1}[F_2]\mathbb{E}_{x_1,\zeta_1}[\pi_1])=\sqrt{r}\Gamma^*\left(P_0+P_1z+\dots\right)\beta
\end{equation*}
    holds.
    By letting $\gamma(\vx,y)=P_0+P_1z+\dots$, we obtain
\begin{equation*}
    \frac{\mathbb{E}[\xi_1]\mathbb{E}[\pi_2]-\mathbb{E}[\xi_2]\mathbb{E}[\pi_1]}{\mathbb{E}[\pi_2]^2}=\vx \cdot \left(\gamma(\vx,y) \inner{\sqrt{r}\Gamma^*\beta}{\vu} \right)+\gamma(\vx,y) \inner{\vx}{\vu} \sqrt{r}\Gamma^*\beta. 
\end{equation*}
    
Following the same argument as Lemma 20 in \citet{nishikawa2025nonlinear} yields the order of $P_i$.

Next we deal with the deviation from the expectation.
Using the same technique as Lemma 18 in \citet{nishikawa2025nonlinear}, we have
\begin{equation*}
    \frac{1}{N_1}\sum_{i=1}^{N_1}\frac{y_i}{\rho}\mathrm{e}^{y_i/\rho}\mathrm{e}^{\vx_i^{*\top}\Gamma_u\vx/\rho}=\mathbb{E}_{\vx_1,\zeta_1}\left[\frac{\bar\sigma_*(\inner{\vx_1}{\beta})+\zeta_1}{\rho}\mathrm{e}^{\bar\sigma_*(\inner{\vx_1}{\beta})/\rho+\zeta_1/\rho}\mathrm{e}^{\vx_1^{*\top}\Gamma_u\vx/\rho}\right]+\tilde{O}(N_1^{-1/2}), 
\end{equation*}
and
\begin{equation*}
    \frac{1}{N_1}\sum_{i=1}^{N_1}\mathrm{e}^{y_i/\rho}\mathrm{e}^{\vx_i^{*\top}\Gamma_u\vx/\rho}=\mathbb{E}_{\vx_1,\zeta_1}\left[\mathrm{e}^{\bar\sigma_*(\inner{\vx_1}{\beta})/\rho+\zeta_1/\rho}\mathrm{e}^{\vx_1^{*\top}\Gamma_u\vx/\rho}\right]+\tilde{O}(N_1^{-1/2}). 
\end{equation*}
Moreover, noting that $\sqrt{r}\Gamma^*\vx_i^*=\tilde{O}(\sqrt{r})$ 
(see Lemma~\ref{lemm:Gammax}), from Lemma~\ref{lemm:difvector2}, we have
\begin{align*}
    &\frac{1}{N_1}\sum_{i=1}^{N_1}\sqrt{r}\Gamma^*\frac{y_i}{\rho}\mathrm{e}^{y_i/\rho}\mathrm{e}^{\vx_i^{*\top}\Gamma_u\vx/\rho}\vx_i^*\\  &=\sqrt{r}\Gamma^*\mathbb{E}_{\vx_1,\zeta_1}\left[\frac{\bar\sigma_*(\inner{\vx_1}{\beta})+\zeta_1}{\rho}\mathrm{e}^{\bar\sigma_*(\inner{\vx_1}{\beta})/\rho+\zeta_1/\rho}\mathrm{e}^{\vx_1^{*\top}\Gamma_u\vx/\rho}\vx_1^*\right]+\tilde{O}(r^{1/2}N_1^{-1/2}), 
\end{align*}
and
\begin{equation*}
    \frac{1}{N_1}\sum_{i=1}^{N_1}\sqrt{r}\Gamma^*\mathrm{e}^{y_i/\rho}\mathrm{e}^{\vx_i^{*\top}\Gamma_u\vx/\rho}\vx_i^*=\sqrt{r}\Gamma^*\mathbb{E}_{\vx_1,\zeta_1}\left[\mathrm{e}^{\bar\sigma_*(\inner{\vx_1}{\beta})/\rho+\zeta_1/\rho}\mathrm{e}^{\vx_1^{*\top}\Gamma_u\vx/\rho}\vx_1^*\right]+\tilde{O}(r^{1/2}N_1^{-1/2}).
\end{equation*}
Hence, we obtain that 
\begin{equation*}
    \sqrt{r}\Gamma^*\pi_2^{-2}\left(F_1\pi_2-F_2\pi_1\right) = \frac{\poly(z)\beta+\delta_2}{\poly(z)+\delta_1}, 
\end{equation*}
for some $\delta_1$ and $\delta_2$ which satisfy $\delta_1 = \tilde{O}(1/\sqrt{N_1})$ and $\delta_2 = \tilde{O}\left(\sqrt{\frac{r}{N_1}}\right)$ with high probability.
Let $\vn_2=\sqrt{r}\Gamma^*\{\pi_2^{-2}\left(F_1\pi_2-F_2\pi_1\right)-\gamma(\vx,y)\beta\}$, then $\vn_2=\tilde{O}(\sqrt{\frac{r}{N_1}})$ holds.

Also, noting that $\inner{\vx_i^*}{\vu}=C_u\tilde{O}_{d}(1)$ with high probability,
it holds that 
\begin{align*}
    &\frac{1}{N_1}\sum_{i=1}^{N_1}\frac{y_i}{\rho}\mathrm{e}^{y_i/\rho}\mathrm{e}^{\vx_i^{*\top}\Gamma_u\vx/\rho}\inner{\vx_i^*}{\vu}\\
    =&\mathbb{E}_{\vx_1,\zeta_1}\left[\frac{\bar\sigma_*(\inner{\vx_1}{\beta})+\zeta_1}{\rho}\mathrm{e}^{\bar\sigma_*(\inner{\vx_1}{\beta})/\rho+\zeta_1/\rho}\mathrm{e}^{\vx_1^{*\top}\Gamma_u\vx/\rho}\inner{\vx_1^*}{\vu}\right]+\tilde{O}(C_uN_1^{-1/2}), 
\end{align*}
and
\begin{equation*}
    \frac{1}{N_1}\sum_{i=1}^{N_1}\mathrm{e}^{y_i/\rho}\mathrm{e}^{\vx_i^{*\top}\Gamma_u\vx/\rho}\inner{\vx_i^*}{\vu}=\mathbb{E}_{\vx_1,\zeta_1}\left[\mathrm{e}^{\bar\sigma_*(\inner{\vx_1}{\beta})/\rho+\zeta_1/\rho}\mathrm{e}^{\vx_1^{*\top}\Gamma_u\vx/\rho}\inner{\vx_1^*}{\vu}\right]+\tilde{O}(C_uN_1^{-1/2}). 
\end{equation*}
Therefore, letting $n_1 = \pi_2^{-2}\left\{\inner{F_1}{\vu}\pi_2-\inner{F_2}{\vu}\pi_1\right\} - \gamma(\vx,y)\inner{\beta^*}{\vu}$, we obtain that 
$n_1 = \tilde{O}(C_u\sqrt\frac{1}{N_1})$.

In summary, we obtain that 
\begin{align*}
    & \sqrt{r}\Gamma_*\frac{\xi_1\pi_2-\xi_2\pi_1}{\pi_2^2}=\sqrt{r}\Gamma_*[\vx \cdot (\gamma(\vx,y) \inner{\beta^*}{\vu})+\gamma(\vx,y) \inner{\vx}{\vu} \beta^*]+\sqrt{r}\Gamma^*\vx \cdot n_1+ \inner{\vx}{\vu} \vn_2, 
\end{align*}
where $n_1=\tilde{O}\left(C_u\sqrt\frac{1}{N_1}\right)$ and $\vn_2=\tilde{O}\left(\sqrt\frac{r}{N_1}\right)$. 
\end{proof}

\begin{lemma}\label{lemm:subweibull}
    The norm of the vector $\|\sqrt{r}\Gamma^*\frac{\partial f_{\mathrm{IC}}}{\partial \vu}\|$ has sub-Weibull tail with tail index $2/(P+2)$.
\end{lemma}

\begin{proof}
    It suffices to show that  $\left\|\sqrt{r}\Gamma_*\frac{\xi_1\pi_2-\xi_2\pi_1}{\pi_2^2}\right\|$ has sub-Weibull tail.  Note that $\frac{\xi_1\pi_2-\xi_2\pi_1}{\pi_2^2} = \frac{\xi_1}{\pi_2}-\frac{\xi_2}{\pi_2}\frac{\pi_1}{\pi_2}$.  Therefore we examine $\frac{\xi_1}{\pi_2}, \frac{\xi_2}{\pi_2}$ and $\frac{\pi_1}{\pi_2}$.  Let us define
    \begin{equation*}
        p_i = \frac{\mathrm{e}^{y_i/\rho}\mathrm{e}^{\vx_i^{*\top}\Gamma_u\vx/\rho}}{\sum_{i=1}^{N_1}\mathrm{e}^{y_i/\rho}\mathrm{e}^{\vx_i^{*\top}\Gamma_u\vx/\rho}},
    \end{equation*}
    then the following holds:
    \begin{equation*}
        \frac{\xi_1}{\pi_2}=\sum_{i=1}^{N_1}p_i\frac{y_i}{\rho}\{\inner{\vx_i^*}{\vu}\vx+\inner{\vu}{\vx}\vx_i^*\}
    \end{equation*}
    \begin{equation*}
        \frac{\xi_2}{\pi_2}=\sum_{i=1}^{N_1}p_i\{\inner{\vx_i^*}{\vu}\vx+\inner{\vu}{\vx}\vx_i^*\}
    \end{equation*}
    \begin{equation*}
        \frac{\pi_1}{\pi_2}=\sum_{i=1}^{N_1}p_i\frac{y_i}{\rho}
    \end{equation*}
    By the triangle inequality, we have
    \begin{equation*}
        \left\|\frac{\xi_1}{\pi_2}\right\|\leq \sum_{i=1}^{N_1}p_i\frac{y_i}{\rho}\|\inner{\vx_i^*}{\vu}\vx+\inner{\vu}{\vx}\vx_i^*\|.
    \end{equation*}
    
    By further applying the Cauchy-Schwarz inequality, we have
    \begin{equation*}
        \left\|\frac{\xi_1}{\pi_2}\right\|\leq \left(\sum_{i=1}^{N_1}p_i^2\right)^{1/2}\left(\sum_{i=1}^{N_1}\frac{y_i^2}{\rho^2}\|\inner{\vx_i^*}{\vu}\vx+\inner{\vu}{\vx}\vx_i^*\|^2\right)^{1/2}.
    \end{equation*}
    Note that $\sum_{i=1}^{N_1}p_i^2\leq1$.  Next, we can see that 
    \begin{equation*}
        y_i^2\|\inner{\vx_i^*}{\vu}\vx+\inner{\vu}{\vx}\vx_i^*\|^2\leq 2y_i^2 \|\inner{\vx_i^*}{\vu}\vx\|^2+2y_i^2\|\inner{\vu}{\vx}\vx_i^*\|^2
    \end{equation*}
    Since $y_i = \sigma_*(\inner{\beta}{\vx})+\zeta_i,$ the tail probability of each term is almost equal to that of $\mathrm{poly}(t)$ where $t\sim\mathcal{N}(0,1)$, which means $y_i^2 \|\inner{\vx_i^*}{\vu}\vx\|^2$ and $y_i^2\|\inner{\vu}{\vx}\vx_i^*\|^2$ have sub-Weibull tail with tail index $2/(2P+4)$.  This means that $\left\|\frac{\xi_1}{\pi_2}\right\|$ has sub-Weibull tail with tail index $2/(P+2)$.  By applying the same argument, you can see that $\frac{\xi_2}{\pi_2}\frac{\pi_1}{\pi_2}$ has sub-Weibull tail.  Therefore, considering that $\|\sqrt{r}\Gamma^*\|_2$ is a finite constant, $\|\sqrt{r}\Gamma^*\frac{\partial f_{\mathrm{IC}}}{\partial \vu}\|$ has sub-Weibull tail with tail index $2/(P+2)$.
\end{proof}

%\addtocounter{theorem}{-5}%数に注意
\begin{replemma}{lemm:gradresult}{formal}
It holds that 
\begin{align*}
    &\frac{1}{2}\sqrt{r}\Gamma^* \nabla_\vu{(f_{\mathrm{IC}}(\vx)-y)^2} \\
    =& \Theta(\alpha m) \inner{\beta}{\vu}  \frac{\beta}{\kappa^{\mathrm{ie}(\sigma_*)+1}\rho^{\mathrm{ie}(\sigma_*)}}\{(\kappa \sqrt{r})^{-(\mathrm{ie}(\sigma_*)-1)}(1+O(1/\sqrt{d}))+\inner{\beta}{\vu}^{2\mathrm{ie}(\sigma_*)-2}(1+O(1/\sqrt{d}))\}\\
    &+
    \tilde{O}(\alpha^2m^2C_u\sqrt{r})
    +
    \tilde{O}\left(\alpha m C_u\sqrt{\frac{r}{N_1}}\right)+\nu, 
\end{align*}
with high probability, where $\nu$ satisfies $\nu=\tilde{O}(\alpha m C_u\sqrt{r})$ with high probability and $\mathbb{E}_{\vx}[\nu]=0$.  Moreover, $\|\nu\|$ has sub-Weibull tail. 
\end{replemma}
%\addtocounter{theorem}{4}

\begin{proof}
Note that 
    \begin{equation}\label{eq:nablasqfICydecomp}
    \frac{1}{2}\sqrt{r}\Gamma^* \nabla_\vu{(f_{\mathrm{IC}}(\vx)-y)^2} =
    f_{\mathrm{IC}}(\vx)\cdot\sqrt{r}\Gamma^*\nabla_\vu{f_{\mathrm{IC}}}(\vx)-y\cdot\sqrt{r}\Gamma^*\nabla_\vu{f_{\mathrm{IC}}}(\vx). 
    \end{equation}
First, we analyze the first term of the RHS.
Since $y_i=\tilde{O}_p(1)$, we have that 
    \begin{equation*}
    |f_{\mathrm{IC}}(\vx)|\leq \alpha m\left|\frac{\sum_{i=1}^{N_1}y_i \mathrm{e}^{y_i/\rho} \mathrm{e}^{\vx_i^* \Gamma_u \vx /\rho}}{\sum_{i=1}^{N_1} \mathrm{e}^{y_i/\rho} \mathrm{e}^{\vx_i^* \Gamma_u \vx /\rho}}\right|=\tilde{O}_p(\alpha m). 
    \end{equation*}
Moreover, from Lemma~\ref{lemm:gradientcalc}, it holds that 
    \begin{equation*}
        \sqrt{r}\Gamma^*\frac{\partial f_{\mathrm{IC}}}{\partial \vu} = \alpha L_m \left\{\sqrt{r}\Gamma^*\vx \cdot (\gamma(\vx,y) \inner{\beta^*}{\vu}+\sqrt{r}\Gamma^*\gamma(\vx,y) \inner{\vx}{\vu} \beta^*+\sqrt{r}\Gamma^*\vx \cdot n_1+ \inner{\vx}{\vu} \vn_2\right\},
    \end{equation*}
where $\sqrt{r}\Gamma^*\vx = \tilde{O}(\sqrt{r})$ (See Lemma~\ref{lemm:Gammax})
    and $\gamma(\vx,y)=\tilde{O}(1)$ with high probability, which means that $ \sqrt{r}\Gamma^*\frac{\partial f_{\mathrm{IC}}}{\partial \vu} = \tilde{O}_p(\alpha m C_u\sqrt{r})$.
    Therefore, we obtain 
    \begin{equation*}
        f_{\mathrm{IC}}(\vx)\cdot\sqrt{r}\Gamma^*\nabla_\vu{f_{\mathrm{IC}}}(\vx) = \tilde{O}_p(\alpha^2 m^2 C_u\sqrt{r}). 
    \end{equation*}
    Next, we evaluate the second term of ~\eqref{eq:nablasqfICydecomp}. 
    For that purpose, let 
     \begin{equation*}
        \nu=y\cdot\sqrt{r}\Gamma^*\nabla_\vu{f_{\mathrm{IC}}(\vx)}-\mathbb{E}_{\vx}[y\cdot\sqrt{r}\Gamma^*\nabla_\vu{f_{\mathrm{IC}}}(\vx)]. 
    \end{equation*}
    By noticing that $y=\tilde{O}_p(1)$ and $ \sqrt{r}\Gamma^*\nabla_{\vu} f_{\mathrm{IC}}(\vx) = \tilde{O}(\alpha m C_u\sqrt{r})$ with high probaibility,
    we have
    \begin{equation*}
        \nu=\tilde{O}_p(\alpha m C_u\sqrt{r}). 
    \end{equation*}
    Moreover, $\mathbb{E}_{\vx}[\nu]=0$ by definition, and $\|\nu\|$ has sub-Weibull tail from Lemma~\ref{lemm:subweibull}.  
    Thus, what remains is to calculate $\mathbb{E}_{\vx}[y\cdot\sqrt{r}\Gamma^*\nabla_\vu{f_{\mathrm{IC}}}(\vx)]$:
    \begin{align}
        &\mathbb{E}_{\vx}[y\alpha L_m \left\{\sqrt{r}\Gamma^*\vx \cdot \gamma(\vx,y) \inner{\beta^*}{\vu}+\sqrt{r}\Gamma^*\gamma(\vx,y) \inner{\vx}{\vu} \beta^*\right\}]\notag\\
        =&\alpha L_m \inner{\beta^*}{\vu}\mathbb{E}_{\vx}[\sqrt{r}\Gamma^*\vx\sigma_*(\inner{\beta}{\vx})(P_0+P_1z+\dots)] \notag \\ 
        & +\alpha L_m\sqrt{r}\Gamma^* \beta^* \mathbb{E}_{\vx}[\inner{\vu}{\vx\cdot\sigma_*(\inner{\beta}{\vx})(P_0+P_1z+\dots)}]. \label{eq:maingradient}
    \end{align}
    Now, we calculate $\mathbb{E}_{\vx}[\sqrt{r}\Gamma^*\vx\sigma_*(\inner{\beta}{\vx})z^k]$ in order to investigate $\mathbb{E}_{\vx}[\sqrt{r}\Gamma^*\vx\sigma_*(\inner{\beta}{\vx})(P_0+P_1z+\dots)]$:
    \begin{align}
        &\mathbb{E}_{\vx}[\sqrt{r}\Gamma^*\vx\sigma_*(\inner{\beta}{\vx})z^k] \notag \\
        =& \sqrt{r}\Gamma^* \left\{ \mathbb{E}_{\vx}[\sigma_*'(\inner{\beta}{\vx})z^k]\beta+\mathbb{E}_{\vx}[\sigma_*(\inner{\beta}{\vx}kz^{k-1})\Gamma_u^\top\beta]
        \right\} \label{eq:gradientzetak}
    \end{align} 
    holds from Stein's lemma, and when $k<\mathrm{ie}(\sigma_*)-1$, $\mathbb{E}_{\vx}[\sigma_*'(\inner{\beta}{\vx})z^k] = 0$ from the definition of the information exponent.  
    When $k=\mathrm{ie}(\sigma_*)-1$, we have that 
    \begin{align*}
        \sqrt{r}\Gamma^* \mathbb{E}_{\vx}[\sigma_*'(\inner{\beta}{\vx})z^k]
        \asymp&\sqrt{r}\Gamma^*\beta \left(\frac{\inner{\beta}{\sqrt{r}\Gamma_u^{\top}\Gamma^*\beta}}{\rho}\right)^{\mathrm{ie}(\sigma_*)-1},
    \end{align*} 
    and from Lemma~\ref{lemm:Gammabeta}, it holds that 
    \begin{equation*}
       \beta^*= \sqrt{r}\Gamma^*\beta=\kappa^{-1}(\beta+O(1/\sqrt{d})),
    \end{equation*}
    and
    \begin{align*}
        \inner{\beta}{\sqrt{r}\Gamma_u^{\top}\Gamma^*\beta}=&\inner{\Gamma_u\beta}{\sqrt{r}\Gamma_*\beta}\\
        =&\inner{\Gamma^*\beta}{\sqrt{r}\Gamma_*\beta}+\inner{\beta}{\vu}\inner{\vu}{\sqrt{r}\Gamma^*\beta}\\
        =&(\sqrt{r}\kappa^2)^{-1}(1+O(1/\sqrt{d}))+\kappa^{-1}\inner{\beta}{\vu}\{\inner{\beta}{\vu}+O(1/\sqrt{d})\}\\
        =&(\sqrt{r}\kappa^2)^{-1}(1+O(1/\sqrt{d}))+\kappa^{-1}\inner{\beta}{\vu}^2\{1+O(1/\sqrt{d})\} 
    \end{align*}
    with high probability.  By ignoring small terms, this yields that  
    \begin{align*}
        &\sqrt{r}\Gamma^*\beta \left(\frac{\inner{\beta}{\sqrt{r}\Gamma_u^{\top}\Gamma_*\beta}}{\rho}\right)^{\mathrm{ie}(\sigma_*)-1}\\
        &=\frac{P_{\mathrm{ie}(\sigma_*)-1}\beta}{\kappa^{\mathrm{ie}(\sigma_*)}\rho^{\mathrm{ie}(\sigma_*)-1}}\{(\kappa \sqrt{r})^{-(\mathrm{ie}(\sigma_*)-1)}(1+O(1/\sqrt{d}))+\inner{\beta}{\vu}^{2\mathrm{ie}(\sigma_*)-2}(1+O(1/\sqrt{d})\}.
    \end{align*}
    We can see that the other term in Eq.~\eqref{eq:gradientzetak} is negligible.  
    Using $P_{\mathrm{ie}(\sigma_*)-1}=\Theta(1/\rho)$, plugging this result into Eq.~\eqref{eq:maingradient} yields
    \begin{align*}
        &\mathbb{E}_{\vx}[y\alpha L_m \left\{\sqrt{r}\Gamma^*\vx \cdot \gamma(\vx,y) \inner{\beta^*}{\vu}+\sqrt{r}\Gamma^*\gamma(\vx,y) \inner{\vx}{\vu} \beta^*\right\}]\\
        =&\alpha L_m \inner{\beta}{\vu}  \frac{\beta}{\kappa^{\mathrm{ie}(\sigma_*)+1}\rho^{\mathrm{ie}(\sigma_*)}}\{(\kappa \sqrt{r})^{-(\mathrm{ie}(\sigma_*)-1)}(1+O(1/\sqrt{d}))+\inner{\beta}{\vu}^{2\mathrm{ie}(\sigma_*)-2}(1+O(1/\sqrt{d})\} 
    \end{align*}
    with high probability.  Finally, we evaluate the term $\mathbb{E}_{\vx}[\alpha L_m \sigma_*(\inner{\beta}{\vx})(\sqrt{r}\Gamma^*\vx \cdot n_1+ \inner{\vx}{\vu} \vn_2)]$.  From Corollary 17 of \citet{oko2024pretrainedtransformerefficientlylearns}, we have $\sigma_*(\inner{\beta}{\vx})=O_p(\log^{\deg(\sigma_*)/2}{d})$.  
    Moreover, from Lemma~\ref{lemm:gradientcalc}, it holds that 
    $n_1=\tilde{O}(C_u\sqrt\frac{1}{N_1})$ and $\vn_2=\tilde{O}(\sqrt\frac{r}{N_1})$ with high probability.
    Therefore, we arrive at 
    \begin{equation*}
        \mathbb{E}_{\vx}[\alpha L_m \sigma_*(\inner{\beta}{\vx})(\sqrt{r}\Gamma^*\vx \cdot n_1+ \inner{\vx}{\vu} \vn_2)] = \tilde{O}_p\left(\alpha m C_u\sqrt{\frac{r}{N_1}}\right).
    \end{equation*}
    Combining all results above completes the proof.
\end{proof}

The following lemma will be useful in the analysis of weak recovery (the next section).

\begin{lemma}\label{lemm:p0_small}
When $\mathrm{ge}(\sigma_*)=2$, $\gamma(\vx,y)=\tilde{O}(1/\sqrt{r})$ holds with high probability. 
\end{lemma}

\begin{proof}
Recall that
\begin{align*}
    \gamma(\vx,y)&=P_0+P_1z+\dots\\
    =&\left(k_\rho\sum_{i\geq0}\frac{B_{i}}{i!}z^i\right)^{-2}\left\{
    \left[k_{\rho}\left(\sum_{i\geq0}\frac{A_{i+1}}{i!}z^i\right)+k_\rho'\left(\sum_{i\geq0}\frac{B_{i+1}}{i!}z^i\right)\right]\cdot k_\rho\left(\sum_{i\geq0}\frac{B_{i}}{i!}z^i\right)\right.\\
    &\left. -k_{\rho} \left(\sum_{i\geq0}\frac{B_{i+1}}{i!}z^i\right)\cdot\left[k_\rho \left(\sum_{i\geq0}\frac{A_{i}}{i!}z^i\right)+k_\rho'\left(\sum_{i\geq0}\frac{B_{i}}{i!}z^i\right)\right] \right\}. 
\end{align*}
    By expanding the RHS in powers of $z$ and comparing their coefficients, we obtain 
    \begin{align*}
        P_0&=\frac{(k_{\rho}A_0+k_{\rho}'B_0)\cdot k_{\rho}B_1-k_{\rho}B_0(k_\rho A_1+k_\rho' B_1)}{(k_{\rho}B_0)^2}\\
        &=\frac{A_0B_1-B_0A_1}{B_0^2}.
    \end{align*}
    When $\mathrm{ge}(\sigma_*)$ is 2, $\sigma_*$ is even, and thus $\exp(\bar\sigma_*)$ and $\bar\sigma_*\exp(\bar\sigma_*)$ are also even.  This means that $\exp(\bar\sigma_*)$ and $\bar\sigma_* \exp(\bar\sigma_*)$ can be expanded in polynomial of $z^2$.  This means that the coefficients of $z$ in the hermite expansion of $\exp(\bar\sigma_*)$ and $\bar\sigma_* \exp(\bar\sigma_*)$ are $0$, namely $A_1 = B_1=0$, which yields $P_0=0$.
    Hence, we have  
    \begin{equation*}
        \gamma(\vx,y)=P_1z+P_2z^2+\dots, 
    \end{equation*}
    where $P_1 = \tilde{O}(1)$ and $z=\tilde{O}_p(1/\sqrt{r})$, which yields the assertion.
\end{proof}

\section{One step gradient descent for weak recovery}\label{appx:weak}

\begin{lemma}\label{lemm:onestepgdvector}
    Set $C_u = 1/\sqrt{r}$, $N_1 = \tilde{\Omega}(r^{\mathrm{ge}(\sigma_*)+1})$ and $N_{new} = \tilde{\Omega}(r^{\mathrm{ge}(\sigma_*)+2})$.  Let $w_i\overset{\text{i.i.d.}}{\sim} \mathcal {N}(0,\mI_d) $ for $i=1,2,\dots,N_{new}$ and $\vh=\frac{1}{2}\sqrt{r}\Gamma^*   \frac{1}{N_{new}}\sum_{i=1}^{N_{new}}\nabla_\vu{(f_{\mathrm{IC}}(\vw_i)-(g(\Gamma^*,\vw_i)-b))^2}$, then
    \begin{align*}
    \vh&= P_2'\Theta(\alpha_1 m) \inner{\beta}{\vu}  \frac{\beta}{\kappa^{2\mathrm{ge}(\sigma_*)}\rho^{\mathrm{ge}(\sigma_*)}}( \sqrt{r})^{-(2\mathrm{ge}(\sigma_*)-1)}(1+O(1/\sqrt{d}))\\
    &+\tilde{O}(\alpha_1^2 m^2 r^{-1/2}r^{-\frac{\mathrm{ge}(\sigma_*)-1}{2}})
    +\tilde{O}\left(\alpha_1 m \sqrt{\frac{r^{-\mathrm{ge}(\sigma_*)}}{N_1}}\right)
    +\tilde{O}\left(\alpha_1 m \sqrt{\frac{r^{-\mathrm{ge}(\sigma_*)}}{N_{new}}}\right)
    +\tilde{O}(\alpha_1 m r^{-\mathrm{ge}(\sigma_*)-1/2}) 
\end{align*}
holds with high probability.
\end{lemma}

\begin{proof}
    From Lemma~\ref{lemm:tightresultforpt}, we have that 
    \begin{equation*}
        g(\Gamma^*,\vx)=P_1'+P_2'\left(\frac{\inner{\vx}{\beta}}{\sqrt{r}}\right)^{\mathrm{ge}(\sigma_*)}+n_3, 
    \end{equation*}
    where $P_1' = o_d(1)$, $P_2'=\Theta_d((\log{d})^{-C_{P_2}})$ and $n_3 = o_d(P_2'r^{-\mathrm{ge}(\sigma_*)/2-1/2}\log^{-2\deg(\sigma_*)+2}d)$.  
    Noticing that $g(\Gamma^*,\vx)=O(1)$ with high probability, 
    then, letting $b = \frac{1}{N_{new}}\sum_{i=1}^{N_{new}}g(\Gamma^*,\vw_i)$ and applying Lemma~\ref{lemm:hoeffding}, it holds that 
    \begin{equation*}
        g(\Gamma^*,\vw_i)-b=P_2'r^{-\frac{\mathrm{ge}(\sigma_*)}{2}}\mathrm{He}_{\mathrm{ge}(\sigma_*)}(\inner{\beta}{\vw_i})+n_3+\tilde{O}(1/\sqrt{N_{new}}). 
    \end{equation*}
    When $N_{new}=O(r^{\mathrm{ge}(\sigma_*)+2})$, the term $\tilde{O}(1/\sqrt{N_{new}})$ can be dominated by $n_3$, which means that 
    \begin{equation*}
        g(\Gamma^*,\vw_i)-b=P_2'r^{-\frac{\mathrm{ge}(\sigma_*)}{2}}\mathrm{He}_{\mathrm{ge}(\sigma_*)}(\inner{\beta}{\vw_i})+n_3, 
    \end{equation*}
    and $n_3 = o_d(P_2'r^{-\mathrm{ge}(\sigma_*)/2-1/2}\log^{-2\deg(\sigma_*)+2}d)$ with high probability.

    In this stage, we use $g(\Gamma^*,\vw_i)-b$ as a teacher signal. 
    Thus, as $C_u = 1/\sqrt{r}$ and $g(\Gamma^*,\vw_i)-b=\tilde{O}(r^{-\frac{\mathrm{ge}(\sigma_*)}{2}})$ , with high probability, we have
    \begin{align*}
        &\frac{1}{2}\sqrt{r}\Gamma^* \nabla_\vu{(f_{\mathrm{IC}}(\vw_i)-(g(\Gamma^*,\vw_i)-b))^2} \\
    =& \tilde{O}(\alpha_1^2 m^2 r^{-1/2} r^{-\frac{\mathrm{ge}(\sigma_*)-1}{2}})-\left\{P_2'r^{-\frac{\mathrm{ge}(\sigma_*)}{2}}\mathrm{He}_{\mathrm{ge}(\sigma_*)}(\inner{\beta}{\vw_i})+n_3\right\}\cdot\sqrt{r}\Gamma^*\nabla_\vu{f_{\mathrm{IC}}}(\vw_i). 
    \end{align*}
    Following the same argument as Lemma~\ref{lemm:gradresult} yields
    \begin{align*}
        &P_2'r^{-\frac{\mathrm{ge}(\sigma_*)}{2}}\mathrm{He}_{\mathrm{ge}(\sigma_*)}(\inner{\beta}{\vw_i})\cdot\sqrt{r}\Gamma^*\nabla_\vu{f_{\mathrm{IC}}}(\vw_i)\\
        =& P_2'\Theta(\alpha_1 m) \inner{\beta}{\vu}  \frac{r^{-\frac{\mathrm{ge}(\sigma_*)}{2}}\beta}{\kappa^{\mathrm{ge}(\sigma_*)+1}\rho^{\mathrm{ge}(\sigma_*)}}
        \{(\kappa \sqrt{r})^{-(\mathrm{ge}(\sigma_*)-1)}(1+O(1/\sqrt{d}))+\inner{\beta}{\vu}^{2\mathrm{ge}(\sigma_*)-2}(1+O(1/\sqrt{d}))\}\\
        & +\nu + \tilde{O}\left(\alpha_1 m \sqrt{\frac{r^{-\mathrm{ge}(\sigma_*)}}{N_1}}\right)
    \end{align*}    
    with high probability.  Since $\inner{\beta}{\vu}=\tilde{O}_p(1/r)$ at the initialization, the first term (regarding $(\kappa \sqrt{r})^{-(\mathrm{ge}(\sigma_*)-1)}$) dominates the second term (regarding $\inner{\beta}{\vu}^{2\mathrm{ge}(\sigma_*)-2}$).
    Taking the average over $i=1,\dots,N_{new}$ yields that 
    \begin{align*}
        &P_2'r^{-\frac{\mathrm{ge}(\sigma_*)}{2}}\frac{1}{N_{new}}\sum_{i=1}^{N_{new}}\mathrm{He}_{\mathrm{ge}(\sigma_*)}(\inner{\beta}{\vw_i})\cdot\sqrt{r}\Gamma^*\nabla_\vu{f_{\mathrm{IC}}}(\vw_i)\\
        =& P_2'\Theta(\alpha_1 m) \inner{\beta}{\vu}  \frac{\beta}{\kappa^{2\mathrm{ge}(\sigma_*)}\rho^{\mathrm{ge}(\sigma_*)}}( \sqrt{r})^{-(2\mathrm{ge}(\sigma_*)-1)}(1+O(1/\sqrt{d}))\\
         & + \frac{1}{N_{new}}\sum_{i=1}^{N_{new}}\nu_i + \tilde{O}\left(\alpha_1 m \sqrt{\frac{r^{-\mathrm{ge}(\sigma_*)}}{N_1}}\right), 
    \end{align*}
    where $\nu_i$ is a series of i.i.d. mean-zero random variable vectors which satisfy $\nu_i = \tilde{O}_p(\alpha_1 m \cdot r^{-\mathrm{ge}(\sigma_*)/2})$. 
    Then, from Hoeffding's inequality, we have
    \begin{equation*}
        \frac{1}{N_{new}}\sum_{i=1}^{N_{new}}\nu_i = \tilde{O}\left(\alpha_1 m \sqrt{\frac{r^{-\mathrm{ge}(\sigma_*)}}{N_{new}}}\right), 
    \end{equation*}
    with high probability.
    Next we investigate the effect of $n_3$.  From Lemma~\ref{lemm:gradientcalc}, $\sqrt{r}\Gamma^*\nabla_\vu{f_{\mathrm{IC}}}(\vw_i)=\tilde{O}(\alpha_1 mC_u)=\tilde{O}(\alpha_1 m r^{-1/2})$ holds.  This leads to
    $n_3\sqrt{r}\Gamma^*\nabla_\vu{f_{\mathrm{IC}}}(\vw_i)=\tilde{O}(\alpha_1 m r^{-\mathrm{ge}(\sigma_*)/2-1})$. Moreover, from Lemma~\ref{lemm:p0_small}, when $\mathrm{ge}(\sigma_*)=2$ we have $\gamma(\vx,y)=O(1/\sqrt{r})$.  Combining this fact with $N_1=\tilde{\Omega}(r^{\mathrm{ge}(\sigma_*)+1})$ yields $\sqrt{r}\Gamma^*\nabla_\vu{f_{\mathrm{IC}}}(\vw_i) = \tilde{O}(\alpha_1 m /r)$.  Therefore, whether $\mathrm{ge}(\sigma_*)=1$ or $\mathrm{ge}(\sigma_*)=2$, we obtain
    \begin{equation*}
        n_3\sqrt{r}\Gamma^*\nabla_\vu{f_{\mathrm{IC}}}(\vw_i)=\tilde{O}(\alpha_1 m r^{-\mathrm{ge}(\sigma_*)-1/2}). 
    \end{equation*}  
    Taking the average over $i=1,\dots,N_{new}$ completes the proof.
\end{proof}

\begin{lemma}\label{lemm:onestepgdinner}
    Under the condition of Lemma~\ref{lemm:onestepgdvector}, it holds that 
    \begin{align*}
    \inner{\beta}{\vh}=& P_2'\Theta(\alpha_1 m) \inner{\beta}{\vu}  \frac{1}{\kappa^{2\mathrm{ge}(\sigma_*)}\rho^{\mathrm{ge}(\sigma_*)}}( \sqrt{r})^{-(2\mathrm{ge}(\sigma_*)-1)}(1+O(1/\sqrt{d}))\\
    &+\tilde{O}\left(\alpha_1^2 m^2 r^{-\frac{\mathrm{ge}(\sigma_*)}{2}}\right)
    +\tilde{O}\left(\alpha_1 m \sqrt{\frac{r^{-\mathrm{ge}(\sigma_*)}}{N_1}}\right)
    +\tilde{O}\left(\alpha_1 m \sqrt{\frac{r^{-\mathrm{ge}(\sigma_*)}}{N_{new}}}\right)
    +o(\alpha_1 m r^{-\mathrm{ge}(\sigma_*)-1/2}), 
\end{align*}
 with high probability.
\end{lemma}

\begin{proof}
    As $\|\beta\|=1$, using the result of Lemma~\ref{lemm:onestepgdvector} and taking the inner product of $\vh$ and $\beta$ yields
    \begin{align*}
    \inner{\beta}{\vh}=& P_2'\Theta(\alpha_1 m) \inner{\beta}{\vu}  \frac{1}{\kappa^{2\mathrm{ge}(\sigma_*)}\rho^{\mathrm{ge}(\sigma_*)}}( \sqrt{r})^{-(2\mathrm{ge}(\sigma_*)-1)}(1+O(1/\sqrt{d}))\\
    &+\tilde{O}(\alpha_1^2 m^2 r^{-\frac{\mathrm{ge}(\sigma_*)}{2}})+
    \tilde{O}\left(\alpha_1 m \sqrt{\frac{r^{-\mathrm{ge}(\sigma_*)}}{N_1}}\right)+
    \tilde{O}\left(\alpha_1 m \sqrt{\frac{r^{-\mathrm{ge}(\sigma_*)}}{N_{new}}})\right)
    +\tilde{O}\left(\alpha_1 m r^{-\mathrm{ge}(\sigma_*)-1/2}\right). 
\end{align*}
    However, by carefully investigating the last term ($\tilde{O}(\alpha_1 m r^{-\mathrm{ge}(\sigma_*)-1/2})$) in the RHS, we can obtain a tighter bound of that term.  
    Indeed, this term arises from
    \begin{align*}
        &n_3 \inner{\beta}{\frac{1}{N_{new}}\sum_{i=1}^{N_{new}}\sqrt{r}\Gamma^*\nabla_\vu{f_{\mathrm{IC}}}(\vw_i)}\\
        &= n_3 \alpha_1 L_m \frac{1}{N_{new}}\sum_{i=1}^{N_{new}}\left\{\inner{\beta}{\sqrt{r}\Gamma^*\vx} \cdot \gamma(\vx,y) \inner{\beta^*}{\vu}+\gamma(\vx,y) \inner{\vx}{\vu}\inner{\beta^*}{\sqrt{r}\Gamma^*\beta}  \right. \\
        & ~~~~~~~~~~~~~~~~~~~~~~~~~~~~~~~~~~~~~~~\left. + \inner{\beta}{\sqrt{r}\Gamma^*\vx} \cdot n_1+ \inner{\vx}{\vu} \inner{\beta}{\vn_2}\right\}.
    \end{align*}
    We notice that the leading term of the right hand side is $n_3 \alpha_1 L_m P_0\inner{\vu}{\vx}$ when $\mathrm{ge}(\sigma_*)=1$, and $n_3\alpha_1 m P_1\inner{\beta}{\sqrt{r}\Gamma_*^{\top}\Gamma_u\vx}/\rho\cdot\inner{\vx}{\vu}$ when $ge_(\sigma_*)=2$.
    Note that $\mathbb{E}_{\vx}[P_0\inner{\vu}{\vx}]=0$ and
    $\mathbb{E}_{\vx}[P_1\inner{\sqrt{r}\Gamma_u^{\top}\Gamma_*\beta}{\vx}/\rho\cdot\inner{\vx}{\vu}]=P_1\inner{\sqrt{r}\Gamma_u^{\top}\Gamma_*\beta}{\vu}/\rho=\tilde{O}(r^{-3/2})$ because of $\sqrt{r}\Gamma_u^{\top}\Gamma^*\beta=\tilde{O}(1/\sqrt{r})\beta+\tilde{O}(1/r)+\inner{\beta}{\vu}\vu$ (see Lemma~\ref{lemm:betaanddoublegamma}).  Moreover, we have $\mathbb{E}_{\vx}[P_0^2\inner{\vu}{\vx}^2]=\tilde{O}(1)$ and $\mathbb{E}_{\vx}[P_1^2\inner{\sqrt{r}\Gamma_u^{\top}\Gamma_*\beta}{\vx}^2/\rho^2\cdot\inner{\vx}{\vu}^2]\leq P_1^2/\rho^2\cdot\mathbb{E}[\inner{\sqrt{r}\Gamma_u^{\top}\Gamma_*\beta}{\vx}^4]^{1/2}\mathbb{E}[\inner{\vx}{\vu}^4]^{1/2} =\tilde{O}(1)$.  Therefore, Lemma~\ref{lemm:hoeffding} yields that 
    \begin{align*}
        & \frac{1}{N_{new}}\sum_{i=1}^{N_{new}}P_0\inner{\vu}{\vw_i} = \tilde{O}\left(\sqrt{\frac{1}{rN_{new}}}\right), \\
        & \frac{1}{N_{new}}\sum_{i=1}^{N_{new}}P_1\inner{\beta}{\sqrt{r}\Gamma_u\Gamma^*\vw_i}/\rho\inner{\vw_i}{\vu} = \tilde{O}(r^{-3/2}) + \tilde{O}(r^{-1}N_{new}^{-1/2}).
    \end{align*}
    Since $n_3 = o_d(P_2'r^{-\mathrm{ge}(\sigma_*)/2-1/2}\log^{-2\deg(\sigma_*)+2}d)$ and $N_{new}=\tilde{\Omega}(r^{\mathrm{ge}(\sigma_*)+2})$, this implies that 
    \begin{equation*}
        n_3 \inner{\beta}{\frac{1}{N_{new}}\sum_{i=1}^{N_{new}}\sqrt{r}\Gamma^*\nabla_\vu{f_{\mathrm{IC}}}(\vw_i)} = o(\alpha_1 m r^{-\mathrm{ge}(\sigma_*)-1/2}). 
    \end{equation*}
    This yields the assertion.
\end{proof}

Setting the parameters and taking the absolute value directly yields the following lemma:
\begin{lemma}\label{lemm:weakfinal}
    Set $N_1 = \tilde{\Omega}(r^{\mathrm{ge}(\sigma_*)+2})$, $N_{new} = \tilde{\Omega}(r^{\mathrm{ge}(\sigma_*)+2})$, $\alpha_1 m = \tilde{\Theta}(r^{-\mathrm{ge}(\sigma_*)/2-1})$, $\eta_1 = \tilde{\Theta}(r^{3\mathrm{ge}(\sigma_*)/2+3/2})$ and $\lambda_1 = \eta_1^{-1}$, then we have 
    \begin{equation*}
        \vu^{(1)} = \tilde{\Theta}(1)\beta + \tilde{O}(1), 
    \end{equation*}
    and 
    \begin{equation*}
        \left|\inner{\beta}{\vu^{(1)}}\right| = \tilde{\Theta}(1) + o(1), 
    \end{equation*}
    with high probability. In particular, we have that 
    \begin{equation*}
        \left|\inner{\beta}{\frac{\vu^{(1)}}{\|\vu^{(1)}\|}}\right| \geq 1/\mathrm{polylog}(d). 
    \end{equation*}
\end{lemma}

\section{Strong recovery}\label{appx:strong}
After achieving weak recovery, we train the vector $\vu$, regarding the in-context examples $(\vx_i,y_i)$ as training data.  The idea of proof in this section is taken from \citet{lee2024neuralnetworklearnslowdimensional}, but we achieve better sample complexity.   

\begin{lemma}\label{lemm:gdforstrongrecovery}
    Let $\vh = \frac{1}{2}\sqrt{r}\Gamma^* \nabla_\vu{(f_{\mathrm{IC}}(\vx)-y)^2}$.  Then, we have that 
    \begin{align*}
        \vh =& \Theta(\alpha_2 m) \inner{\beta}{\vu}^{2\mathrm{ie}(\sigma_*)-1}  \frac{\beta}{\kappa^{\mathrm{ie}(\sigma_*)+1}\rho^{\mathrm{ie}(\sigma_*)}}(1+O(1/\sqrt{d}))\\
    &+\tilde{O}\left(\alpha_2^2m^2\sqrt{r}\right)+\tilde{O}\left(\alpha_2 m \sqrt{\frac{r}{N_2}}\right)+\nu, \\
    %%%%%
        \inner{\beta}{\vh} =& \Theta(\alpha_2 m) \inner{\beta}{\vu}^{2\mathrm{ie}(\sigma_*)-1}  \frac{1}{\kappa^{\mathrm{ie}(\sigma_*)+1}\rho^{\mathrm{ie}(\sigma_*)}}(1+O(1/\sqrt{d}))\\
    &+\tilde{O}(\alpha_2^2m^2\sqrt{r})+\tilde{O}\left(\alpha_2 m \sqrt{\frac{r}{N_2}}\right)+\nu_6, \\
    %%%%%
        \inner{\vu}{\vh} =& \Theta(\alpha_2 m) \inner{\beta}{\vu}^{2\mathrm{ie}(\sigma_*)}  \frac{1}{\kappa^{\mathrm{ie}(\sigma_*)+1}\rho^{\mathrm{ie}(\sigma_*)}}(1+O(1/\sqrt{d}))\\
    &+\tilde{O}(\alpha_2^2m^2\sqrt{r})+\tilde{O}\left(\alpha_2 m \sqrt{\frac{r}{N_2}}\right)+\nu_7, 
    \end{align*}
    with high probability, where $\nu_6$ and $\nu_7$ are mean-zero sub-Weibull random variables.  Moreover, $\nu_6=\tilde{O}(\alpha_2 m)$ and $\nu_7 = \tilde{O}(\alpha_2 m)$ hold with high probability.  
\end{lemma}

\begin{proof}
    Remember that, in the stage of strong recovery, the initialization scale is $\alpha_2,$ the context length is $N_2$ and $C_u = 1$. By using the same argument to derive Lemma~\ref{lemm:gradresult}, we have
    \begin{align*}
       \vh =& \Theta(\alpha_2 m) \inner{\beta}{\vu}  \frac{\beta}{\kappa^{\mathrm{ie}(\sigma_*)+1}\rho^{\mathrm{ie}(\sigma_*)}}\{(\kappa \sqrt{r})^{-(\mathrm{ie}(\sigma_*)-1)}(1+O(1/\sqrt{d}))+\inner{\beta}{\vu}^{2\mathrm{ie}(\sigma_*)-2}(1+O(1/\sqrt{d}))\}\\
    &+\tilde{O}(\alpha_2^2m^2\sqrt{r})+\tilde{O}\left(\alpha_2 m \sqrt{\frac{r}{N_2}}\right)+\nu.
    \end{align*}
    Now, since $\inner{\beta}{\vu} \geq 1/\mathrm{polylog}(d)$, the term $\inner{\beta}{\vu}^{2\mathrm{ie}(\sigma_*)-2}$ dominates over the term $(\kappa \sqrt{r})^{-(\mathrm{ie}(\sigma_*)-1)}$.  Therefore, with high probability, we have
     \begin{align*}
        \vh =& \Theta(\alpha_2 m) \inner{\beta}{\vu}^{2\mathrm{ie}(\sigma_*)-1}  \frac{\beta}{\kappa^{\mathrm{ie}(\sigma_*)+1}\rho^{\mathrm{ie}(\sigma_*)}}(1+O(1/\sqrt{d}))\\
    &+\tilde{O}(\alpha_2^2m^2\sqrt{r})+\tilde{O}\left(\alpha_2 m \sqrt{\frac{r}{N_2}}\right)+\nu. 
    \end{align*}
    Next, since $\|\beta\|=1$, we have
    \begin{align*}
        \inner{\beta}{\vh} =& \Theta(\alpha_2 m) \inner{\beta}{\vu}^{2\mathrm{ie}(\sigma_*)-1}  \frac{1}{\kappa^{\mathrm{ie}(\sigma_*)+1}\rho^{\mathrm{ie}(\sigma_*)}}(1+O(1/\sqrt{d}))\\
    &+\tilde{O}(\alpha_2^2m^2\sqrt{r})+\tilde{O}\left(\alpha_2 m \sqrt{\frac{r}{N_2}}\right)+\inner{\nu}{\beta}, 
    \end{align*}
    with high probability, where $\nu = \tilde{O}(\alpha_2 m \sqrt{r})$.  
    Then, since we see that  
    \begin{align*}
        \nu =& y\cdot \alpha_2 L_m \left\{\sqrt{r}\Gamma^*\vx \cdot \gamma(\vx,y) \inner{\beta^*}{\vu}+\sqrt{r}\Gamma^*\gamma(\vx,y) \inner{\vx}{\vu} \beta^*+\sqrt{r}\Gamma^*\vx \cdot n_1+ \inner{\vx}{\vu} \vn_2\right\}\\
        &- \mathbb{E}_{\vx}[y\cdot \alpha_2 L_m \left\{\sqrt{r}\Gamma^*\vx \cdot (\gamma(\vx,y) \inner{\beta^*}{\vu}+\sqrt{r}\Gamma^*\gamma(\vx,y) \inner{\vx}{\vu} \beta^*+\sqrt{r}\Gamma^*\vx \cdot n_1+ \inner{\vx}{\vu} \vn_2\right\}],
    \end{align*}
    it holds that 
    \begin{align*}
        &\inner{\beta}{\nu}\\
        =& y\cdot \alpha_2 L_m \left\{\inner{\beta}{\sqrt{r}\Gamma^*\vx} \cdot \gamma(\vx,y) \inner{\beta^*}{\vu}+\gamma(\vx,y) \inner{\vx}{\vu}\inner{\beta}{\sqrt{r}\Gamma^*\beta^*} +\inner{\beta}{\sqrt{r}\Gamma^*\vx} \cdot n_1+ \inner{\vx}{\vu} \inner{\beta}{\vn_2}\right\}\\
        &-\mathbb{E}_{\vx}[ y\cdot \alpha_2 L_m \left\{\inner{\beta}{\sqrt{r}\Gamma^*\vx} \cdot \gamma(\vx,y) \inner{\beta^*}{\vu}+\gamma(\vx,y) \inner{\vx}{\vu}\inner{\beta}{\sqrt{r}\Gamma^*\beta^*} +\inner{\beta}{\sqrt{r}\Gamma^*\vx} \cdot n_1+ \inner{\vx}{\vu} \inner{\beta}{\vn_2}\right\} ]. 
    \end{align*}
    Here, the first term of the right hand side satisfies 
    \begin{align*}
        & y\cdot \alpha_2 L_m   \left\{\inner{\beta}{\sqrt{r}\Gamma^*\vx} \cdot \gamma(\vx,y) \inner{\beta^*}{\vu}+\gamma(\vx,y) \inner{\vx}{\vu}\inner{\beta}{\sqrt{r}\Gamma^*\beta^*}\right. \\ 
        & \left.  ~~~~~~~+\inner{\beta}{\sqrt{r}\Gamma^*\vx} \cdot n_1+ \inner{\vx}{\vu} \inner{\beta}{\vn_2}\right\} \\
        &= \tilde{O}(\alpha_2 m), 
    \end{align*}
    with high probability, and $\inner{\beta}{\nu}$ is the difference between this random variable and its expectation. 
    Thus, by defining $\nu_6 = \inner{\beta}{\nu}$, Hoeffding's inequality yields that 
    $\nu_6 = \tilde{O}(\alpha_2 m)$ with high probability.  Moreover, since $|\inner{\beta}{\nu}| \leq \|\nu\|$, $\nu_6 = \inner{\beta}{\nu}$ has sub-Weibull tail.
    
    Likewise, we have that 
     \begin{align*}
        \inner{\vu}{\vh} =& \Theta(\alpha_2 m) \inner{\beta}{\vu}^{2\mathrm{ie}(\sigma_*)}  \frac{1}{\kappa^{\mathrm{ie}(\sigma_*)+1}\rho^{\mathrm{ie}(\sigma_*)}}(1+O(1/\sqrt{d}))\\
    &+\tilde{O}(\alpha_2^2m^2\sqrt{r})+\tilde{O}\left(\alpha_2 m \sqrt{\frac{r}{N_2}}\right)+\nu_7, 
    \end{align*}
    where $\nu_7$ is a mean-zero sub-Weibull random variable satisfying $\nu_7 = \tilde{O}(\alpha_2 m)$.
\end{proof}

\begin{lemma}\label{lemm:strongonestep}
  Let $a^{(n)} = \inner{\beta}{\vu^{(n)}}$.  Suppose that $ c_1 \leq a^{(n)}\leq 1-\varepsilon$  where $c_1=\tilde{\Theta}(1) > 0$.  Set $\alpha_2 m = \tilde{\Theta}(\varepsilon/r)$, $N_2 = \tilde{\Theta}(r^2)$ and $\eta_2 = \tilde{\Theta}(1/\sqrt{r})$.  Then, there exists $c_3=\tilde{\Theta}(1)$ which satisfies
  \begin{equation*}
      a^{(n+1)} \geq a^{(n)} + \frac{c_3\varepsilon}{r\sqrt{r}}{a^{(n)}}^{2\mathrm{ie}(\sigma_*)-1}(1-{a^{(n)}}^2)(1-O(1/\sqrt{d})) + \nu_8 - \tilde{O}(\varepsilon/r^2), 
  \end{equation*}
  with high probability, where $\nu_8$ is a mean-zero sub-Weibull random variable which satisfies $\nu_8 = \tilde{O}(\varepsilon/r\sqrt{r})$ with high probability.
\end{lemma}

\begin{proof}
        Using the projection matrix $\mP_u=\mI-uu^{\top}$, online SGD update can be written as
    \begin{align*}
        \vu^{(n+1)} &= \frac{\vu^{(n)}+\eta_2\mP_{\vu^{(n)}}\vh}{\|\vu^{(n)}+\eta_2\mP_{\vu^{(n)}}\vh\|}, 
    \end{align*}
    which gives  
    \begin{align*}
        & \inner{\beta}{\vu^{(n+1)}} \notag \\ 
        &= \inner{\beta}{\frac{\vu^{(n)}+\eta_2\mP_{\vu^{(n)}}\vh}{\|\vu^{(n)}+\eta_2\mP_{\vu^{(n)}}\vh\|}} \notag \\
        &=\inner{\beta}{\vu^{(n)}} + \eta_2\inner{\beta}{\mP_{u^{(n)}}\vh} - \frac{1}{2}\eta_2^2\|\mP_{u^{(n)}}\vh\|^2a^{(n)} + O(\eta_2^3) \notag \\
        &= \inner{\beta}{\vu^{(n)}} + \eta_2\inner{\beta}{\vh} - \eta_2\inner{\beta}{\vu^{(n)}}\inner{\vu^{(n)}}{\vh}- \frac{1}{2}\eta_2^2\|\mP_{u^{(n)}}\vh\|^2a^{(n)} + O((\eta_2\|\vh\|)^3). 
    \end{align*}
    By the definition of $a^{(n)}$, i.e., $a^{(n)}=\inner{\beta}{\vu^{(n)}}$, we have
    \begin{equation*}
        a^{(n+1)} = a^{(n)} + \eta_2\inner{\beta}{\vh} - \eta_2a^{(n)}\inner{\vu^{(n)}}{\vh} -\frac{1}{2}\eta_2^2\|\mP_{u^{(n)}}\vh\|^2a^{(n)} + O((\eta_2\|\vh\|)^3). 
    \end{equation*}
    Now, from Lemma~\ref{lemm:gdforstrongrecovery}, we also have
    \begin{equation*}
        \inner{\beta}{\vh} = \frac{\varepsilon}{r}{a^{(n)}}^{2\mathrm{ie}(\sigma_*)-1}(\tilde{\Theta}(1)+O(1/\sqrt{d})) + \tilde{O}(\varepsilon/r\sqrt{r}) + \tilde{O}(\varepsilon/r\sqrt{r}) +\nu_6
    \end{equation*}
    and
    \begin{equation*}
        \inner{\vu}{\vh} = \frac{\varepsilon}{r} {a^{(n)}}^{2\mathrm{ie}(\sigma_*)}(\tilde{\Theta}(1)+O(1/\sqrt{d})) +\tilde{O}(\varepsilon/r\sqrt{r}) + \tilde{O}(\varepsilon/r\sqrt{r}) +\nu_7. 
    \end{equation*}
    Moreover, since $\|\mP_{u^{(n)}}\vh\| = \tilde{O}(\alpha_2 m \sqrt{r})=\tilde{O}(\varepsilon/\sqrt{r})$, it holds that 
    \begin{equation*}
        \frac{1}{2}\eta_2^2\|\mP_{u^{(n)}}\vh\|^2 = \tilde{O}(\varepsilon^2/r^2). 
    \end{equation*}
    Therefore, ignoring the term $O((\eta_2\|\vh\|)^3)$, we arrive at 
    \begin{align*}
        a^{(n+1)} &= a^{(n)} + \eta_2\inner{\beta}{\vh} - \eta_2a^{(n)}\inner{\vu}{\vh} -\frac{1}{2}\eta_2^2\|\mP_{u^{(n)}}\vh\|^2a^{(n)}\\
        = &  a^{(n)} +\frac{\eta_2\varepsilon}{r} {a^{(n)}}^{2\mathrm{ie}(\sigma_*)-1}(1-{a^{(n)}}^2)(\tilde{\Theta}(1)+O(1/\sqrt{d}))+\eta_2(\nu_6-a^{(n)}\nu_7)+\eta_2\tilde{O}(\varepsilon/r\sqrt{r})+\tilde{O}(\varepsilon/r^2)\\
        \geq &  a^{(n)} + \frac{c_3\varepsilon}{r\sqrt{r}}{a^{(n)}}^{2\mathrm{ie}(\sigma_*)-1}(1-{a^{(n)}}^2)(1-O(1/\sqrt{d})) + \nu_8 - \tilde{O}(\varepsilon/r^2), 
    \end{align*}
    where $\nu_8$ is a mean-zero sub-Weibull random variable satisfying $\nu_8 = \tilde{O}(\varepsilon/r\sqrt{r})$. 
\end{proof}

\begin{lemma}\label{lemm:strongrecoveryachieve}
   Suppose that $\vu^{(1)}$ satisfies $\inner{\beta}{\vu^{(1)}} \geq c_1$, where $c_1 \geq 1/\mathrm{polylog}(d) > 0$.  Then, there exists $N_3 = \tilde{\Theta}(\frac{r\sqrt{r}}{\varepsilon}\log{\frac{1}{\varepsilon}})$ such that 
  \begin{equation*}
      \inner{\beta}{\vu^{(N_3+1)}} \geq 1- \varepsilon
  \end{equation*}
  with high probability.  
\end{lemma}

\begin{proof}
  Before going into the proof, we first explain the main idea of the proof.  Following the same argument as Lemma 19 in \citet{lee2024neuralnetworklearnslowdimensional}, we can see that after $\tilde{\Theta}(r\sqrt{r}/\varepsilon)$ steps, $a^{(n)}$ becomes larger than a constant $\sqrt{\frac{k+1}{k+2}}$, where $k=2\mathrm{ie}(\sigma_*)-1$.  Then, by applying the Mean Value theorem, we can observe that $1-a^{(i+1)} \lesssim (1-\frac{c_3\varepsilon}{r\sqrt{r}}C_k)(1-a^{(i)})$ where $C_k=\Theta(1)$.  This means $1-a^{(i)}$ converges to 0 geometrically, which yields the required data length $N_3 = \tilde{\Theta}(\frac{r\sqrt{r}}{\varepsilon}\log{\frac{1}{\varepsilon}})$.

  Let $\nu_9^{(i)} = r\sqrt{r}\nu_8^{(i)} /\varepsilon $ and $k=2\mathrm{ie}(\sigma_*)-1$.  Because $v_9^{(i)}$ is a sequence of mean-zero sub-Weibull random variables with $v_9^{(i)} = \tilde{O}_p(1)$, we have
  \begin{equation}
     \left| \sum_{i=j}^{j+l} \nu_9^{(i)} \right|= C\sqrt{l},\label{eq:sumofnu}
  \end{equation}
  for any $1 \leq j,l \leq N_3$ with high probability, where $C=\tilde{O}(1)$.    Let $c_k = 1-\sqrt{\frac{k+1}{k+2}}$.  Note that $c_k$ is a constant that only depends on $k$, that is $c_k=O(1)$. Suppose that $c_1 \leq a^{(i)} \leq 1-\frac{1}{3}c_k$ for all $i=1,2,\dots,N_3$.  Then, from Lemma~\ref{lemm:strongonestep}, we have 
  \begin{equation*}
     a^{(n+1)} \geq a^{(n)} + \frac{c_3\varepsilon}{r\sqrt{r}}{a^{(n)}}^{2\mathrm{ie}(\sigma_*)-1}(1-{a^{(n)}}^2)(1-O(1/\sqrt{d})) + \nu_8 - \tilde{O}(\varepsilon/r^2), 
  \end{equation*}
  for $i \leq N_3$.  The term $O(1/\sqrt{d})$ is smaller than 1, thus we may ignore this term.  Moreover, the term $\tilde{O}(\varepsilon/r^2)$ is dominated by $\frac{c_3\varepsilon}{r\sqrt{r}}{a^{(n)}}^{2\mathrm{ie}(\sigma_*)-1}(1-{a^{(n)}}^2)=\tilde{O}(\frac{\varepsilon}{r\sqrt{r}})$.
  Let $c_2=c_1^{(2\mathrm{ie}(\sigma_*)-1)}$.  By ignoring these terms and using $1-{a^{(i)}}^2 \geq \frac{1}{3}c_k$, we have
  \begin{align}
      a^{(i+1)} \geq & a^{(i)} + \frac{c_2c_3} {r\sqrt{r}}\varepsilon\cdot c_k/3 + \frac{\varepsilon}{r\sqrt{r}}\cdot\nu_9^{(i)}\\
      \geq & a^{(1)} + \frac{c_2c_3c_k}{3r\sqrt{r}}\varepsilon i - \frac{\varepsilon}{r\sqrt{r}}|\sum_{j=1}^{i}\nu_9^{(j)}|\\
      \geq & a^{(1)} + (\frac{c_2c_3c_k}{3r\sqrt{r}}\varepsilon)i - \frac{\varepsilon}{r\sqrt{r}}C\sqrt{i}. 
      \label{eq:aiupdatelowerbound}
  \end{align}  
  If $i \leq \frac{r^3c_1^2}{4\varepsilon^2C^2}$, then $\frac{\varepsilon}{r\sqrt{r}}C\sqrt{i}\leq c_1/2$, and if $i \geq (\frac{6C}{c_2c_3c_k})^2$, then $\frac{\varepsilon}{r\sqrt{r}}C\sqrt{i} \leq \frac{c_2c_3c_k}{6r\sqrt{r}}\varepsilon i$.  By observing the order in terms of $r$, we have $\frac{r^3c_1^2}{4\varepsilon^2C^2} \geq (\frac{6C}{c_2c_3c_k})^2$ when $r$ is sufficiently large. Therefore, it holds that  
  \begin{equation}
      a^{(i+1)} \geq \frac{c_1}{2}+\frac{c_2c_3c_k}{6r\sqrt{r}}\varepsilon i.\label{eq:aiupdatelowerbound2}
  \end{equation}
  When $i=\frac{6r\sqrt{r}}{c_2c_3c_k\varepsilon}=\tilde{\Theta}(r\sqrt{r}/\varepsilon)$, then the RHS of Eq.~\eqref{eq:aiupdatelowerbound2} exceeds 1.  
  Therefore, there exists $i^* \leq N_3 = \tilde{\Theta}(\frac{r\sqrt{r}}{\varepsilon}\log{\frac{1}{\varepsilon}})$ such that $a^{(i^*)} \geq 1-c_k/3$.
  Next we prove that $a^{(i)} \geq 1-c_k  = \sqrt{\frac{k+1}{k+2}}$ for all $i=i^*,i^*+1,\dots,N_3$.  In this setting, $a^{(i+1)}-a^{(i)}=\tilde{O}(\frac{\varepsilon}{r\sqrt{r}})$ holds.  Therefore, if there exists $i \geq i^*$ such that $a^{(i-1)}\geq1-c_k/3$ and $a^{(i)} \leq 1-c_k/3$, we have $a^{(i)} \geq 1-2c_k/3$ with high probability.  Also, if $a^{(i+l)}\leq 1-c_k/3$ holds for all $l=0,1,\dots,j-1$, we have
  \begin{equation*}
      a^{(i+j)}\geq 1-\frac{2c_k}{3} +\frac{c_2c_3c_k}{3r\sqrt{r}}\varepsilon j-\frac{\varepsilon}{r\sqrt{r}}C\sqrt{j}.
  \end{equation*}
  If $j\leq \frac{r^3c_k^2}{9\varepsilon^2C^2}=\tilde{O}(r^3/\varepsilon^2)$, then
  $\frac{\varepsilon}{r\sqrt{r}}C\sqrt{j} \leq c_k/3$, and if $j \geq (\frac{3C}{c_2c_3c_k})^2=\tilde{O}(1)$, then $\frac{\varepsilon}{r\sqrt{r}}C\sqrt{j} \leq  \frac{c_2c_3c_k}{3r\sqrt{r}}\varepsilon j$.  Since $\frac{r^3c_k^2}{9\varepsilon^2C^2}\geq(\frac{3C}{c_2c_3c_k})^2$ when $r$ is sufficiently large, we have 
  \begin{equation*}
      a^{(i+j)}\geq 1-c_k 
  \end{equation*}
  until $i+j=N_3+1$ or $a^{(i+j)}\geq1-c_k/3$ holds again.  By repeating this argument if necessary, we get $a^{(i+1)}\geq1-c_k=\sqrt{\frac{k+1}{k+2}}$ for all $i^* \leq i \leq N_3$.  

  We continue by showing that, after we achieve $a^{(i^*)} \leq 1 - c_k$, the number of remaining steps needed to ensure $a^{(i)} \geq 1-\varepsilon$ is $\tilde{O}(\frac{r\sqrt{r}}{\varepsilon}\log{\frac{1}{\varepsilon}})$.
  Let $F(x)=x+\frac{c_3\varepsilon}{r\sqrt{r}}x^k(1-x^2).$  Then,
  \begin{equation*}
  a^{(i^*+i+1)} = F(a^{(i^*+i)}) + \frac{\varepsilon}{r\sqrt{r}}\nu_9^{(i^*+i)}-\tilde{O}(\varepsilon/r^2). 
  \end{equation*}
  By the Mean Value theorem, there exists $\gamma$ such that
  \begin{equation*}
      \frac{1-F(a^{(i^*+i)})}{1-a^{(i^*+i)}} = F'(\gamma), 
  \end{equation*}
  and $a^{(i^*+i)} < \gamma < 1$.  Now $F'(x)=1+\frac{c_2c_3\varepsilon}{r\sqrt{r}}x^{k-1}(k-(k+2)x^2)$, and since $\gamma > a^{(i^*+i)} \geq \sqrt{\frac{k+1}{k+2}}$, we have $k-(k+2)\gamma^2 < -1$, which leads to
  \begin{equation*}
      F'(\gamma) \leq 1-\frac{c_3\varepsilon}{r\sqrt{r}}\gamma^{k-1} < 1-\frac{c_3\varepsilon}{r\sqrt{r}}\left(\sqrt{\frac{k+1}{k+2}}\right)^{k-1}. 
  \end{equation*}
  Let $C_k = \left(\sqrt{\frac{k+1}{k+2}}\right)^{k-1}$, then
  $1-F(a^{(i^*+i)}) < (1-\frac{c_3\varepsilon}{r\sqrt{r}}C_k)(1-a^{(i^*+i)})$, which yields that 
  \begin{equation*}
      1-a^{i^*+i+1} < \left(1-\frac{c_3\varepsilon}{r\sqrt{r}}C_k\right)(1-a^{(i^*+i)}) + \frac{\varepsilon}{r\sqrt{r}}\nu_9^{(i^*+i)}+\tilde{O}(\varepsilon/r^2). 
  \end{equation*}
  Noting that $(1-\frac{c_3\varepsilon}{r\sqrt{r}}C_k)<1$, taking the sum leads to
  \begin{equation*}
      1- a^{(N_3+1)} < \left(1-\frac{c_3\varepsilon}{r\sqrt{r}}C_k \right)^{N_3+1-i^*}(1-a^{(i^*)}) + \sum_{i=i^*}^{N_3}\frac{\varepsilon}{r\sqrt{r}}\nu_9^{(i)}+\tilde{O}(\varepsilon N_3/r^2). 
  \end{equation*}
  Since $\lim_{t\rightarrow \infty}(1-\frac{1}{t})^t = 1/\mathrm{e}$, if we set $N_3 = i^* -1 + \frac{r\sqrt{r}}{c_3C_k\varepsilon}\log(1/\varepsilon)$, then we have 
  \begin{align*}
    \left(1-\frac{c_3\varepsilon}{r\sqrt{r}}C_k\right)^{N_3+1-i^*}(1-a^{(i^*)}) &<  
    \left(1-\frac{c_3\varepsilon}{r\sqrt{r}}C_k\right)^{\frac{r\sqrt{r}}{c_3C_k\varepsilon}\log(1/\varepsilon)}\\
    & \approx (1/\mathrm{e})^{\log(1/\varepsilon)} = \varepsilon. 
  \end{align*}
  Finally, by noticing that 
  \begin{equation*}
      \left|\sum_{i=i^*}^{N_3}\frac{\varepsilon}{r\sqrt{r}}\nu_9^{(i)}\right| \leq \frac{\varepsilon}{r\sqrt{r}}C\sqrt{N_3-i^*+1} = \tilde{O}(r^{-3/4}), 
  \end{equation*}
  we can see that this term is negligible. 
  Moreover, the third term $\tilde{O}(\varepsilon N_3/r^2)=\tilde{O}(\frac{\log{1/\varepsilon}}{\sqrt{r}})$ is also negligible. 
  By summarizing the argument above, we conclude that  
  \begin{equation*}
      1- a^{(N_3+1)} < \varepsilon. 
  \end{equation*}
\end{proof}

In Lemma~\ref{lemm:gradientcalc}, we assumed that $\inner{\vu}{\vx_i}=\tilde{O}_d(1)$.  Now $\vu^{(n)}$ and $\vx_i$ are not independent of each other, so we need to ensure that $\inner{\vu^{(j)}}{\vx_i}=\tilde{O}_d(1)$ for all $j=1,2,\dots, N_3+1$.

\begin{lemma}\label{lemm:thegrowthofux}
    It holds that 
       \begin{align*}
        \inner{\vx_i^*}{\vh} =&  \Theta(\alpha_2 m) \inner{\beta}{\vu}^{2\mathrm{ie}(\sigma_*)-1}  \frac{\inner{\vx_i^*}{\beta}}{(\kappa\rho)^{\mathrm{ie}(\sigma_*)}}(1+O(1/\sqrt{d}))\\
        & + \tilde{O}(\alpha_2^2m^2r) + \tilde{O}(\alpha_2 m r/\sqrt{N_2}) + \nu_*,
    \end{align*}
    where $\nu_*$ is a mean-zero random variable satisfying $\nu_* = \tilde{O}(\alpha_2 m \sqrt{r})$.
\end{lemma}

\begin{proof}
    From Lemma~\ref{lemm:gdforstrongrecovery}, we have
    \begin{align*}
        \vh=  & \Theta(\alpha_2 m) \inner{\beta}{\vu}^{2\mathrm{ie}(\sigma_*)-1}  \frac{\beta}{(\kappa\rho)^{\mathrm{ie}(\sigma_*)}}(1+O(1/\sqrt{d}))\\
    &+\tilde{O}(\alpha_2^2m^2\sqrt{r})+\tilde{O}(\alpha_2 m \sqrt{\frac{r}{N_2}})+\nu. 
    \end{align*}
    Considering $\vx_i^*=\tilde{O}(\sqrt{r})$, taking the inner product of $\vx_i^*$ and $\vh$ leads to
    \begin{align*}
        \inner{\vx_i^*}{\vh} =&  \Theta(\alpha_2 m) \inner{\beta}{\vu}^{2\mathrm{ie}(\sigma_*)-1}  \frac{\inner{\vx_i^*}{\beta}}{(\kappa\rho)^{\mathrm{ie}(\sigma_*)}}(1+O(1/\sqrt{d}))\\
        &+ \tilde{O}(\alpha_2^2m^2r) + \tilde{O}(\alpha_2 m r/\sqrt{N_2}) + \inner{\vx_i^*}{\nu}. 
    \end{align*}
    Following the same argument as Lemma~\ref{lemm:gdforstrongrecovery} yields $\inner{\vx_i^*}{\nu}=\tilde{O}(\alpha_2 m \sqrt{r})$. Thus, by letting $\nu_* = \inner{\vx_i^*}{\nu}$, we obtain the assertion.
\end{proof}

\begin{lemma}
    Set $\alpha_2 m = \tilde{\Theta}(\varepsilon/r)$, $N_2 = \tilde{\Theta}(r^2)$, $N_3=\tilde{\Theta}(\frac{r\sqrt{r}}{\varepsilon}\log{\frac{1}{\varepsilon}})$ and $\eta_2 = \tilde{\Theta}(1/\sqrt{r})$(This is exactly the same situation as Lemma~\ref{lemm:strongonestep} and Lemma~\ref{lemm:strongrecoveryachieve}).  Then we have
    \begin{equation*}
        \inner{\vx_i^*}{\vu^{(j)}} = \tilde{O}_{d}(1), 
    \end{equation*}
    for all $j=1,2,\dots,N_3+1$.
\end{lemma}

\begin{proof}
    When $j=1$, $\vx_i$ and $\vu^{(1)}$ are independent of each other, thus Lemma~\ref{lemm:Gammaxandu} yields the desired result. 
    Substituting the parameters into the result of Lemma~\ref{lemm:thegrowthofux} yields
    \begin{equation*}
        \eta_2 \inner{\vx_i^*}{\vh} = \tilde{O}\left(\frac{\varepsilon}{r\sqrt{r}}(1+O(1/\sqrt{d}))\right) + \tilde{O}\left(\frac{\varepsilon}{r\sqrt{r}}\right) +\eta_2\nu_*, 
    \end{equation*}
    where $\eta_2\nu_* = \tilde{O}(\varepsilon/r)$.  First, following the same argument to derive Eq.~\eqref{eq:sumofnu} yields
    \begin{equation*}
        \left|\sum_{k=1}^{j}\eta_2\nu_*\right| = \tilde{O}(\varepsilon\sqrt{j}/r),
    \end{equation*}
    for $j=1,2,\dots,N_3$.  Since $\varepsilon\sqrt{N_3}/r = \tilde{O}_d(r^{-1/4})=o_d(1)$, we can ignore the effect of $\eta_2\nu_*$.  
    Then, we can see that 
    \begin{align*}
        \inner{\vx_i^*}{\vu^{(j+1)}} =& \inner{\vx_i^*}{\vu^{(j)}} +\eta_2 \inner{\vx_i^*}{\vh} - \eta_2 \inner{\vh}{\vu^{(j)}} \inner{\vu^{(j)}}{\vx_i^*} -\frac{1}{2}\eta_2^2\|\mP_{u^{(n)}}\vh\|^2 + O((\eta_2\|\vh\|)^3)\\
        =& \inner{\vx_i^*}{\vu^{(j)}} + \inner{\beta}{\vu^{(j)}}^{2\mathrm{ie}(\sigma_*)-1}\tilde{O}\left(\frac{\varepsilon}{r\sqrt{r}}\right)\left(\inner{\beta}{\vx_i^*}-\inner{\vu^{(j)}}{\vx_i^*}\inner{\beta}{\vu^{(j)}}\right) \\
        & -\tilde{O}\left(\frac{\varepsilon}{r\sqrt{r}}\right)- \tilde{O}\left(\frac{\varepsilon^2}{r^2}\right)
    \end{align*}
    holds.  Now $\inner{\beta}{\vx_i^*}=\tilde{O}_p(1)$ and $\inner{\beta}{\vu^{(j)}} \leq 1$.  Therefore, if $\inner{\vu^{(j)}}{\vx_i^*}= \tilde{O}_p(1)$ holds, then $\inner{\vx_i^*}{\vu^{(j+1)}} = \inner{\vx_i^*}{\vu^{(j)}} +\tilde{O}\left(\frac{\varepsilon}{r\sqrt{r}}\right)$.  Therefore, by induction, we have 
    \begin{equation*}
        \inner{\vx_i^*}{\vu^{(N_3+1)}} = \tilde{O}\left(\log{\frac{1}{\varepsilon}}\right),
    \end{equation*}
    which does not depend on $d$.
\end{proof}
This lemma allows us to use Lemma~\ref{lemm:gradientcalc}.

From Lemma~\ref{lemm:weakfinal}, there exists $c_1 \geq  1/\mathrm{polylog}(d) > 0$ that satisfies either $\inner{\beta}{\vu^{(1)}} \geq c_1$ or $\inner{\beta}{\vu^{(1)}} \leq -c_1$.  In the former case, we can directly use Lemma~\ref{lemm:strongrecoveryachieve}.  In the latter case, $\inner{\beta}{\vu^{(1)}} \leq -c_1$ means $\inner{-\beta}{\vu^{(1)}} \geq c_1$.  Therefore, we can apply the same argument as Lemma~\ref{lemm:strongonestep} and Lemma~\ref{lemm:strongrecoveryachieve} to $-\beta$.  In conclusion, the following holds:

\begin{lemma}
Suppose that $\vu^{(1)}$ satisfies $|\vu^{(1)}|=1$ and $\|\inner{\beta}{\vu^{(1)}}\| \geq c_1$, where $c_1 \geq 1/\mathrm{polylog}(d) > 0$.  Then, there exists $N_3 = \tilde{\Theta}(\frac{r\sqrt{r}}{\varepsilon}\log{\frac{1}{\varepsilon}})$ and $s \in \{-1,1\}$ such that 
  \begin{equation*}
      \inner{s\beta}{\vu^{(N_3+1)}} \geq 1- \varepsilon
  \end{equation*}
  with high probability.  
\end{lemma}

Note that $N_3 = \tilde{\Theta}(\frac{r\sqrt{r}}{\varepsilon}\log{\frac{1}{\varepsilon}})$ means, by ignoring the term $\log{\frac{1}{\varepsilon}}$, $\varepsilon = \tilde{\Theta}({\frac{r\sqrt{r}}{N_3}})$.
Now we can estimate $\inner{\beta}{s\vx}$ using $\inner{\vu^{(N_3+1)}}{\vx}$.  For notational simplicity, we write $\vu^{(N_3+1)}$ as $\vu$ from now on.

\begin{lemma}\label{lemm:epsilonanddelta}
    Let $s\in \{-1,1\}$, $\vx\sim \mathcal{N}(0,\mI_d)$, and $\vu$ be a vector independent of $\vx$ satisfying $\|\vu\|=1$ and $\inner{s\beta}{\vu}\geq 1-\varepsilon$.  Then, we have 
    \begin{equation*}
        \inner{\beta}{s\vx} = \inner{\vu}{\vx} + O_p(\sqrt{2\varepsilon\log{d}}). 
    \end{equation*}
\end{lemma}

\begin{proof} First, note that 
    \begin{equation*}
        \inner{\beta}{s\vx} = \inner{s\beta}{\vx} =\inner{\vu}{\vx} + \inner{s\beta-\vu}{\vx}. 
    \end{equation*}
    Since we have $\inner{s\beta}{\vu}\geq 1-\varepsilon$ and $s^2=1$, we have that 
    \begin{equation*}
        \|s\beta-\vu\|^2 = \|s\beta\|^2 + \|\vu\|^2 - 2\inner{s\beta}{\vu} \leq 2\varepsilon,
    \end{equation*} 
    which yields 
    \begin{equation*}
        \|s\beta-\vu\| \leq \sqrt{2\varepsilon}. 
    \end{equation*}  
    Then, combining with Lemma~\ref{lemm:gaussianinner} yields $\inner{s\beta-\vu}{\vx}=O_p(\sqrt{2\varepsilon\log{d}})$.
\end{proof}

\section{Training MLP layer}\label{appx:mlp}

In this section, we show that the MLP layer can fit the polynomial $ \sigma_*^{\mathrm{test}}$.  Most of the argument in this section is taken from \citet{nishikawa2025nonlinear}, but we do not omit the proof for readability.

\begin{lemma}\label{lemm:continuousnn}
    Suppose that there exists $g(\vx)$ and $s \in \{-1,1\}$ such that
    \begin{equation*}
        |g(\vx)-\inner{\beta}{s\vx}| \leq \delta. 
    \end{equation*}
    Then, there exists $\pi(v,b)$ such that
    \begin{equation*}
        \left|\mathbb{E}_{v\sim\mathrm{Unif}\{\pm 1\},b\sim [-\log^2{d},\log^2{d}]}[\pi(v,b)\sigma(v\cdot g(\vx)+b)]-\sigma_*(\inner{\beta}{\vx})\right|=O(\delta(\log d)^{2deg(\sigma_*)-2})
    \end{equation*}
    with high probability over $\vx\sim \mathcal{N}(0,\mI_d)$.
    Moreover, $\sup_{v,b}|\pi(v,b)|=\tilde{O}(1)$ holds.
\end{lemma}
\begin{proof}

Let $\sigma_*(z)=\sum_{k=0}^{\deg(\sigma_*)}p_k z^k$.
Now from Lemma 9 in~\citet{damian2022neural}, there exists $\pi'_k(v,b)$ such that $\sup_{v,b}|\pi'_k(v,b)|=O(1)$ and
\begin{equation*}
    \mathbb{E}_{v\sim\mathrm{Unif}\{\pm 1\},b\sim [-1,1]}[\pi'_k(v,b)\sigma(vz+b)]=z^k
\end{equation*}
for any $|z|\leq 1$.
If we define 
\begin{equation*}
    \pi(v,b)= \sum_{k=0}^{\deg(\sigma_*)}p_k\frac{\pi'_k(v,b\log^{-2} d)}{\log^2 d} s^{k} \log^{2k}d,
\end{equation*}

then, we have $\sup_{v,b}|\pi(v,b)|=O(\log^{2\deg(\sigma_*)-2}d)$.  Moreover, when $|z\log^{-2} d|\leq 1$, we have
\begin{align*}
&\mathbb{E}_{v\sim\mathrm{Unif}\{\pm 1\},b\sim [-\log^2 d,\log^2 d]}[\pi(v,b)\sigma(vz+b)]\\
=&\sum_{k=0}^{\deg(\sigma_*)}p_ks^{k}\mathbb{E}_{v\sim\mathrm{Unif}\{\pm 1\},b\sim \left[-\log^2 d,\log^2 d\right]}\left[\frac{\pi'_k(v,b\log^{-2} d)}{\log^2 d} \log^{2k}d\sigma(vz+b)\right]\\
=&\sum_{k=0}^{\deg(\sigma_*)}p_ks^{k}\log^{2k-2}d\mathbb{E}_{v\sim\mathrm{Unif}\{\pm 1\},b\sim \left[-1,1\right]}\left[\pi'_k(v,b)\sigma(\log^{2}d(vz\log^{-2}d+b))\right]\\
=&\sum_{k=0}^{\deg(\sigma_*)}p_k s^{k} z^k=\sigma_*(sz).
\end{align*}

Therefore, it holds that 
\begin{align*}
    &\mathbb{E}_{v\sim\mathrm{Unif}\{\pm 1\},b\sim [-\log^2{d},\log^2{d}]}[\pi(v,b)\sigma(v\cdot g(\vx)+b)]\\
=& \mathbb{E}_{v\sim\mathrm{Unif}\{\pm 1\},b\sim [-\log^2{d},\log^2{d}]}
[\pi(v,b)\sigma(v \inner{\beta}{s\vx}+v\delta+b)]\\
=& \mathbb{E}_{v\sim\mathrm{Unif}\{\pm 1\},b\sim [-\log^2{d},\log^2{d}]}[\pi(v,b)\sigma(v \inner{\beta}{s\vx}+b)]+O(\delta \log^{2 \deg(\sigma_*)-2}d)\\
=& \sigma_*(s\inner{\beta}{s\vx})+O(\delta \log^{2\deg(\sigma_*)-2}d)\\
=& \sigma_*(\inner{\beta}{\vx})+O(\delta \log^{2\deg(\sigma_*)-2}d).
\end{align*}
Here, we used the facts that $s^2=1$ and $|\inner{\beta}{s\vx}|\leq \log^2d$ with high probability.

\end{proof}

\begin{lemma}\label{lemm:discretenn}

Under the condition of Lemma~\ref{lemm:continuousnn}, there exists $\va'\in\mathbb{R}^m$ such that
\begin{equation*}
    \left| \sum_{j=1}^m a_j' \sigma(v_j\cdot g(\vx)+b_j)-\sigma_*(\inner{\beta}{\vx}) \right| = \tilde{O}(m^{-\frac{1}{2}})+O(\delta \log^{2 \deg(\sigma_*)-2}d)
\end{equation*}
holds with high probability over $\vx\sim \mathcal{N} (0,\mI_d)$ and the random features $v_j$ and $b_j$.
Moreover, $\|\va'\|^2=\tilde{O}(m^{-1})$ holds with high probability.

\end{lemma}

\begin{proof}
Using $\pi(v,b)$ in Lemma~\ref{lemm:continuousnn}, define $a_j' = m^{-1}\pi(v_j,b_j)$. 
Since $\sup_{v,b}|\pi(v,b)|=\tilde{O}(1)$, from Hoeffdng's inequality, it holds that 
\begin{equation*}
    \left|m^{-1}\sum_{j=1}^m \pi(v_j,b_j)\sigma(v_j\cdot g(\vx)+b_j)-\mathbb{E}_{v,b}[\pi(v,b)\sigma(v\cdot g(\vx)+b)]\right|=\tilde{O}(m^{-1/2}), 
\end{equation*}
with high probability.  Hence we have
\begin{equation*}
    \left| \sum_{j=1}^m a_j' \sigma(v_j\cdot g(\vx)+b_j)-\sigma_*(\inner{\beta}{\vx}) \right| = \tilde{O}(m^{-\frac{1}{2}})+O(\delta \log^{2 \deg(\sigma_*)-2}d). 
\end{equation*}

As for the upper bound of the norm $\|\va'\|^2=m^{-1}\cdot m^{-1}\sum_{j=1}^m \pi(v_j,b_j)^2$,  Hoeffding's inequality yields 
$$
\left|m^{-1}\sum_{j=1}^m \pi(v_j,b_j)^2-\mathbb{E}_{v,b}[\pi(v,b)^2]\right|=\tilde{O}(m^{-1/2}). 
$$
Since $\sup_{v,b}|\pi(v,b)|=\tilde{O}(1)$, we can say that $\mathbb{E}_{v,b}[\pi(v,b)^2]=\tilde{O}(1)$, which completes the proof.
\end{proof}

\begin{lemma}\label{lemm:optimization}
Let $\va^*$ be the parameter trained via Algorithm \ref{alg:pretraining}. Then there exists $\lambda_2$ such that
\begin{equation*}
    \frac{1}{N_4}\sum_{t=N_1+N_2+N_3+1}^{N_1+N_2+N_3+N_4}\left|y_t-f_\mathrm{TF}(\vx_t,\vu,\vv^*,\va^*,\vb^*)\right|=\tau+\tilde{O}(m^{-1/2})+O(\delta\log^{2\deg(\sigma_*)-2}d) 
\end{equation*}
with high probability.  Moreover,  $\|\va^*\|^2 \leq \tilde{O}_p(m^{-1})$ is satisfied.
\end{lemma}

\begin{proof}
Let $M = N_1+N_2+N_3$ and $\va'$ be the output parameter constructed in Lemma~\ref{lemm:discretenn}: from the equivalence between $\ell_2$-regularized and norm-constrained optimization algorithms, if we carefully choose $\lambda_2$, then we have 
\begin{align*}
    \left(\frac{1}{N_4}\sum_{t=M+1}^{M+N_4}\left|y_t-f_\mathrm{TF}(\vx_t,\vu,\vv^*,\va^*,\vb^*)\right|\right)^2&\leq\frac{1}{N_4}\sum_{t=M+1}^{M+N_4}\left(y_t-f_\mathrm{TF}(\vx_t,\vu,\vv^*,\va^*,\vb^*)\right)^2\\
    &\leq\frac{1}{N_4}\sum_{t=M+1}^{M+N_4}\left(y_t-f_\mathrm{TF}(\vx_t,\vu,\vv^*,\va',\vb^*)\right)^2\\
    &\leq(\tau+\tilde{O}(m^{-1/2})+O(\delta\log^{2\deg(\sigma_*)-2}d))^2 
\end{align*}
(recall that we assumed $\tau=\Theta(1)$) with high probability, which yields the first assertion. 
Moreover, from Lemma~\ref{lemm:discretenn}, we can see that 
\begin{equation*}
    \|\va^*\|^2 \leq \|\va'\|^2 \leq \tilde{O}_p(m^{-1}), 
\end{equation*}
which completes the proof.
\end{proof}

\begin{lemma}\label{lemm:rademacher}
    We fix $\vb$ and $\vv$.  Let $\mathcal{F}_A = \left\{ \va \mapsto \sum_{i=1}^{m}a_j\sigma(v_j\inner{\vx}{\vu}+b_j)\mid\|\va\|\leq A\right\}$ the set of transformers where the norm of the MLP parameter $\va$ is upper bounded.  When $\|\vb\|\leq B$ and $v_j=\pm 1$, then it holds that
    \begin{equation*}
        \mathrm{Rad}_N(\mathcal{F}_A) = \tilde{O}\left(\frac{A(B+\sqrt{m})}{\sqrt N_4}\right).
    \end{equation*}
\end{lemma}

\begin{proof}
Note that $\E [g(\vx)^2]=\E[\inner{\vu}{\vx}^2]=O(\log{d})$,
considering $\|\vu\|=1$.  Then, following the same argument as Lemma 25 in \citet{nishikawa2025nonlinear} yields the assertion.
\end{proof}

\section{Proof of the Theorem \ref{theo:main}}\label{sec:ProofOfMainTheorem}
Now, we are ready to give the proof of Theorem \ref{theo:main}. 
\begin{proof}
First note that 
\begin{align*}
    \mathcal{R}_{f_{\mathrm{TF}}}(\vu,\vv^*,\va^*,\vb^*)-\tau=& \frac{1}{N_4}\sum_{i=N_1+N_2+N_3+1}^{N_1+N_2+N_3+N_4}\left|y_i-f_\mathrm{TF}(\vx_i,\vu,\vv^*,\va^*,\vb^*)\right|-\tau\\
    &+\mathcal{R}_{f_{\mathrm{TF}}}(\vu,\vv^*,\va^*,\vb^*)-\frac{1}{N_4}\sum_{i=N_1+N_2+N_3+1}^{N_1+N_2+N_3+N_4}\left|y_i-f_\mathrm{TF}(\vx_i,\vu,\vv^*,\va^*,\vb^*)\right|. 
\end{align*}
From Lemma~\ref{lemm:optimization},the first two terms $\frac{1}{N_4}\sum_{i=N_1+N_2+N_3+1}^{N_!+N_2+N_3+N_4}\left|y_i-f_\mathrm{TF}(\vx_i,\vu,\vv^*,\va^*,\vb^*)\right|-\tau$ are  bounded by $\tilde{O}(m^{-1/2})+O(\delta\log^{2\deg(\sigma_*)-2}d)$.
Moreover, from Lemma~\ref{lemm:discretenn} and the definition of $\vb^*$, we have $\|\va^*\|=\tilde{O}(m^{-1/2})$ and  $\|\vb^*\|=\tilde{O}(\sqrt{m})$.  Therefore, from Lemma~\ref{lemm:rademacher}, we have
\begin{equation*}
     \mathrm{Rad}_N(\mathcal{F}_A) = \tilde{O}(N_4^{-1/2}).
\end{equation*}
Using the standard technique yields(see Appendix D.3 in \citet{oko2024pretrainedtransformerefficientlylearns})
\begin{equation*}
    \mathcal{R}_{f_{\mathrm{TF}}}(\vu,\vv^*,\va^*,\vb^*)-\tau=\tilde{O}(N_4^{-1/2})+\tilde{O}(m^{-1/2})+O(\delta\log^{2\deg(\sigma_*)-2}d), 
\end{equation*}
with probability at least 0.995.  Note that all the desired events occur with high probability, except for this event.  Therefore, when $d$ is large enough, all the events occur with probability at least 0.99.

Since we consider the discrete noise $\zeta \sim \mathrm{Unif}(\{-\tau,\tau\})$, the Bayes optimal risk is $\mathbb{E}[|\zeta|]=\tau$.  Since no predictor can achieve a predictive risk lower than this irreducible risk, it holds that 
\begin{equation*}
    \mathcal{R}_{f_{\mathrm{TF}}}(\vu,\vv^*,\va^*,\vb^*)-\tau=|\mathcal{R}_{f_{\mathrm{TF}}}(\vu,\vv^*,\va^*,\vb^*)-\tau|. 
\end{equation*}

Finally, we rewrite the term $O(\delta\log^{2\deg(\sigma_*)-2}d)$.  From Lemma~\ref{lemm:epsilonanddelta}, $\delta = O(\sqrt{2\varepsilon\log{d}})$ holds.  Also, $N_3 = \tilde{\Theta}(\frac{r\sqrt{r}}{\varepsilon}\log{\frac{1}{\varepsilon}})$, which means, when we ignore the term $\log{\frac{1}{\varepsilon}}$, $\varepsilon=\tilde{\Theta}({\frac{r\sqrt{r}}{N_3}})$.  Therefore
\begin{equation*}
    O(\delta\log^{2\deg(\sigma_*)-2}d) = \tilde{O}\left(\sqrt{\frac{r\sqrt{r}}{N_3}}\right)
\end{equation*}
is satisfied.  This completes the proof.
\end{proof}

\section{Additional Lemmas}

%\subsection{High probability upper bounds}

\begin{lemma}\label{lemm:gaussianinner}
Suppose $\vx\sim \mathcal{N}(0,\mI_d)$ and $\vy$ is a vector satisfying $\|\vy\|=C$ independent of $\vx$.  Then
\begin{equation*}
    \inner{\vx}{\vy}=O(C\sqrt{\log{d}})
\end{equation*}
holds with high probability.
\end{lemma}

%\begin{lemma}\label{lemm:difvector}
%Let $\vx_1,\vx_2,\dots\vx_n$ i.i.d sub-Gaussian vectors which satisfies $\|x_i\|\leq C_x$ with high probability for any $1 \leq i \leq n$, and $z_1,\dots.z_n$ i.i.d. random variables which satisfy $\|z_i\|\leq C_z$ with high probability.  Then, 
%\begin{equation*}
%    \left\|\frac{1}{n}\sum_{i=1}^n(z_i\vx_i)-\mathbb{E}[z_1\vx_1]\right\|=\tilde{O}\left(C_zC_x\sqrt{\frac{1}{n}}\right)
%\end{equation*}
%holds with high probability.
%\end{lemma}

\begin{lemma}\label{lemm:hoeffding}
    Let $z_1,\dots,z_n$ be i.i.d. random variables which satisfy $\|z_i\|\leq C$ with high probability and $\mathbb{E}[z_1^2]=O(d^{\alpha})$, where $\alpha$ is a constant independent of $d$.  When $n=\mathrm{poly}(d)$,
    \begin{equation*}
        \frac{1}{n}\sum_{i=1}^{n}z_i-\mathbb{E}[z_i]=O\left(C\sqrt{\frac{1}{n}}\right)
    \end{equation*}
    holds with high probability.
\end{lemma}

\begin{proof}
    Let $z_i' = z_i\1_\mathrm{|z_i|\leq C}$.  Applying the uniform bound argument, we may consider that $z_i=z_i'$ for all $i=1,\dots,n$ with high probability because our assumption $n=\mathrm{poly}(d)$ yields
    \begin{equation*}
        P(\max_i |z_i|\leq C)\geq 1-O(nd^{-C_*})\geq 1-O(d^{-C'_*}),
    \end{equation*}
    where $C'_*$ is a constant determined appropriately.  Then, using Hoeffding's inequality yields

    \begin{equation*}
        \frac{1}{n}\sum_{i=1}^{n}z_i-\mathbb{E}[z_i']=O\left(C\sqrt{\frac{1}{n}}\right).
    \end{equation*}
    We complete the proof by showing that $|\mathbb{E}[z_i]-\mathbb{E}[z_i']|$ is sufficiently small.  This can be shown by
    \begin{align*}
        \mathbb{E}[z_i]-\mathbb{E}[z_i']=&\mathbb{E}[z_i\1_\mathrm{|z_i| > C_z}]\\
        \leq& \mathbb{E}[\1_\mathrm{|z_i| > C_z}]^{1/2} \mathbb{E}[z_i]^{1/2}\\
        \leq & O(d^{\alpha-C_*}),
    \end{align*}
    where $C_*$ can be taken sufficiently large from the definition of high probability event.  Since we assumed that $n=\mathrm{poly}(d)$, this term is smaller than the main term, which completes the proof.
\end{proof}

\begin{lemma}[\cite{damian2023smoothinglandscapeboostssignal}, Property 1]
    Let $\alpha, \beta \in \mathbb{S}^{d-1}$, then
    \begin{equation*}
        \mathbb{E}_{\vx \sim \mathcal{N}(0,\mI_d)}[\mathrm{He}_i(\inner{\alpha}{\vx})\mathrm{He}_i(\inner{\beta}{\vx})] = \1_\mathrm{i=j}\, i!\inner{\alpha}{\beta}^i.
    \end{equation*}
\end{lemma}

\subsection{Pretrained matrix}

Let $\Gamma^*$ be a matrix obtained after pretraining.
Then, from \citet{nishikawa2025nonlinear}, $\Gamma^*$ can be written as
$$\Gamma^*=\frac{r\mathbb{E}_{\beta}[\beta\beta^\top]+\mN}{\kappa\sqrt{r}}$$ for some matrix $\mN$ satisfying $\|\mN\|_F=O(1/\sqrt{d})$ and a number $\kappa=\Theta(\log^{C_\kappa}{d})$.
Also, from \citet{nishikawa2025nonlinear} and the assumption on the support of $\beta$, 
\begin{equation*}
    \mU\beta\sim \mathrm{Unif}\{(\alpha_1,\alpha_2,\dotsc,\alpha_r,0,\dotsc,0)\mid\alpha_1^2+\dots+\alpha_r^2=1\}.
\end{equation*}
holds for some orthogonal matrix $\mU$.  Then, using 
$\mD=\mathrm{diag}(\underbrace{1,\dotsc,1}_{r},\underbrace{0,\dots,0}_{d-r})$, we have that 
$r\mathbb{E}_{\beta}[\beta\beta^\top]=\mU^{\top}\mD\mU$.

\begin{lemma}\label{lemm:Gammax}
Let $\vx\sim \mathcal{N}(0,\mI_d)$, then
\begin{equation*}
    \|\sqrt{r}\kappa\Gamma^*\vx\|=\tilde{O}(\sqrt{r})
\end{equation*}
holds with high probability.
\end{lemma}

\begin{proof}
By the argument above, we can write that  
\begin{align*}
    \sqrt{r}\kappa\Gamma^*\vx =& (\mU^{\top}\mD\mU + \mN) \vx. 
\end{align*}
From rotational invariance, when we define $\vy$ as $\vy = \mU \vx$, $\vy\sim \mathcal{N}(0,\mI_d)$ is satisfied.
From the definition of $\mD$, the first $r$ components of $\mD \vy$ follow standard normal distribution i.i.d., and the other $(n-r)$ components are equal to zero.  Since applying $U^{\top}$ to a vector does not change the norm of the vector, we obtain 
\begin{equation*}
    \mU^{\top}\mD\mU \vx = \mU^{\top}\mD \vy = \tilde{O}(\sqrt{r}), 
\end{equation*}
with high probability.  Finally, since $\|\mN\|_F=O(1/\sqrt{d})$, 
$\mN \vx = \tilde{O}(1)$ holds.  Hence we can ignore the term $\mN \vx$.
\end{proof}

\begin{lemma}\label{lemm:Gammaxandu}
    Let $\vx\sim \mathcal{N}(0,\mI_d)$, and $\vu$ be a vector independent of $\vx$ satisfying $\|\vu\| = 1$.  then
    \begin{equation*}
         \inner{\sqrt{r}\kappa\Gamma^*\vx}{\vu} = O(\sqrt{\log{d}})
    \end{equation*}
    holds with high probability.
\end{lemma}

\begin{proof}
As in the previous lemma, we know that 
    \begin{align*}
    \sqrt{r}\kappa\Gamma^*\vx =& (\mU^{\top}\mD\mU + \mN) \vx 
\end{align*}
holds.  
Again, from rotational invariance, when we define $\vy$ as $\vy = \mU \vx$, $\vy\sim \mathcal{N}(0,\mI_d)$ is satisfied.
From the definition of $\mD$, the first $r$ components of $\mD \vy$ follow standard normal distribution i.i.d., and the other $(n-r)$ components are equal to zero.  Also, when we define $\vu'=\mU\vu$, we have $\|\vu'\| = 1$, which means $\sqrt{u_1'^2+\dots+u_r'^2}\leq1$.  Therefore, from Corollary 30 in \citet{nishikawa2025nonlinear}, we have
\begin{equation*}
    \inner{\mU^{\top}\mD\mU\vx}{\vu} = \inner{\mD\vy}{\vu'} = O_p(\sqrt{\log{d}}). 
\end{equation*}
As for $\mN \vx$, again from Corollary 30 in \citet{nishikawa2025nonlinear}, we have
\begin{equation*}
    \inner{\mN \vx}{\vu} = O_p\left(\sqrt{\frac{\log{d}}{d}}\right). 
\end{equation*}
This is smaller than the main term. 
\end{proof}

\begin{lemma}\label{lemm:difvector2}
     Let $\vx_1 \dots \vx_n\sim \mathcal{N}(0,\mI_d)$, and $z_1,\dots.z_n$ are i.i.d. random variables which satisfy $|z_i|\leq C_z$ with high probability. 
     When $n=\mathrm{poly}(d)$, it holds that 
     \begin{equation*}
    \left\|\frac{1}{n}\sum_{i=1}^n(z_i\sqrt{r}\kappa\Gamma^*\vx_i)-\mathbb{E}[z_1\sqrt{r}\kappa\Gamma^*\vx_1]\right\|=\tilde{O}\left(C_z\sqrt\frac{r}{n}\right),  
     \end{equation*}    
     with high probability.
\end{lemma}

\begin{proof}
    We use the same proof strategy as Lemma 31 in \citet{nishikawa2025nonlinear}.  Let $z_i' = z_i\1_\mathrm{|z_i|\leq C_z}$.  First, we can confirm that
    \begin{align*}
        \mathbb{E}[z_1\sqrt{r}\kappa\Gamma_*\vx_1] - \mathbb{E}[z_1'\sqrt{r}\kappa\Gamma_*\vx_1] =& \mathbb{E}[\1_{|z_i|>C_z}\sqrt{r}\kappa\Gamma_*\vx_1]\\
        \leq & \mathbb{E}[\1_{|z_i|>C_z}^2]^{1/2}\mathbb{E}[\|\sqrt{r}\kappa\Gamma_*\vx_1\|^2]^{1/2}\\
        \leq & O(d^{-C}). 
    \end{align*}
    Where $C$ can be taken sufficiently large from the definition of high probability event.  Because $n=\mathrm{poly}(d),$ by redefining $C$ if necessary, we can make this term smaller than $\tilde{O}(C_z\sqrt{\frac{r}{n}})$.  As $\frac{1}{n}\sum_{i=1}^n(z_i\sqrt{r}\kappa\Gamma^*\vx_i)=\frac{1}{n}\sum_{i=1}^n(z_i'\sqrt{r}\kappa\Gamma^*\vx_i)$ with high probability, 
    we continue to evaluate $\frac{1}{n}\sum_{i=1}^n(z_i'\sqrt{r}\kappa\Gamma^*\vx_i)-\mathbb{E}[z_1'\sqrt{r}\kappa\Gamma_*\vx_1]$.  
    Let $\vy_i=\mU\vx_i$.  Then $\vy_i \sim \mathcal{N}(0,\mI_d)$ holds.  $z_i'\mD \vy_i$ is a sub-Gaussian vector, and $(\mD \vy_i)_k=0$ when $k>r$ from the definition of $\mD$.  Then, applying a standard concentration bound for a sub-Gaussian vector to the $r$-dimensional vector $(\mD\vy_i)_{1:r}$ yields
    \begin{equation*}
        \frac{1}{n}\sum_{i=1}^{n}z_i'\mD \vy_i-\mathbb{E}[z_1'\mD\vy_1]=\tilde{O}\left(C_z\sqrt{\frac{r}{n}}\right). 
    \end{equation*}
    As multiplying the orthogonal matrix $U^{\top}$ does not change the norm of the vector, we have
      \begin{equation*}
        \frac{1}{n}\sum_{i=1}^{n}z_i'\mU^{\top}\mD\mU\vx_i-\mathbb{E}[z_1'\mU^{\top}\mD\mU\vx_1]=\tilde{O}\left(C_z\sqrt{\frac{r}{n}}\right).
    \end{equation*}
    Again from the standard concentration bound for a sub-Gaussian vector, we have 
    \begin{equation*}
        \frac{1}{n}\sum_{i=1}^{n}z_i'\vx_i-\mathbb{E}[z_1'\vx_1]=\tilde{O}\left(C_z\sqrt{\frac{d}{n}} \right), 
    \end{equation*}
    with high probability. 
    Since $\|\mN\|_F =O(1/\sqrt{d})$, we have
    \begin{equation*}
        \frac{1}{n}\sum_{i=1}^{n}z_i'\mN\vz_i-\mathbb{E}[z_1'\mN\vx_1]=\tilde{O}\left(C_z\sqrt{\frac{1}{n}}\right). 
    \end{equation*}
    In summary, we arrive at 
    \begin{align*}
        &\frac{1}{n}\sum_{i=1}^n(z_i\sqrt{r}\kappa\Gamma^*\vx_i)-\mathbb{E}[z_1\sqrt{r}\kappa\Gamma^*\vx_1]\\
        &=\frac{1}{n}\sum_{i=1}^n(z_i(\mU^{\top}\mD\mU+\mN)\vx_i)-\mathbb{E}[z_1(\mU^{\top}\mD\mU+\mN)\vx_1]
        =\tilde{O}\left(C_z\sqrt\frac{r}{n}\right), 
    \end{align*}
    with high probability, which completes the proof.
\end{proof} 

\begin{lemma}\label{lemm:Gammabeta}
We have that 
    \begin{equation*}
        \sqrt{r}\kappa\Gamma^*\beta = \beta + O(1/\sqrt{d}). 
    \end{equation*}
\end{lemma}

\begin{proof}
First note that 
    \begin{equation*}
        \sqrt{r}\kappa\Gamma^*\beta = (\mU^{\top}\mD\mU + \mN) \beta. 
    \end{equation*}
    From the definition of $U$, we have $\mU\beta = (\alpha_1,\alpha_2,\dotsc,\alpha_r,0,\dotsc,0)^\top$ for some $\alpha_1,\alpha_2,\dots,\alpha_r \in \R$.
    Therefore, we obtain that 
    \begin{align*}
        \mU^{\top}\mD\mU\beta = \mU^{\top} \mU \beta = \beta. 
    \end{align*}
    Finally, by noticing $\|\mN\|_F=O(1/\sqrt{d})$ and $\|\beta\|=1$, 
    we can see that $\mN \beta = O(1/\sqrt{d})$ holds.
\end{proof}

\begin{lemma}\label{lemm:betaanddoublegamma}
We have that 
    \begin{equation*}
        \sqrt{r}\kappa\Gamma^*(\sqrt{r}\kappa\Gamma^*\beta)=\beta+O(1/\sqrt{d}). 
    \end{equation*}
\end{lemma}
\begin{proof}
    From lemma~\ref{lemm:Gammabeta}, we have $\sqrt{r}\kappa\Gamma^*\beta = \beta + \vd$ where $\vd = O(1/\sqrt{d})$.  Therefore, it holds that 
    \begin{equation*}
        \sqrt{r}\kappa\Gamma^*(\sqrt{r}\kappa\Gamma^*\beta)=\sqrt{r}\kappa\Gamma^*(\beta + \vd). 
    \end{equation*}
    Again from Lemma~\ref{lemm:Gammabeta}, we see that $\sqrt{r}\kappa\Gamma^*\beta = \beta + O(1/\sqrt{d})$. Also we have $\sqrt{r}\kappa\Gamma^*\vd = \mU^{\top}\mD\mU\vd + \mN\vd$, and since $\|\mU^{\top}\mD\mU\|_2 = 1$ from the definition of $\mU$ and $\mD$, we have $\mU^{\top}\mD\mU\vd = O(1/\sqrt{d})$.  Finally, since $\|\mN\|_F=O(1/\sqrt{d})$, we have $\mN\vd = O(1/d)$.  This indicates that $\sqrt{r}\kappa\Gamma^*\vd =O(1/\sqrt{d})$.  
\end{proof}

\section{Experimental details and additional experiment}\label{appx:experiment}

\subsection{Additional experiment}
In the main experiment (Figure~\ref{fig:main}), we only assessed the model with TTT in the situation where $r=d$. To further assess the capabilities of TTT, we examined its ability to leverage low-dimensional task structures.  We fixed the intrinsic dimension at $r=4$ and compared the predictive error across varying ambient dimensions $d \in \{4,8,16\}$.  In this experiment, to facilitate the convergence of the attention mechanism on the target subspace, we fixed the link function as $\sigma_*^t = \frac{1}{\sqrt{3!}}\mathrm{He}_3$ during the pretraining.  After pretraining, the model was evaluated on the standard task where the link function $\sigma_*$ varies across tasks.  The distribution of the link function was the same as in the in-distribution case in the experiment for Figure~\ref{fig:main}. As illustrated in Figure~\ref{fig:comparison_dr}, the model performance remains virtually unchanged despite the fourfold increase in the ambient dimension. This empirical result provides strong evidence for our main theorem: the sample complexity of TTT-equipped transformers is governed by the intrinsic dimension $r$ rather than the ambient dimension $d$, showcasing the model's ability to exploit low-dimensional task structures effectively.

\begin{figure}[tb]
    \centering % 全体を中央揃え
    \includegraphics[width=0.45\linewidth]{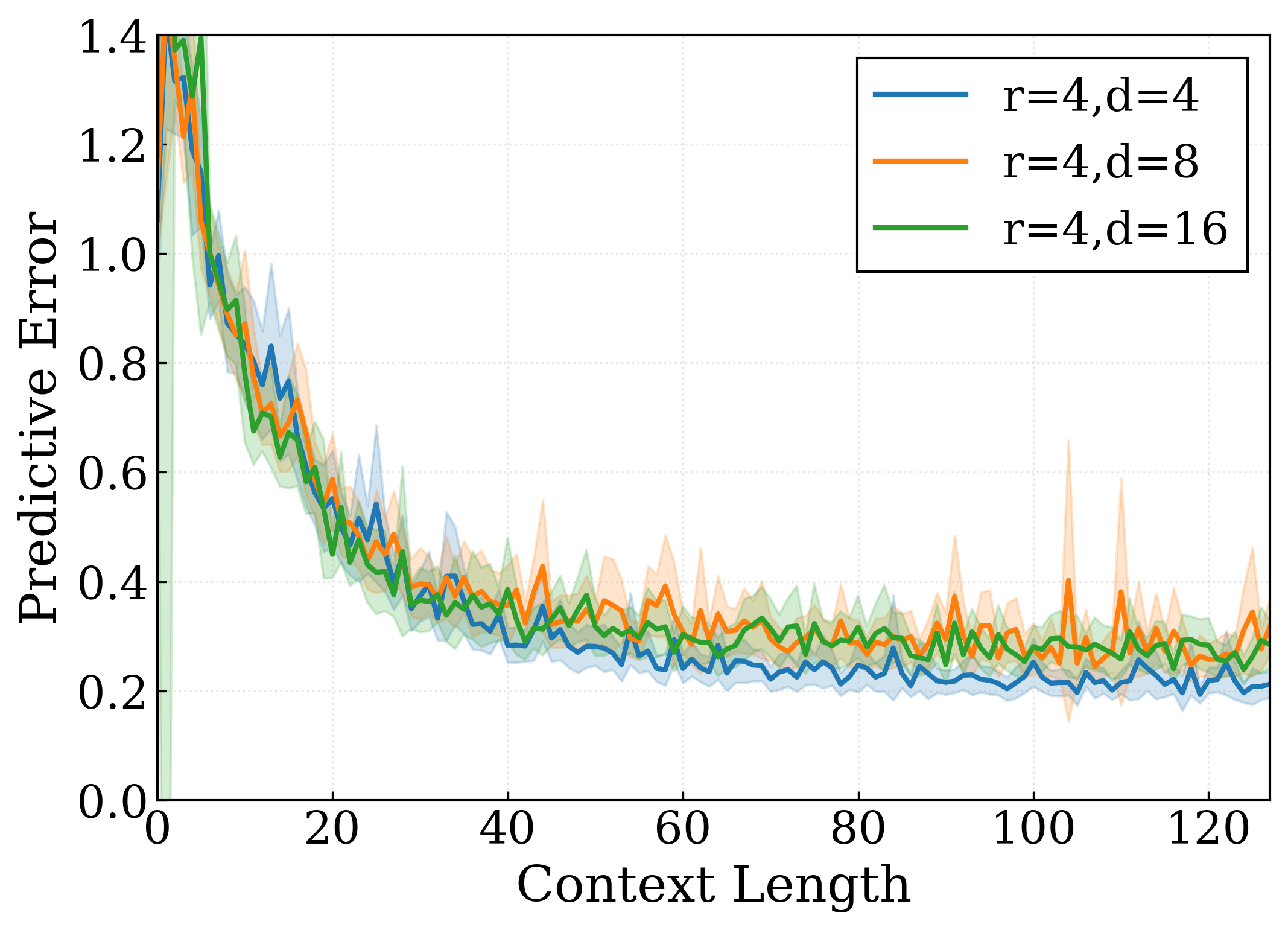}
    \caption{Predictive error of TTT across different ambient dimensions $d \in \{4,8,16\}$ with a fixed intrinsic dimension $r=4$. The overlap of the convergence curves demonstrates that TTT is robust to increases in the ambient dimension, empirically verifying that sample complexity scales with the intrinsic dimensionality $r$.}
    \label{fig:comparison_dr}
\end{figure}

\subsection{Experimental details}
The GPT-2 model architecture we used in Section~\ref{sec:experiment} originates from \citet{garg2023transformerslearnincontextcase} under MIT license.  Given the $(N+1)$-length prompt $\{(\vx_i,y_i)\}_{i=1}^{N+1}$, we first construct the embedding as 
\begin{equation*}
    \mE = [\vx_1,\tilde{\vy_1},\dots,\vx_{N+1},\tilde{\vy_{N+1}}] \in \mathbb{R}^{d\times (2N+2)}, 
\end{equation*}
where $\tilde{\vy_i} = [y_i,0,\dots,0]^{\top}$.  Next, the read-in layer transforms this embedding into $\tilde{\mE} \in \mathbb{R}^{D\times(2N+2)}$, where $D=128$.  This mapped embedding $\tilde{\mE}$ goes through a 2-layer GPT-2 backbone with 4 attention heads, following the configuration by \citet{garg2023transformerslearnincontextcase}.  Finally, the output of GPT-2 backbone is transformed by the read-out layer into the vector $[z_1,z_2,\dots,z_{2N+1},z_{2N+2}]$.  Here, $z_{2i-1}$ is the prediction of $y_i$ given the context $(\vx_1,y_1,\dots,\vx_{i-1},y_{i-1},\vx_i)$. 
We used the Adam optimizer \citep{kingma2017adammethodstochasticoptimization}  with a learning rate of 0.0001.  To reduce the pretraining cost, we adopted the curriculum learning strategy, which is also used in \citet{garg2023transformerslearnincontextcase}.  The training started with the dimension $d=4$, and the dimension was increased by two until it reached the target dimension. 

For the model with TTT, we introduced low-rank adaptation (LoRA) to the attention projection layers (c\_attn and c\_proj) and the feedforward layer (c\_fc) of the base model.  The rank of LoRA was set to $4$, and the parameters LoRA\_alpha and LoRA\_dropout were set to $8$ and $0.1$, respectively.  The LoRA parameters were updated $30$ times and $300$ times via SGD for each query in the experiment in Figure~\ref{fig:main} and Figure~\ref{fig:comparison_dr}, respectively.  The inference-time learning rate was $0.1$ in the experiment in Figure~\ref{fig:main} and Figure~\ref{fig:comparison_dr}.  To prevent information leakage when evaluating the prediction of $y_i$, the model used only the preceding data $(\vx_1,y_1,\dots,\vx_{i-1},y_{i-1},\vx_i)$ to update the weight.  The LoRA weights were reset before proceeding to the prediction of the next target, $y_{i+1}$.  To prevent gradient explosion, the gradient clipping was introduced during test-time: the norm of the gradient was restricted to at most 1.0.

The test loss was averaged over 256 runs, with each run containing 256 independent queries.

All experiments were conducted on an internal computing cluster equipped with NVIDIA A100 GPUs (including both 40GB and 80GB variants), AMD EPYC 7413 24-Core Processors, and 2.0 TiB of system memory. For each experimental run, pretraining requires up to approximately 1 day, while evaluation takes a few minutes for standard ICL. For TTT, evaluation takes approximately 5.5 hours for the setup in Figure~\ref{fig:main} and roughly 2.5 days for Figure~\ref{fig:comparison_dr}. Including hyperparameter tuning and preliminary experiments, the total computational cost is estimated to be approximately 25 GPU-days.

%%%%%%%%%%%%%%%%%%%%%%%%%%%%%%%%%%%%%%%%%%%%%%%%%%%%%%%%%%%%%%%%%%%%%%%%%%%%%%%
%%%%%%%%%%%%%%%%%%%%%%%%%%%%%%%%%%%%%%%%%%%%%%%%%%%%%%%%%%%%%%%%%%%%%%%%%%%%%%%
\newpage

\section*{NeurIPS Paper Checklist}

\begin{enumerate}

\item {\bf Claims}
    \item[] Question: Do the main claims made in the abstract and introduction accurately reflect the paper's contributions and scope?
    \item[] Answer: \answerYes{} % Replace by \answerYes{}, \answerNo{}, or \answerNA{}.
    \item[] Justification: We study test-time training combined with in-context learning to learn nonlinear single-index models, and we established an upper bound on the prediction risk.  This is clearly stated in the abstract and introduction.

\item {\bf Limitations}
    \item[] Question: Does the paper discuss the limitations of the work performed by the authors?
    \item[] Answer: \answerYes{} % Replace by \answerYes{}, \answerNo{}, or \answerNA{}.
    \item[] Justification: We discuss future work and limitations in section~\ref{section:conclusion} at the end of the paper.

\item {\bf Theory assumptions and proofs}
    \item[] Question: For each theoretical result, does the paper provide the full set of assumptions and a complete (and correct) proof?
    \item[] Answer: \answerYes{} % Replace by \answerYes{}, \answerNo{}, or \answerNA{}.
    \item[] Justification: The assumptions are clearly stated in the Theorem~\ref{theo:main}.  All proofs are provided in the appendix.

    \item {\bf Experimental result reproducibility}
    \item[] Question: Does the paper fully disclose all the information needed to reproduce the main experimental results of the paper to the extent that it affects the main claims and/or conclusions of the paper (regardless of whether the code and data are provided or not)?
    \item[] Answer: \answerYes{} % Replace by \answerYes{}, \answerNo{}, or \answerNA{}.
    \item[] Justification: We describe all the information on the experiments in Synthetic experiment section in the main paper and Experimental details section in the appendix.  We also provide the code for the experiment.

\item {\bf Open access to data and code}
    \item[] Question: Does the paper provide open access to the data and code, with sufficient instructions to faithfully reproduce the main experimental results, as described in supplemental material?
    \item[] Answer: \answerYes{} % Replace by \answerYes{}, \answerNo{}, or \answerNA{}.
    \item[] Justification: We provide the full source code and pre-trained model weights in the supplementary material. The readme.txt file includes all the necessary information.

\item {\bf Experimental setting/details}
    \item[] Question: Does the paper specify all the training and test details (e.g., data splits, hyperparameters, how they were chosen, type of optimizer) necessary to understand the results?
    \item[] Answer: \answerYes{} % Replace by \answerYes{}, \answerNo{}, or \answerNA{}.
    \item[] Justification: All the details of the numerical experiment are provided in the Appendix~\ref{appx:experiment}.

\item {\bf Experiment statistical significance}
    \item[] Question: Does the paper report error bars suitably and correctly defined or other appropriate information about the statistical significance of the experiments?
    \item[] Answer: \answerYes{} % Replace by \answerYes{}, \answerNo{}, or \answerNA{}.
    \item[] Justification: The shaded areas in the Figures represent the 95\% confidence intervals.

\item {\bf Experiments compute resources}
    \item[] Question: For each experiment, does the paper provide sufficient information on the computer resources (type of compute workers, memory, time of execution) needed to reproduce the experiments?
    \item[] Answer: \answerYes{} % Replace by \answerYes{}, \answerNo{}, or \answerNA{}.
    \item[] Justification: The details on computer resources are provided in the Appendix~\ref{appx:experiment}.
    
\item {\bf Code of ethics}
    \item[] Question: Does the research conducted in the paper conform, in every respect, with the NeurIPS Code of Ethics \url{https://neurips.cc/public/EthicsGuidelines}?
    \item[] Answer: \answerYes{} % Replace by \answerYes{}, \answerNo{}, or \answerNA{}.
    \item[] Justification: The authors read and followed NeurIPS code of ethics. 

\item {\bf Broader impacts}
    \item[] Question: Does the paper discuss both potential positive societal impacts and negative societal impacts of the work performed?
    \item[] Answer: \answerNA{} % Replace by \answerYes{}, \answerNo{}, or \answerNA{}.
    \item[] Justification: This paper is a theoretical work, and we believe our work has no societal impacts that require a discussion.
    
\item {\bf Safeguards}
    \item[] Question: Does the paper describe safeguards that have been put in place for responsible release of data or models that have a high risk for misuse (e.g., pre-trained language models, image generators, or scraped datasets)?
    \item[] Answer: \answerNA{} % Replace by \answerYes{}, \answerNo{}, or \answerNA{}.
    \item[] Justification: This paper is mainly a theoretical work, and the code for the synthetic experiment has no risk of misuse. 

\item {\bf Licenses for existing assets}
    \item[] Question: Are the creators or original owners of assets (e.g., code, data, models), used in the paper, properly credited and are the license and terms of use explicitly mentioned and properly respected?
    \item[] Answer: \answerYes{} % Replace by \answerYes{}, \answerNo{}, or \answerNA{}.
    \item[] Justification: We clarified in the experimental details section in the appendix that the original code for the synthetic experiments is taken from \citet{garg2023transformerslearnincontextcase} under MIT license.

\item {\bf New assets}
    \item[] Question: Are new assets introduced in the paper well documented and is the documentation provided alongside the assets?
    \item[] Answer: \answerNA{} % Replace by \answerYes{}, \answerNo{}, or \answerNA{}.
    \item[] Justification: This paper does not release new assets.

\item {\bf Crowdsourcing and research with human subjects}
    \item[] Question: For crowdsourcing experiments and research with human subjects, does the paper include the full text of instructions given to participants and screenshots, if applicable, as well as details about compensation (if any)? 
    \item[] Answer: \answerNA{} % Replace by \answerYes{}, \answerNo{}, or \answerNA{}.
    \item[] Justification: Our work includes no crowdsourcing nor research with human subjects.

\item {\bf Institutional review board (IRB) approvals or equivalent for research with human subjects}
    \item[] Question: Does the paper describe potential risks incurred by study participants, whether such risks were disclosed to the subjects, and whether Institutional Review Board (IRB) approvals (or an equivalent approval/review based on the requirements of your country or institution) were obtained?
    \item[] Answer: \answerNA{} % Replace by \answerYes{}, \answerNo{}, or \answerNA{}.
    \item[] Justification: Our work includes no crowdsourcing nor research with human subjects.

\item {\bf Declaration of LLM usage}
    \item[] Question: Does the paper describe the usage of LLMs if it is an important, original, or non-standard component of the core methods in this research? Note that if the LLM is used only for writing, editing, or formatting purposes and does \emph{not} impact the core methodology, scientific rigor, or originality of the research, declaration is not required.
    %this research? 
    \item[] Answer: \answerNA{} % Replace by \answerYes{}, \answerNo{}, or \answerNA{}.
    \item[] Justification: Our usage of LLM is limited to writing, searching for prior works and code assistance.  We did not use any LLM assistant for core aspects of the work, including designing the problem settings and constructing the proofs.

\end{enumerate}

\end{document}